\documentclass{article}
\PassOptionsToPackage{numbers,square,sort&compress}{natbib}
\usepackage[numbers,square,sort&compress]{natbib}
\usepackage[preprint]{neurips_2023}
\usepackage[sectionbib]{bibunits}
\defaultbibliographystyle{unsrtnat}
\usepackage[utf8]{inputenc}
\usepackage[hidelinks]{hyperref}      
\usepackage{url}
\usepackage{amsfonts}
\usepackage{nicefrac}
\usepackage{microtype}
\usepackage{xcolor}
\usepackage{graphicx}
\usepackage{caption}
\usepackage{xcolor}
\usepackage{csquotes}
\usepackage{microtype}
\usepackage{amsmath,amsfonts,amssymb,amsthm}
\usepackage{multirow}
\usepackage[symbol]{footmisc}
\usepackage{titlesec}
\titleformat*{\section}{\large\bfseries}
\titleformat*{\subsection}{\normalsize\bfseries}
\titleformat*{\subsubsection}{\normalsize\bfseries}

\newtheorem{theorem}{Theorem}
\newtheorem{definition}{Definition}

\makeatletter
\def\renewtheorem#1{%
  \expandafter\let\csname#1\endcsname\relax
  \expandafter\let\csname c@#1\endcsname\relax
  \gdef\renewtheorem@envname{#1}
  \renewtheorem@secpar
}
\def\renewtheorem@secpar{\@ifnextchar[{\renewtheorem@numberedlike}{\renewtheorem@nonumberedlike}}
\def\renewtheorem@numberedlike[#1]#2{\newtheorem{\renewtheorem@envname}[#1]{#2}}
\def\renewtheorem@nonumberedlike#1{  
\def\renewtheorem@caption{#1}
\edef\renewtheorem@nowithin{\noexpand\newtheorem{\renewtheorem@envname}{\renewtheorem@caption}}
\renewtheorem@thirdpar
}
\def\renewtheorem@thirdpar{\@ifnextchar[{\renewtheorem@within}{\renewtheorem@nowithin}}
\def\renewtheorem@within[#1]{\renewtheorem@nowithin[#1]}
\makeatother

\title{Maximum diffusion reinforcement learning}

\author{%
  Thomas A.~Berrueta$^*$ \quad \quad Allison Pinosky \quad \quad Todd D.~Murphey$^*$ \\
  Department of Mechanical Engineering\\
  Northwestern University\\
  Evanston, IL 60208 \\
  \texttt{\{tberrueta, apinosky\}@u.northwestern.edu} \quad \texttt{t-murphey@northwestern.edu}
}

\begin{document}
\begin{bibunit}
\maketitle


\begin{abstract}
Robots and animals both experience the world through their bodies and senses. Their embodiment constrains their experiences, ensuring they unfold continuously in space and time. As a result, the experiences of embodied agents are intrinsically correlated. Correlations create fundamental challenges for machine learning, as most techniques rely on the assumption that data are independent and identically distributed. In reinforcement learning, where data are directly collected from an agent's sequential experiences, violations of this assumption are often unavoidable. Here, we derive a method that overcomes this issue by exploiting the statistical mechanics of ergodic processes, which we term maximum diffusion reinforcement learning. By decorrelating agent experiences, our approach provably enables single-shot learning in continuous deployments over the course of individual task attempts. Moreover, we prove our approach generalizes well-known maximum entropy techniques, and robustly exceeds state-of-the-art performance across popular benchmarks. Our results at the nexus of physics, learning, and control form a foundation for transparent and reliable decision-making in embodied reinforcement learning agents.
\end{abstract}

\footnotetext{$^*$Corresponding authors. Code and data are available in the following \href{https://github.com/MurpheyLab/MaxDiffRL}{\color{blue}{\texttt{GitHub Repository}}}. \\ Supplementary movies are available in the following \href{https://www.youtube.com/playlist?list=PLO5AGPa3klrCTSO-t7HZsVNQinHXFQmn9}{\color{blue}{\texttt{YouTube Playlist}}}.}

\addtocontents{toc}{\protect\setcounter{tocdepth}{0}}

\section{Introduction}
\label{sec:introduction}
Reinforcement learning (RL) is a flexible decision-making framework based on the experiences of artificial agents, whose potential for scalable real-world impact has been well-established through the power of deep learning architectures. From controlling nuclear fusion reactors~\cite{Degrave2022} to besting curling champions~\cite{Won2020},  RL agents have achieved remarkable feats when they can exhaustively explore how their actions impact the state of their environment. Despite their impressive achievements, RL agents---especially deep RL agents---suffer from limitations preventing their widespread deployment in the real world: Their performance varies across initializations, their sample inefficiency demands the use of simulators, and they struggle to learn outside of episodic problem structures~\cite{Irpan2018,Henderson2018,Ibarz2021}. At the heart of these shortcomings lies a violation of the assumption that data are independent and identically distributed (\textit{i.i.d.}), which underlies most of machine learning. While learning typically requires \textit{i.i.d.} data, the experiences of RL agents are unavoidably sequential and correlated across points in time. It is no wonder, then, that many of deep RL's most impactful advances have sought to overcome this roadblock~\cite{Lillicrap2016,Haarnoja2018,Plappert2018}.

Over the past decades, researchers have started to converge onto an understanding that destroying temporal correlations is essential to sample efficiency and agent performance, seeking to address them in two primary domains: During optimization and during sample generation. When we consider optimizing a policy from a database of sequential agent-environment interactions, sampling in random batches is known to reduce temporal correlations. For this reason, experience replay~\cite{Lin1992} and its many variants~\cite{Schaul2015,Andrychowicz2017,Zhang2017_replay} have been successful in producing large performance and sample efficiency gains across tasks and algorithms~\cite{Wang2017_replay,Hessel2018,Fedus2020}. This simple insight---merely sampling experiences out of order---was a key contributing factor to one of deep RL's landmark triumphs: Achieving superhuman performance in Atari video game benchmarks~\cite{Mnih2015}. 

Nonetheless, temporal correlations also arise during data generation, where their impact cannot be alleviated through sampling alone. In turn, temporal correlations must be mitigated during data acquisition as well, which requires techniques to sufficiently randomize the sample generation process. In this regard, the maximum entropy (MaxEnt) RL framework has emerged as a key advance~\cite{Ziebart2008,Ziebart2010,Ziebart2010_thesis,Todorov2009,Toussaint2009,Rawlik2012,Levine2013,Haarnoja2017,Haarnoja2018_rss}. These methods seek to generate randomness during optimization and data acquisition by maximizing the entropy of an agent's policy, which decorrelates their action sequences. In doing so, MaxEnt RL techniques have been able to achieve better exploration and more robust performance~\cite{Eysenbach2022}. However, does maximizing the entropy of an agent's policy actually decorrelate their experiences?

Here, we prove that this is generally not the case. To address this gap we introduce maximum diffusion (MaxDiff) RL, a framework that provably decorrelates agent experiences during sample generation, and realizes statistics indistinguishable from \textit{i.i.d.} sampling by exploiting the statistical mechanics of ergodic processes. Our approach efficiently exceeds state-of-the-art performance by diversifying agent experiences and improving state exploration. By articulating the relationship between an agent's embodiment, diffusion, and learning, we prove that MaxDiff RL agents are capable of single-shot learning regardless of how they are initialized. We additionally prove that MaxDiff RL agents are robust to random seeds and environmental stochasticity, which enables consistent and reliable performance with low-variance across agent deployments and learning tasks. Our work sheds a light on foundational issues holding back the field, highlighting the impact that agent properties and data acquisition can play on downstream learning tasks, and paving the way towards more transparent and reliable decision-making in embodied RL agents.

\section{Results}
\label{sec:results}
\subsection{Temporal correlations hinder performance}
\label{sec:sec1}

\begin{figure*}[t!]
    \centering
    \includegraphics[width=0.95\linewidth]{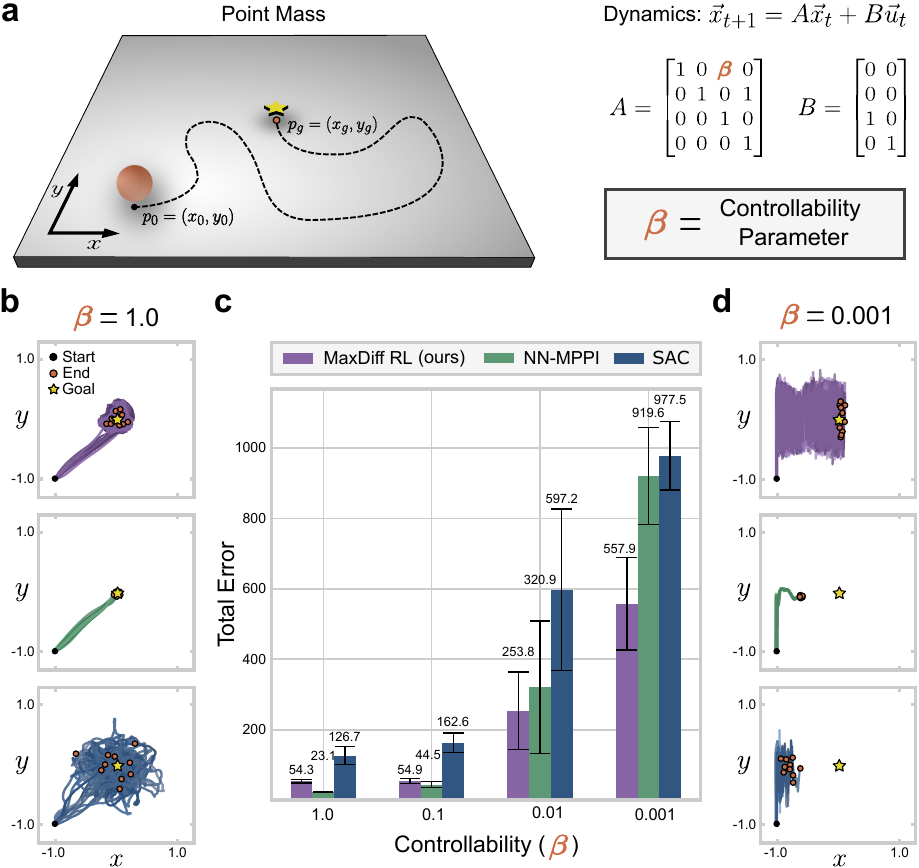}
    \caption{\small
    \textbf{Temporal correlations break the state-of-the-art in RL.} For most systems, controllability properties determine temporal correlations between state transitions (Supplementary Note~\ref{sec:controllability}). \textbf{a}, Planar point mass with dynamics simple enough to explicitly write down and whose policy admits a globally optimal analytical solution. The system's 4-dimensional state space is comprised of its planar positions and velocities. We parametrize its controllability through $\beta \in [0,1]$, where $\beta=0$ produces a formally uncontrollable system. The task is to translate the point mass from $p_0$ to $p_g$ within a fixed number of steps at different values of $\beta$, and the reward is specified by the negative squared Euclidean distance between the agent's state and the goal. We compare state-of-the-art model-based and model-free algorithms, NN-MPPI and SAC respectively, to our proposed maximum diffusion (MaxDiff) RL framework (see Supplementary Note~\ref{sec:implementation} for implementation details). \textbf{b,d}, Representative snapshots of MaxDiff RL, NN-MPPI, and SAC agents (top to bottom) in well-conditioned ($\beta=1$) and poorly-conditioned ($\beta=0.001$) controllability settings. \textbf{c}, Even in this simple system, poor controllability can break the performance of RL agents. As $\beta \rightarrow 0$ the system's ability to move in the $x$-direction diminishes, hindering the performance of NN-MPPI and SAC, while MaxDiff RL remains task-capable. For all bar charts, data are presented as mean values above each error bar, where each error bar represents the standard deviation from the mean with $n=1000$ (100 evaluations over 10 seeds for each condition). All differences between MaxDiff RL and comparisons within this figure are statistically significant with $P<0.001$ using an unpaired two-sided Welch's t-test (see Methods and Supplementary Table~\ref{table:stats}).
    }
    \label{fig:fig1}
\end{figure*}

Whether temporal correlations and their impact can be avoided depends on the properties of the underlying agent-environment dynamics. Completely destroying correlations between agent experiences requires the ability to discontinuously jump from state to state without continuity of experience. For some RL agents, this poses no issue. In settings where agents are disembodied, there may be nothing preventing effective exploration through jumps between uncorrelated states. This is one of the reasons why deep RL recommender systems have been successful in a wide range of applications, such as YouTube video suggestions~\cite{Chen2019_youtube,Afsar2022,Chen2023_recommender}. However, continuity of experience is essential to many RL problem domains. For instance, the smoothness of Newton's laws makes correlations unavoidable in the motions of most physical systems, even in simulation. This suggests that for systems like robots or self-driving cars overcoming the impact of temporal correlations on performance presents a major challenge.

To illustrate the impact this can have on learning performance, we devised a toy task to evaluate deep RL algorithms as a function of correlations intrinsic to the agent's state transitions. Our toy task and agent dynamics are shown in Fig.~\ref{fig:fig1}(a), corresponding to a double integrator system with parametrized momentum anisotropy. The task requires learning reward, dynamics, and policy representations from scratch in order to move a planar point mass from a fixed initial position to a goal location. The system's true linear dynamics are simple enough to explicitly write down, which allows us to rigorously study temporal correlations across state transitions by analyzing its controllability. Controllability is a formal property of control systems that describes their ability to reach arbitrary states in an environment~\cite{Sontag2013,Hespanha2018}. In linearizable systems, state transitions become degenerate and irreversibly correlated when they are uncontrollable. However, if the agent is controllable the impact of correlations can be overcome, at least in principle. While the relationship between controllability and temporal correlations has been studied for decades~\cite{Mitra1969}, it is only recently that researchers have begun to study its impact on learning processes~\cite{Dean2020,Tsiamis2021,Tsiamis2022}. 

Figure~\ref{fig:fig1} parametrically explores the relationship between our toy system's controllability properties and the learning performance of state-of-the-art deep RL algorithms. The point mass dynamics are parametrized by $\beta \in [0,1]$, which determines the relative difficulty of translating along the $x$-axis (Fig.~\ref{fig:fig1}(a)). When $\beta = 0$ the system is uncontrollable and can only translate along the $y$-axis, which illustrates the sense in which state transitions become irreversibly correlated. While the system is formally controllable for all non-zero $\beta$, as $\beta\rightarrow 0$ fewer lateral transitions become available for the same range of actions, introducing temporal correlations along the system's $x$-coordinate (see Supplementary Figure~\ref{fig:supp_isotropic_dists}). We evaluated the performance of state-of-the-art model-based and model-free deep RL algorithms on our task---model-predictive path integral control (NN-MPPI)~\cite{Theodorou2017} and soft actor-critic (SAC)~\cite{Haarnoja2018}, respectively---at varying values of $\beta$, from 1 to 0.001. As expected, at $\beta=1$ both NN-MPPI and SAC are able to accomplish the toy task (Fig.~\ref{fig:fig1}(b)). However, as $\beta \rightarrow 0$ the performance of NN-MPPI and SAC degrades parametrically (Fig.~\ref{fig:fig1}(c)), up until the point that neither algorithm can solve the task, as shown in Fig.~\ref{fig:fig1}(d). Hence, temporal correlations can completely hinder the learning performance of the state-of-the-art in deep RL even in toy problem settings such as this one, where a globally optimal policy can be analytically computed in closed form.

Failure to mitigate temporal correlations between state transitions can prevent effective exploration, severely impacting the performance of deep RL agents. As Fig.~\ref{fig:fig1}(d) illustrates, neither NN-MPPI nor SAC agents are able to sufficiently explore in the $x$-dimension of their state space as a result of their decreasing degree of controllability (see Supplementary Note~\ref{sec:controllability}). This is the case despite the fact that NN-MPPI and SAC are both MaxEnt RL algorithms~\cite{Haarnoja2018,So2022}, designed specifically to achieve improved exploration outcomes by decorrelating their agent's action sequences. In contrast, our proposed approach---MaxDiff RL---is able to consistently succeed at the task and is guaranteed to realize effective exploration by focusing instead on decorrelating agent experiences, i.e., their state sequences (see purple in Fig.~\ref{fig:fig1}(b-d)), as we discuss in the following section.

\subsection{Maximum diffusion exploration and learning}
\label{sec:sec2}
Most RL methods presuppose that taking random actions produces effective exploration~\cite{Thrun1992,Amin2021}, and sophisticated techniques like MaxEnt RL are no different. However, as we previously illustrated, whether this is actually possible depends on the agent's controllability properties and the temporal correlations these spontaneously induce in their experiences (see Fig.~\ref{fig:fig2}(c) and Supplementary Note~\ref{sec:controllability}). To overcome these limitations, we propose decorrelating agent experiences as opposed to their action sequences, which forms the starting point to our derivation of the MaxDiff RL framework.

Prior to synthesizing policies or assessing their impact on learning outcomes, we require a formalization of agent experiences. Without considering policies, we see the agent-environment state transition dynamics as an autonomous stochastic process, whose sample paths $x(t)$ take value in a state space $\mathcal{X}\subset\mathbb{R}^d$ at each point in time within an interval $\mathcal{T}=[t_0,t]$. Then, we see agent experiences as collections of random variables parametrized by time, whose realizations $x(t)$ are the sample paths of the underlying agent-environment process. When $\mathcal{T}=\{0,\cdots,T\}$ is discrete, we use $x_{0:T}$ instead of $x(t)$. In this context, the probability density function over state trajectories, $P[x(t)]$ (or $P[x_{0:T}]$), completely characterizes an agent's experiences and their properties (see Supplementary Note~\ref{sec:exploration_sampling}). We may now begin our derivation by asking: What is the most decorrelated that agent experiences can be?

To answer this question, we draw from the statistical physics literature on maximum caliber~\cite{Jaynes1957,Dill2018,Chvykov2021}, which is a generalization of the variational principle of maximum entropy~\cite{Kapur1989}. The goal of a maximum caliber variational optimization is to find the trajectory distribution $P_{max}[x(t)]$, which optimizes an entropy functional $S[P[x(t)]]$. The optimal distribution would then describe the paths of an agent with the least-correlated experiences, but its specific form and properties depend on how the variational optimization is constrained. Without constraints, agents could sample states discontinuously and uniformly in a way that is equivalent to \textit{i.i.d.} sampling but is not consistent with the continuous experiences of embodied agents (Fig.~\ref{fig:fig2}(a,b)). Hence, we tailor our assumptions to agents with continuous experiences. Then, to ensure our optimization produces a distribution over continuous paths, we constrain the volume of states accessible within any finite time interval by bounding their fluctuations (see Supplementary Note~\ref{sec:exploration_opt_problem}).

\begin{figure*}[t!]
    \centering
    \includegraphics[width=1.0\linewidth]{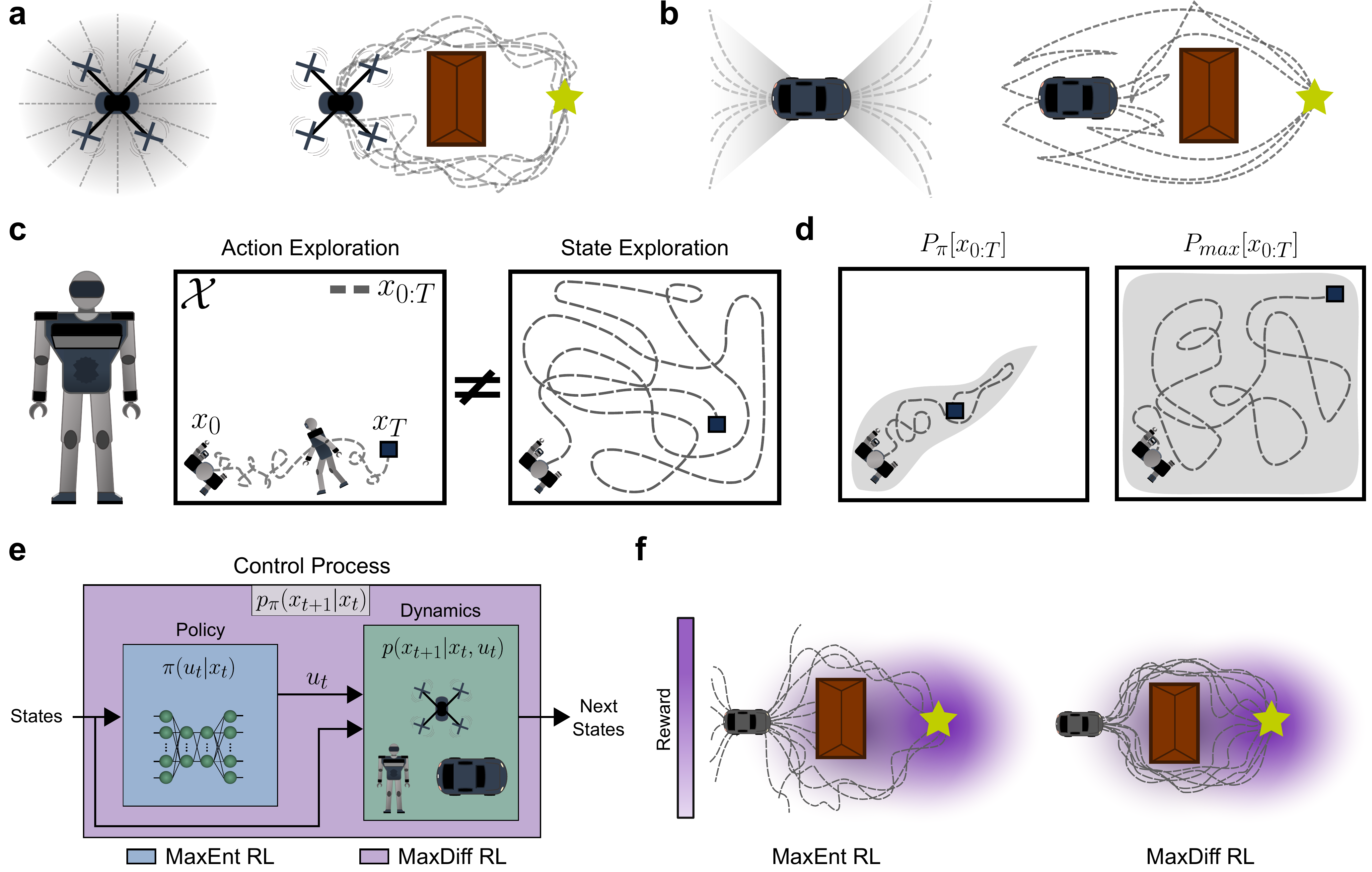}
    \caption{\small
    \textbf{Maximum diffusion RL mitigates temporal correlations to achieve effective exploration}. \textbf{a,b}, Systems with different planar controllability properties. \textbf{c}, Whether action randomization leads to effective state exploration depends on the properties of the underlying state-transition dynamics (see Supplementary Note~\ref{sec:controllability}), as in our illustration of a complex bipedal robot falling over and failing to explore. \textbf{d}, While any given policy induces a path distribution (left), MaxDiff RL produces policies that maximize the path distribution's entropy (right). The projected support of the robot's path distribution is illustrated by the shaded gray region. We prove that maximizing the entropy of an agent's state transitions results in effective exploration (see Supplementary Notes~\ref{sec:exploration_diffusion} and~\ref{sec:maxdiff_exploration}). \textbf{e}, Our approach generalizes the MaxEnt RL paradigm by considering agent trajectories. We prove that maximizing a policy's entropy does not generally maximize the entropy of an agent's state transitions (see Supplementary Note~\ref{sec:maxdiff_RL}). \textbf{f}, This approach leads to distinct learning outcomes because agents reason about the impact of their actions on state transitions, rather than their actions alone.
    }
    \label{fig:fig2}
\end{figure*}

Surprisingly, this constrained variational optimization admits an analytical solution for the maximum entropy path distribution. The derived optimal path distribution is
\begin{equation}\label{eq:undirectedpathdist}
    P_{max}[x(t)] = \frac{1}{Z} \exp\Big[-\frac{1}{2}\int_{t_0}^{t} \dot{x}(\tau)^T\mathbf{C}^{-1}[x(\tau)]\dot{x}(\tau) d\tau\Big],
\end{equation}
where $Z$ is a normalization constant. At every point in space $x^*\in\mathcal{X}$, the matrix $\mathbf{C}[x^*]$ measures temporal correlations locally over an interval of duration $\Delta t$, such that
\begin{equation}\label{eq:temporalcorrelations}
    \mathbf{C}[x^*] = \int_{t_i}^{t_i+\Delta t} K_{XX}(t_i,\tau)d\tau,
\end{equation}
where $K_{XX}(t_1,t_2)$ is the autocovariance of $x(t)$ at pairs of samples in time evaluated over a chosen interval, $[t_i, t_i+\Delta t]\subset \mathcal{T}$, with a given $x(t_i)=x^*$ (see Supplementary Note~\ref{sec:exploration_sampling}). This distribution describes the trajectories of an optimal agent with minimally correlated paths, subject to the constraints imposed by continuity of experience. Moreover, Eq.~\ref{eq:undirectedpathdist} is equivalent to the path distribution of an anisotropic, spatially-inhomogeneous diffusion process. Thus, minimizing correlations among agent trajectories leads to diffusion-like exploration, whose properties can actually be analyzed using statistical mechanics (see Supplementary Figure~\ref{fig:supp_diffusion}). This also means that the sample paths of the optimal agent are Markovian and ergodic (see Supplementary Notes~\ref{sec:exploration_diffusion} and~\ref{sec:directed_exploration} for associated theorems, corollaries, and their proofs). Unlike alternative RL frameworks, our approach does not assume the Markov property, but rather enforces it as a property intrinsic to the optimal agent's path distribution.

Satisfying ergodicity has profound implications for the properties of resulting agents. Ergodicity is a formal property of dynamical systems that guarantees that the statistics of individual trajectories are asymptotically equivalent to those of a large ensemble of trajectories~\cite{Moore2015,Taylor2021}. Put in RL terms, while the sequential nature of RL agent experiences can make \textit{i.i.d.} sampling technically impossible, the global statistics of an ergodic RL agent are indistinguishable from those of an \textit{i.i.d.} sampling process. In this sense, ergodic Markov sampling is the best possible alternative to \textit{i.i.d.} sampling in sequential decision-making processes. Beyond resolving the issue of generating \textit{i.i.d.} samples in RL, ergodicity forms the basis of many of MaxDiff RL's theoretical guarantees, as we show in the following sections.

When an agent's trajectories satisfy Eq.~\ref{eq:undirectedpathdist}, we describe the agent as maximally diffusive. However, agents do not realize maximally diffusive trajectories spontaneously. Doing so requires finding a policy capable of satisfying maximally diffusive path statistics, which forms the core of what we term MaxDiff RL. While any given policy induces a path distribution, finding policies that realize maximally diffusive trajectories requires optimization and learning (Fig.~\ref{fig:fig2}(d)). To this end, we define:
\begin{equation}
    \label{eq:path_policydists}
    \begin{split}
    P_{\pi}[x_{0:T},u_{0:T}] &= \prod_{t=0}^{T-1} p(x_{t+1}|x_t,u_t)\pi(u_t|x_t) \\
    P_{max}^{r}[x_{0:T},u_{0:T}] &= \prod_{t=0}^{T-1} p_{max}(x_{t+1}|x_t) e^{r(x_t,u_t)},
    \end{split}
\end{equation}
where we discretized the distribution in Eq.~\ref{eq:undirectedpathdist} as $p_{max}(x_{t+1}|x_t)$, and analytically rederived the optimal path distribution under the influence of a reward landscape, $r(x_t,u_t)$ (see Supplementary Note~\ref{sec:directed_exploration}). Given the distributions in Eq.~\ref{eq:path_policydists}, the goal of MaxDiff RL can be framed as minimizing the Kullback-Leibler (KL) divergence between them---that is, between the agent's current path distribution and the maximally diffusive one. 

To draw connections between our framework and the broader MaxEnt RL literature, we recast the KL-divergence formulation of MaxDiff RL as an equivalent stochastic optimal control (SOC) problem. In SOC, the goal is to find a policy that maximizes the expected cumulative rewards of an agent in an environment. In this way, we can express the MaxDiff RL objective as
\begin{equation}\label{eq:soc_maxdiff}
    \pi_{\text{MaxDiff}}^* = \underset{\pi}{\text{argmax}}  \ E_{(x_{0:T},u_{0:T})\sim P_{\pi}} \Bigg[\sum_{t=0}^{T-1} \gamma^t\hat{r}(x_t,u_t) \Bigg],
\end{equation}
with $\gamma\in[0,1)$ and modified rewards given by   
\begin{equation}\label{eq:soc_maxdiff_running_cost}
    \hat{r}(x_t,u_t) = r(x_t,u_t)-\alpha \log \frac{p(x_{t+1}|x_t,u_t)\pi(u_t|x_t)}{p_{max}(x_{t+1}|x_t)},
\end{equation}
where $\alpha>0$ is a temperature-like parameter we introduce to balance diffusive exploration and reward exploitation, as we discuss in the following section. With these results in hand, we may now state one of our main theorems.
\begin{theorem}
    \label{thm:thm1}
    (MaxDiff RL generalizes MaxEnt RL) Let the state transition dynamics due to a policy $\pi$ be $p_{\pi}(x_{t+1}|x_t)=E_{u_t\sim\pi}[ p(x_{t+1}|x_t,u_t)]$. If the state transition dynamics are assumed to be decorrelated, then the optimum of Eq.~\ref{eq:soc_maxdiff} is reached when $D_{KL}(p_{\pi}||p_{max}) = 0$ and the MaxDiff RL objective reduces to the MaxEnt RL objective.
\end{theorem}
\noindent Proving this result is simple and only relies on the sense in which state transitions are decorrelated, which we discuss in detail in Supplementary Note~\ref{sec:maxdiff_RL}. 

Completely destroying temporal correlations generally requires discontinuous jumps between states, which can only be achieved by fully controllable agents~\cite{Rawlik2012}. When an agent is fully controllable, there always exists a policy that enables it to take any arbitrary path through state space. If this condition is met, then the optimum of Eq.~\ref{eq:soc_maxdiff} is attained when $p_{\pi^{max}}(x_{t+1}|x_t)=p_{max}(x_{t+1}|x_t)$ at each point in time, where actions drawn from an optimized policy $\pi^{max}$. In turn, this simplifies Eq.~\ref{eq:soc_maxdiff_running_cost} and recovers the MaxEnt RL objective~\cite{Haarnoja2018}, as shown in Supplementary Note~\ref{sec:maxdiff_RL}. This proves not only that MaxDiff RL is a generalization of the MaxEnt RL framework to agents with correlations in their state transitions, but also makes clear that maximizing policy entropy cannot decorrelate agent experiences in general. In contrast, MaxDiff RL actively enforces path decorrelation at all points in time. We can think of this intuitively by noting that MaxDiff RL simultaneously accounts for the effect of the policy and of the temporal correlations induced by agent-environment dynamics in its optimization (Fig.~\ref{fig:fig2}(e)). As such, MaxDiff RL typically produces distinct learning outcomes from MaxEnt RL (Fig.~\ref{fig:fig2}(f)). Our result also implies that all theoretical robustness guarantees of MaxEnt RL (e.g.,~\cite{Eysenbach2022}) should be interpreted as guarantees of MaxDiff RL when state transitions are decorrelated. Moreover, we suggest that many of the gaps between MaxEnt RL's theoretical results and their practical performance may be explained by the impact of temporal correlations, as we saw in Fig.~\ref{fig:fig1}.

While these results seem to suggest that model-free implementations of MaxDiff RL are not feasible, we note that local estimates of the agent's path entropy can be learned from observations. This effectively reinterprets temporal correlations as a state-dependent property of the environment (see Supplementary Note~\ref{sec:exploration_controllers}). Similar entropy estimates have been used in model-free RL~\cite{Seo2021} and more broadly in the autoencoder literature~\cite{Prabhakar2022}. For the results in this manuscript, we derived a model-agnostic objective using an analytical expression for the local path entropy,
\begin{equation}
    \label{eq:maxdiff_ratt}
    \underset{\pi}{\text{argmax}}  \  E_{(x_{0:T},u_{0:T})\sim P_{\pi}}\Bigg[ \sum_{t=0}^{T-1} r(x_t,u_t)+\frac{\alpha}{2} \log \det \text{\textbf{C}}[x_t]\Bigg],
\end{equation}
whose optimum realizes the same optimum as Eq.~\ref{eq:soc_maxdiff}, and where we omitted $\gamma$. There are many ways to express the MaxDiff RL objective, each of which may have implementation-specific advantages (see Fig.~\ref{fig:fig3}(a) and Supplementary Note~\ref{sec:exploration_entropymax}). In this sense, MaxDiff RL is not a specific algorithm implementation but rather a general problem statement and solution framework, similar to MaxEnt RL. In this work, our MaxDiff RL implementation is exactly identical to NN-MPPI except for the path entropy term shown above. However, this simple modification can have a drastic effect on agent outcomes.

\begin{figure*}[t!]
    \centering
    \includegraphics[width=1.0\linewidth]{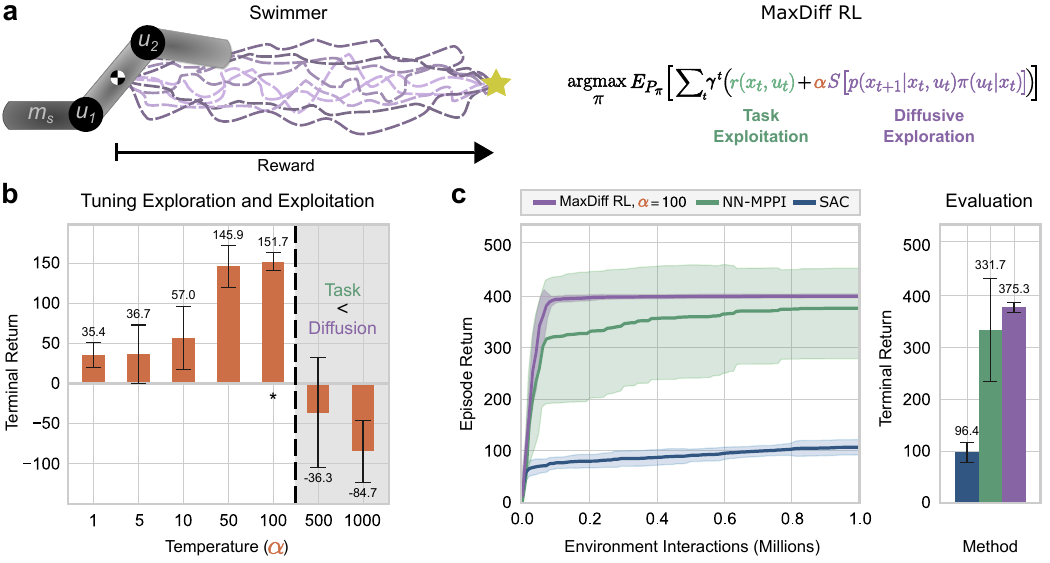}
    \caption{\small
    \textbf{Maximally diffusive RL agents are robust to random seeds and initializations}. \textbf{a}, Illustration of MuJoCo swimmer environment (left panel). The swimmer has 2 degrees of actuation, $u_1$ and $u_2$, that rotate its limbs at the joints, with tail mass $m_s$ and $m=1$ for other limbs. MaxDiff RL synthesizes robust agent behavior by learning policies that balance task-capability and diffusive exploration (right panel). In practice this balance is tuned by a temperature-like parameter, $\alpha$. \textbf{b}, To explore the role that $\alpha$ plays in the performance of MaxDiff RL, we examine the terminal returns of swimmer agents (10 seeds each) across values of $\alpha$ with $m_s=1$. Diffusive exploration leads to greater returns until a critical point (inset dotted line), after which the agent starts valuing diffusing more than accomplishing the task (see also \href{https://www.youtube.com/watch?v=XZOTG9KNifs&list=PLO5AGPa3klrCTSO-t7HZsVNQinHXFQmn9&index=1}{Supplementary Movie~1}). \textbf{c}, Using $\alpha=100$, we compared MaxDiff RL against SAC and NN-MPPI with $m_s=0.1$. We observe that MaxDiff RL outperforms comparisons on average with near-zero variability across random seeds, which is a formal property of MaxDiff RL agents (see also \href{https://www.youtube.com/watch?v=eq6Fk-lp1i0&list=PLO5AGPa3klrCTSO-t7HZsVNQinHXFQmn9&index=2}{Supplementary Movie~2}). For all reward curves, the shaded regions correspond to the standard deviation from the mean across 10 seeds. For all bar charts, data are presented as mean values above each error bar, where each error bar represents the standard deviation from the mean with $n=1000$ (100 evaluations over 10 seeds for each condition). All differences between MaxDiff RL and comparisons within this figure are statistically significant with $P<0.001$ using an unpaired two-sided Welch's t-test (see Methods and Supplementary Table~\ref{table:stats}).
    }
    \label{fig:fig3}
\end{figure*}

\subsection{Robustness to initializations in ergodic agents}
\label{sec:sec3}
The introduction of an entropy term in Eq.~\ref{eq:maxdiff_ratt} means that MaxDiff RL agents must balance between two aims: Achieving the task and embodying diffusion (Fig.~\ref{fig:fig3}(a)). While asymptotically there is no trade-off between maximally diffusive exploration and task exploitation, managing the relative balance between these two aims is important over finite time horizons, which we achieve with a temperature-like parameter, $\alpha$. In practice, our entropy term plays a similar role as other exploration bonuses that reward agent curiosity or provide intrinsic motivation~\cite{Chentanez2004,Pathak2017,Taiga2020}. Unlike other bonuses, however, the role of path entropy can be interpreted through the lens of statistical mechanics. If $\alpha$ is set too high, the system's fluctuations can overpower the reward and break the agent's ergodicity in ways that resemble the physics of diffusion processes in potential fields~\cite{Wang2019}. Unfortunately, predicting where this critical $\alpha$ threshold lies is generally challenging due to its conceptual ties to the phenomenon of ergodicity-breaking in nonequilibrium processes~\cite{Palmer1982}. 

Since ergodicity provides many of MaxDiff RL's desirable properties and guarantees, tuning the value of $\alpha$ is essential. In Fig.~\ref{fig:fig3} and \href{https://www.youtube.com/watch?v=XZOTG9KNifs&list=PLO5AGPa3klrCTSO-t7HZsVNQinHXFQmn9&index=1}{Supplementary Movie~1}, we explore the effect of tuning $\alpha$ on the learning performance of MaxDiff RL agents in MuJoCo's swimmer environment. The swimmer system is comprised of three rigid links of nominally equal mass, $m=1$, with two degrees of actuation at the joints. The agent's objective is to swim as fast as possible within a fixed time interval, while being subjected to viscous drag forces (Fig.~\ref{fig:fig3}(a)). In Fig.~\ref{fig:fig3}(b), we vary $\alpha$ across multiple orders of magnitude and examine its impact on the terminal returns of MaxDiff RL swimmer agents. As we modulate the value of $\alpha$ from 1 to 100, we observe that diffusive exploration leads to greater returns. However, after $\alpha=100$ we cross the critical threshold beyond which the strength of the system's diffusive exploration overpowers the reward (see inset dotted line in Fig~\ref{fig:fig3}(b)), thereby breaking the ergodicity of our agents with respect to the underlying potential and performing poorly at the task---just as predicted by our theoretical framework.

Given a constant temperature of $\alpha=100$ that preserves the swimmer's ergodicity, we compared the performance of MaxDiff RL to NN-MPPI and SAC across 10 seeds each. To ensure the task was solvable by all agents, we lowered the mass of the swimmer's third link (i.e., its tail) to $m_s=0.1$. We find that while SAC struggles to succeed within a million environment interactions, NN-MPPI achieves good performance but with high variance across seeds. This is in stark contrast to MaxDiff RL, whose performance is near-identical and competitive across all random seeds (see Fig.~\ref{fig:fig3}(c) and \href{https://www.youtube.com/watch?v=eq6Fk-lp1i0&list=PLO5AGPa3klrCTSO-t7HZsVNQinHXFQmn9&index=2}{Supplementary Movie~2}). Hence, by decorrelating state transitions, our agent was able to exhibit robustness to seeds and environment randomization beyond what is typically possible in deep RL. Moreover, since our implementation of MaxDiff RL is identical to that of NN-MPPI, we can attribute any performance gains and added robustness to the properties of MaxDiff RL's theoretical framework. 

\begin{figure*}[t!]
    \centering
    \includegraphics[width=0.9\linewidth]{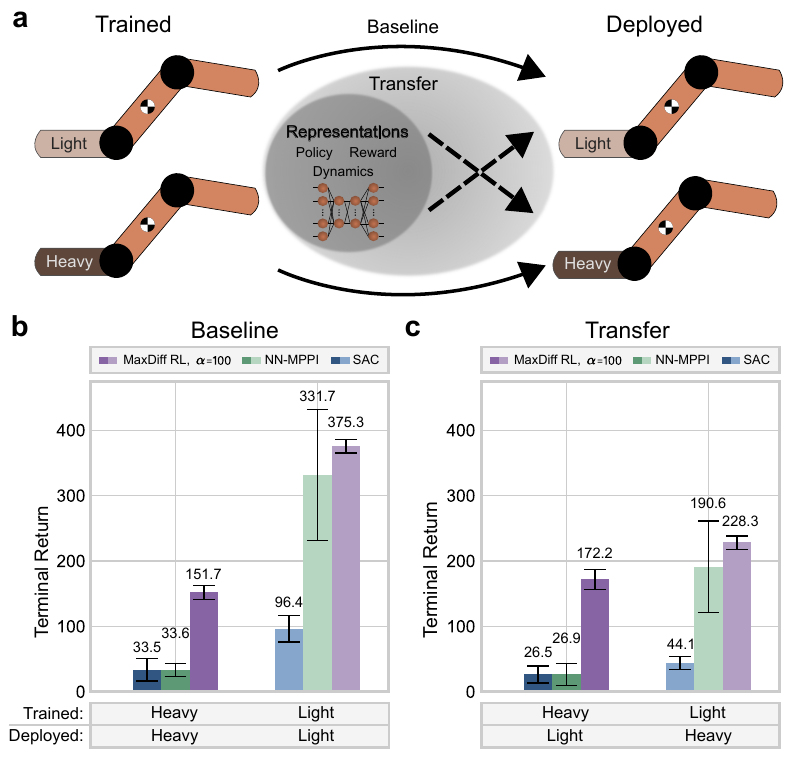}
    \caption{\small
    \textbf{Trained system embodiment determines deployed system performance}. \textbf{a}, Two variants of the MuJoCo swimmer environment: One with $m_s=1$ and one with $m_s=0.1$. As a baseline, we deploy learned representations on the same swimmer variant trained on. Then, we carry out a transfer experiment where the trained and deployed swimmer variants are swapped. \textbf{b}, Baseline experiments confirm previous results: All algorithms benefit from a more controllable swimmer. Since MaxDiff RL optimizes system controllability, it is the only method capable of achieving the task with a heavy-tailed swimmer (see also \href{https://www.youtube.com/watch?v=eq6Fk-lp1i0&list=PLO5AGPa3klrCTSO-t7HZsVNQinHXFQmn9&index=2}{Supplementary Movie~2}). \textbf{c}, Both NN-MPPI and SAC performance degrades when deployed on a more controllable system than was trained on, which is undesirable. In contrast, MaxDiff RL benefits from the ``Heavy-to-Light'' transfer because it learns policies that take advantage of a more capable system during deployment. We also observe that MaxDiff RL performance further increases in the ``Light-to-Heavy'' transfer experiment, showing that system controllability during training is more important to overall performance than the particular embodiment of the system it is ultimately deployed on (see also \href{https://www.youtube.com/watch?v=UD08f2aYIjM&list=PLO5AGPa3klrCTSO-t7HZsVNQinHXFQmn9&index=3}{Supplementary Movie~3}). For all bar charts, data are presented as mean values above each error bar, where each error bar represents the standard deviation from the mean with $n=1000$ (100 evaluations over 10 seeds for each condition). All differences between MaxDiff RL and comparisons within this figure are statistically significant with $P<0.001$ using an unpaired two-sided Welch's t-test (see Methods and Supplementary Table~\ref{table:stats}).
    }
    \label{fig:fig4}
\end{figure*}

Robustness to random seeds and environmental randomizations is a highly desirable feature of deep RL agents~\cite{Islam2017,Henderson2018,Moos2022}. However, guaranteeing such robustness is challenging because it requires modeling the impact of neural representations on learning outcomes. Nonetheless, we can provide representation-agnostic guarantees through the probably approximately correct in Markov decision processes (PAC-MDP) learning framework~\cite{Strehl2006,Strehl2009}. In short, an algorithm is PAC-MDP if it is capable of generating policies that are at least $\epsilon$-optimal at least $100\times(1-\delta)\%$ of the time, for any $\epsilon>0$ and $\delta\in(0,1)$ (see Methods). Under this framework, we can provide formal robustness guarantees.
\begin{theorem}
    \label{thm:thm2}
    (MaxDiff RL agents are robust to random seeds) If there exists a PAC-MDP algorithm $\mathcal{A}$ with policy $\pi^{max}$ for the MaxDiff RL objective (Eq.~\ref{eq:soc_maxdiff}), then the Markov chain induced by $\pi^{max}$ is ergodic, and $\mathcal{A}$ will be asymptotically $\epsilon$-optimal regardless of initialization.
\end{theorem}
\noindent We refer the reader to Supplementary Note~\ref{sec:maxdiff_RL} for details, but the proof follows from treating the condition for PAC-MDP learnability as an observable in Birkhoff's ergodic theorem~\cite{Moore2015}. Since maximally diffusive agents are ergodic, any two arbitrary initializations will asymptotically achieve identical learning outcomes, which implies robustness to random seeds and environmental stochasticity. Despite excluding neural representations from our analysis, Fig.~\ref{fig:fig3}(c) suggests that our guarantees hold empirically.

\subsection{Zero-shot generalization across embodiments}
\label{sec:sec4}
When agents can find optimal policies, their dynamics become indistinguishable from an ergodic diffusion process. In doing so, the MaxDiff RL objective (see Eq.~\ref{eq:soc_maxdiff_running_cost}) reduces the influence of agent dynamics on performance. This suggests that successful MaxDiff RL policies may exhibit favorable generalization properties across agent embodiments. To explore this possibility, as well as the robustness of MaxDiff RL agents to variations in their neural representations, we devised a transfer experiment in the MuJoCo swimmer environment. We designed two variants of the swimmer: One with a heavy, less controllable tail of $m_s=1$, and another with a light, more controllable tail of $m_s=0.1$ (Fig.~\ref{fig:fig4}(a)). We trained two sets of representations for each algorithm. One set was trained with the light-tailed swimmer, and another set was trained with the heavy-tailed swimmer. Then, we deployed and evaluated each set of representations on both the swimmer variant that they observed during training, as well as its counterpart. Our experiment's outcomes are shown in Fig.~\ref{fig:fig4}(b,c), where the results are categorized as ``baseline'' if the trained and deployed swimmer variants match, or ``transfer'' if they were swapped. The baseline experiments validate other results shown throughout the manuscript: All algorithms benefit from working with a more controllable system whose dynamics induce weaker temporal correlations (see Fig.~\ref{fig:fig4}(b) and \href{https://www.youtube.com/watch?v=eq6Fk-lp1i0&list=PLO5AGPa3klrCTSO-t7HZsVNQinHXFQmn9&index=2}{Supplementary Movie~2}). However, as MaxDiff RL is the only approach taking temporal correlations into account, it is the only method that remains task-capable with a heavy-tailed swimmer.

For the transfer experiments, all of the learned neural representations of the reward function, control policy, and agent dynamics were deployed on the swimmer variant that was not seen during training (Fig.~\ref{fig:fig4}(a)). First, we note that for both NN-MPPI and SAC representation transfer leads to degrading performance across the board. This is the case even when the swimmer variant they were deployed onto was more controllable, which is counterintuitive and undesirable behavior. In contrast, our MaxDiff RL agents can actually benefit and improve their performance when deployed on the more controllable swimmer variant, as desired (see ``Heavy-to-Light'' transfer in Fig.~\ref{fig:fig4}(c) and \href{https://www.youtube.com/watch?v=UD08f2aYIjM&list=PLO5AGPa3klrCTSO-t7HZsVNQinHXFQmn9&index=3}{Supplementary Movie~3}). In other words, as the task becomes easier in this way, we can expect the performance of MaxDiff RL agents to improve. 

A more surprising result is the performance increase in MaxDiff RL agents between the baseline heavy-tailed swimmer and the ``Light-to-Heavy'' transfer swimmer (Fig.~\ref{fig:fig4}(c) and \href{https://www.youtube.com/watch?v=UD08f2aYIjM&list=PLO5AGPa3klrCTSO-t7HZsVNQinHXFQmn9&index=3}{Supplementary Movie~3}). We found that training with a more controllable swimmer increased the performance of agents when deployed on a heavy-tailed swimmer, showing that system controllability during training matters more to overall performance than the particular embodiment of the deployed system. This kind of zero-shot generalization~\cite{Kirk2023} from an easier task to a more challenging task is reminiscent of results seen in RL agents trained via curriculum learning~\cite{Oh2017}, as well as of the incremental learning dynamics of biological systems during motor skill acquisition~\cite{Krakauer2019}. However, here it emerges spontaneously from the properties of MaxDiff RL agents. In part, this occurs because greater controllability leads to improved exploration, which increases the diversity of data observed during training.

\subsection{Single-shot learning in ergodic agents}
\label{sec:sec5}
When agents are deployed in the real world, they face situations at test time that were never encountered during training. Since exhaustively accounting for every possible scenario is infeasible, agents capable of real-time adaptation and learning during individual deployments are desirable~\cite{Ibarz2021}. Most RL methods excel at episodic multi-shot learning over the course of several deployments (Fig.~\ref{fig:fig5}(b)), where randomized instantiations of a given task and environment passively provide a kind of variability that is essential to the learning process~\cite{Lu2021_reset}. However, episodic problems of this kind are very rare in real-world applications. For this reason, there is a need for methods that allow agents to perform a task successfully within a single trial---or, in other words, for methods that enable single-shot learning.

Single-shot learning concerns learning in non-episodic environments over the course of a single task attempt, similar to the ``single-life'' RL setting considered in~\cite{Chen2022_yolo}. Despite the challenges associated with studying the behavior of agents based on neural network representations, the ergodic properties of MaxDiff RL enables one to provide representation-agnostic guarantees on the feasibility of single-shot learning through the PAC-MDP learning framework.
\begin{theorem}
    \label{thm:thm3}
    (MaxDiff RL agents can learn in single-shot deployments) If there exists a PAC-MDP algorithm $\mathcal{A}$ with policy $\pi^{max}$ for the MaxDiff RL objective (Eq.~\ref{eq:soc_maxdiff}), then the Markov chain induced by $\pi^{max}$ is ergodic, and any individual initialization of $\mathcal{A}$ will asymptotically satisfy the same $\epsilon$-optimality as an ensemble of initializations. 
\end{theorem}
\noindent Thus, any MaxDiff RL agent capable of solving a task in a multi-shot fashion (in the PAC-MDP sense) is capable of solving the same task in a single-shot fashion. This theorem also follows from Birkhoff's ergodic theorem and is closely related to Theorem~\ref{thm:thm2}. Since any two MaxDiff RL agents will asymptotically achieve identical learning outcomes, any individual MaxDiff RL agent will also achieve identical learning outcomes as an ensemble (see Supplementary Note~\ref{sec:maxdiff_RL} for details). Because ergodicity is central to this proof, we expect its guarantees to fail when ergodicity is broken by either the agent or the environment.

\begin{figure*}[t!]
    \centering
    \includegraphics[width=1.0\linewidth]{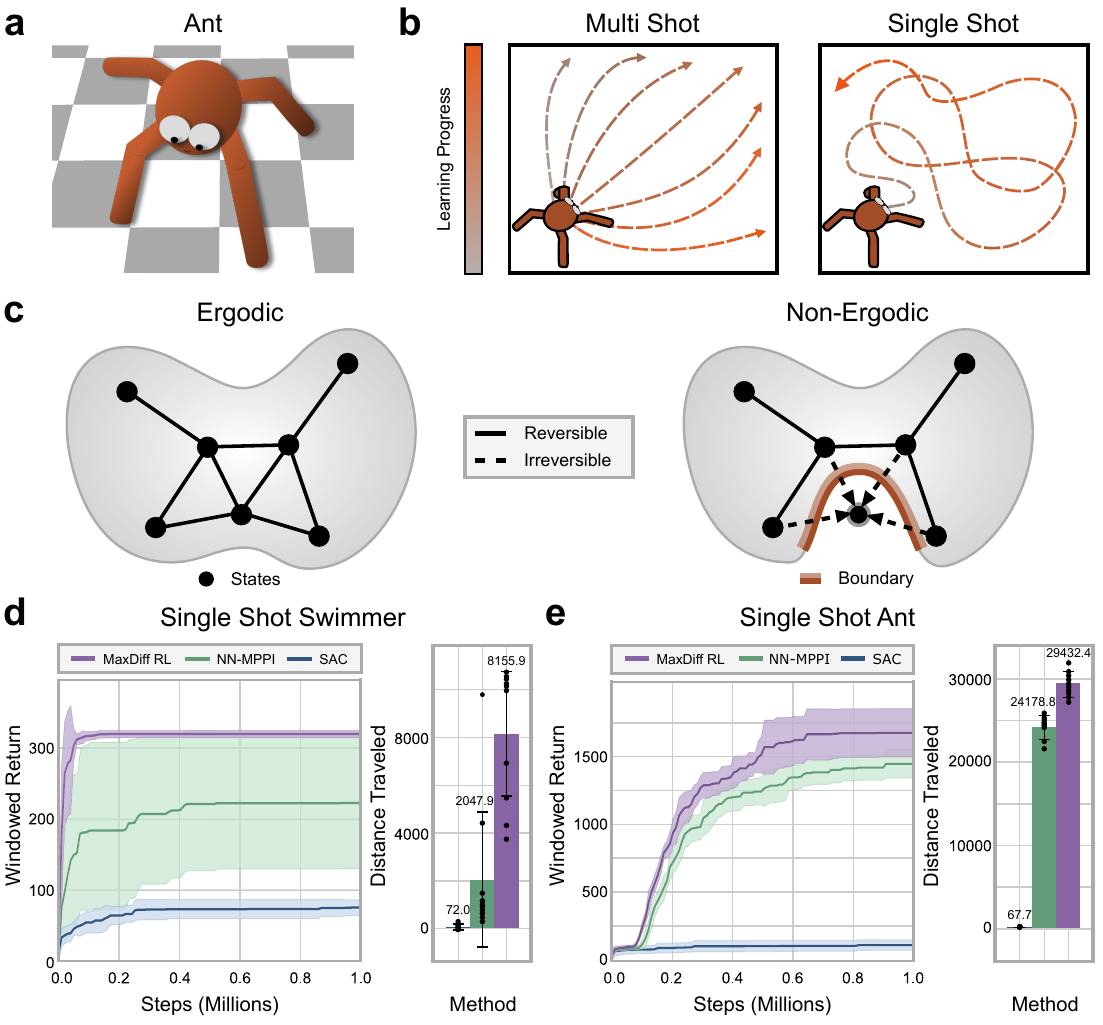}
    \caption{\small
     \textbf{Maximally diffusive RL agents are capable of single-shot learning.} \textbf{a}, Illustration of MuJoCo ant environment. \textbf{b}, Typical algorithms learn across many different initializations and deployments of an agent, which is known as multi-shot learning. In contrast, single-shot learning insists on a single task attempt, which requires learning through continuous deployments. Here, we prove that MaxDiff RL agents are equivalently capable of single-shot and multi-shot learning in a broad variety of settings. \textbf{c}, Single-shot learning depends on the ability to generate data samples ergodically, which MaxDiff RL guarantees when there are no irreversible state transitions in the environment. \textbf{d}, Single-shot learning in the swimmer MuJoCo environment. We find that MaxDiff RL achieves robust performance comparable to its multi-shot counterpart (see also \href{https://www.youtube.com/watch?v=kgdZeSoqaKk&list=PLO5AGPa3klrCTSO-t7HZsVNQinHXFQmn9&index=4}{Supplementary Movie~4}). \textbf{e}, In contrast to the swimmer, the MuJoCo ant environment contains irreversible state transitions (e.g., flipping upside down) preventing ergodic trajectories. Nonetheless, MaxDiff RL remains state-of-the-art in single-shot learning. Note that we report returns over a window of 1000 steps in analogy to our multi-shot results, where episodes consist of 1000 environment interactions. For all reward curves, the shaded regions correspond to the standard deviation from the mean across 10 seeds. For all bar charts, data are presented as mean values above each error bar, where each error bar represents the standard deviation from the mean and the data distribution is plotted directly ($n=10$ seeds for each condition). All differences between MaxDiff RL and comparisons within this figure are statistically significant with $P<0.001$ using an unpaired two-sided Welch's t-test (see Methods and Supplementary Table~\ref{table:stats}).
    }
    \label{fig:fig5}
\end{figure*}

Figure~\ref{fig:fig5} demonstrates the single-shot learning capabilities of MaxDiff RL agents, and explores what happens when ergodicity is broken by the topological properties of the environment. Here, we examine both the MuJoCo swimmer and ant environments (Fig.~\ref{fig:fig5}(a)). The primary difference between these two environments is the existence of irreversible state transitions that can violate the ergodicity requirement of our single-shot learning guarantees topologically (Fig.~\ref{fig:fig5}(c)), which have been previously referred to as ``sink states'' in the literature~\cite{Lu2021_reset}. Unlike the swimmer, the ant is capable of transitioning into such states by flipping upside down, thereby breaking ergodicity. Irreversible state transitions are common in real-world applications because they can arise as a result of unsafe behavior, such as a robot breaking or malfunctioning during learning. While such transitions can be prevented in principle through the use of safety-preserving methods~\cite{Ames2014,Ames2020,Xiao2023}, we omit their implementation to illustrate our point. As expected, the MaxDiff RL single-shot swimmer is capable of learning in continuous deployments (see Fig.~\ref{fig:fig5}(d) and \href{https://www.youtube.com/watch?v=kgdZeSoqaKk&list=PLO5AGPa3klrCTSO-t7HZsVNQinHXFQmn9&index=4}{Supplementary Movie~4}), retaining the same robustness of its multi-shot counterpart in Fig.~\ref{fig:fig3}(c), and achieving similar task performance. Despite ergodicity-breaking in the single-shot ant environment, MaxDiff RL still leads to improved outcomes over NN-MPPI and SAC, as in Fig.~\ref{fig:fig5}(e), where we plot the final distance traveled to ensure that no reward hacking took place. However, the loss of ergodicity leads to an increase in the variance of single-shot MaxDiff RL agent performance, as well as equivalent performance to NN-MPPI in multi-shot (see Supplementary Figure~\ref{fig:ant_multi}), which we expect as a result of our robustness guarantees no longer holding.

\section{Discussion}
\label{sec:discussion}
Throughout this work, we have highlighted the ways in which RL is fragile to temporal correlations intrinsic to many sequential decision-making processes. We introduced a framework based on the statistical mechanics of ergodic processes to overcome these limitations, which we term MaxDiff RL. Our framework offers a generalization of the current state-of-the-art in RL and addresses many foundational issues holding back the field: The ergodicity of MaxDiff RL agents enables data acquisition that is indistinguishable from \textit{i.i.d.} sampling, performance that is robust to seeds, and single-shot learning. Through its roots in statistical physics, our work forms a starting point for a more scientific study of embodied RL---one in which falsifiable predictions can be made about agent properties and their performance. 

However, much more work at the nexus of physics, learning, and control remains to be done in pursuit of this goal. For one, approaches grounded in statistical physics for tuning or annealing temperature-like parameters during learning will be necessary to achieve effective exploration without sacrificing agent performance~\cite{Seung1992}. Additionally, control techniques capable of enforcing ergodicity in the face of environmental irreversibility are needed to guarantee desirable agent properties like robustness to random seeds in complex problem settings~\cite{Taylor2021}. Beyond RL, our work also has the potential to open new lines of interdisciplinary inquiry in areas such as biological learning and animal behavior. For example, the importance of ergodicity to animal behaviors like foraging and tracking has been extensively studied~\cite{Chen2020_ergodic}. As such, our work presents an avenue for studying these behaviors within an RL framework that is sensitive to physical embodiment. For biological motor learning, our findings also suggest that controllability may be a promising frame of reference for studying motor skill acquisition~\cite{Song2021}. More broadly, our work is particularly well-suited to applications in soft matter systems where the impact of correlations may in fact be impossible to avoid entirely~\cite{Berrueta2024}. Taken together, our results present a major advance towards transparently understanding and reliably synthesizing complex behavior in embodied decision-making agents, which will be crucial to the long-term viability of deep RL as a field.

\section*{Methods}
\label{sec:methods}
\subsection*{Reinforcement learning preliminaries}
RL problems are modeled as Markov decision processes (MDPs). MDPs are typically defined according to a 5-tuple, $(\mathcal{X},\mathcal{U},p,r,\gamma)$, where we take both the state space, $\mathcal{X}$, and the action space, $\mathcal{U}$, to be continuous. Note that in this work, we typically take $\mathcal{X}$ to be some subset of $\mathbb{R}^d$. Then, $p: \mathcal{X}\times\mathcal{X}\times\mathcal{U}\rightarrow[0,\infty)$ represents the probability density of transitioning from state $x_t\in\mathcal{X}$ to state $x_{t+1}\in\mathcal{X}$ after taking action $u_t\in\mathcal{U}$. At every state and for each action taken, the environment emits a bounded reward $r:\mathcal{X}\times\mathcal{U}\rightarrow [r_{min},r_{max}]$ discounted by a factor of $\gamma\in[0,1)$. In general, the goal is to learn an optimized policy $\pi : \mathcal{U}\times\mathcal{X}\rightarrow [0,\infty)$ capable of producing actions that maximize an agent's expected cumulative rewards over the course of $T$ discrete time stages, where $t\in\{0,\cdots,T\}$. In standard RL, this optimization takes place over the course of ensembles of episodes (i.e., task attempts) of duration $T$, where the environment is reset after each episode~\cite{Sutton2018}. This is what we refer to as the multi-shot learning setting. In contrast, non-episodic RL considers reset-free learning over the course of a single task attempt in the limit of $T\rightarrow\infty$, or until the task is done~\cite{Chen2022_yolo,Lu2021_reset}. We refer to this as the single-shot learning setting.

\subsection*{PAC-MDP framework}
Many properties of MaxDiff RL agents arise from the relationship between ergodicity and learning performance. To formalize how this is the case, we use the probably approximately correct in Markov decision processes (PAC-MDP) learning framework~\cite{Strehl2006,Strehl2009}.
\begin{definition}
    \label{def:pac-mdp}
    An algorithm $\mathcal{A}$ is said to be PAC-MDP if, for any $\epsilon>0$ and $\delta \in (0,1)$, a policy $\pi$ can be produced with $poly(|\mathcal{X}|, |\mathcal{U}|, 1/\epsilon, 1/\delta, 1/(1-\gamma))$ sample complexity that is at least $\epsilon$-optimal with probability at least $1-\delta$. In other words, if $\mathcal{A}$ satisfies
    \begin{equation}
        \label{eq:pac_mdp_condition}
        {\normalfont\text{Pr}}\big(\mathcal{V}_{\pi^*}(x_0)-\mathcal{V}_{\pi}(x_0)\leq \epsilon\big) \geq 1-\delta\notag
    \end{equation}
    with polynomial sample complexity for all $x_0\in\mathcal{X}$, where $\mathcal{V}_{\pi}(\cdot)$ is a value function due to policy $\pi$, and $\mathcal{V}_{\pi^*}(\cdot)$ is the optimal value function, then $\mathcal{A}$ is PAC-MDP.
\end{definition}
\noindent Thus, an algorithm is PAC-MDP if it is capable of producing a policy that is at least $\epsilon$-optimal at least $100\times(1-\delta)\%$ of the time for any valid choice of $\epsilon$ and $\delta$. 

\subsection*{Statistical analysis of empirical benchmarks}
Since all learning experiments were run across 10 seeds, for each task there are 10 policies per method (i.e., MaxDiff RL, NN-MPPI, and SAC). Due to differences between multi-shot and single-shot settings, we evaluated them differently. In multi-shot experiments, we took the 10 final policies learned and evaluated their performance across 100 episodes with randomized initial conditions. For each algorithm, this results in a total of 1000 sampled returns per task. Then, to assess statistical differences between the sampled 1000 episodic returns per algorithm, we used an unpaired two-sided Welch's t-test as implemented in Python's scientific computing package~\cite{SciPy2020}. An important note is that episodic return curves illustrate the policies' learning progress across each of the 10 random seeds, rather than policy evaluation. Policy evaluation is depicted in bar plots instead (e.g., Fig.~\ref{fig:fig3}(c) right).

The non-episodic nature of single-shot learning means that there is no individual time-stamp at which policies can be fairly evaluated. For this reason, in single-shot experiments we used a task-specific performance measure (i.e., distance traveled) sampled across the 10 runs of each task to perform statistical comparisons. In addition to our task-specific metrics, we also took the terminal windowed returns sampled during each of the 10 seeds of the learning tasks. As before, we applied a Welch's t-test onto the 10 sampled returns and 10 sampled task-specific metrics per algorithm. For statistics, we refer readers to Supplementary Table~\ref{table:stats}.


\section*{Data availability}
\label{sec:data_statement}
Data supporting the findings of this study are available in the following repository: \href{https://github.com/MurpheyLab/MaxDiffRL}{\color{blue}{\texttt{github.com/MurpheyLab/MaxDiffRL}}}.

\section*{Code availability}
\label{sec:code_statement}
Code supporting the findings of this study is available in the following repository: \href{https://github.com/MurpheyLab/MaxDiffRL}{\color{blue}{\texttt{github.com/MurpheyLab/MaxDiffRL}}}.

\section*{Acknowledgements}
\label{sec:acknowledgements}
We thank Annalisa T. Taylor, Jamison Weber, and Pavel Chvykov for their comments on early drafts of this work. We acknowledge funding from the US Army Research Office MURI grant \#W911NF-19-1-0233, and the US Office of Naval Research grant \#N00014-21-1-2706. We also acknowledge hardware loans and technical support from the Intel Corporation, and T.A.B. is partially supported by the Northwestern University Presidential Fellowship.

\section*{Author contributions} 
\label{sec:author_contributions}
T.A.B. derived all theoretical results, performed supplementary data analyses and control experiments, supported reinforcement learning experiments, and wrote the manuscript. A.P. developed and tested reinforcement learning algorithms, carried out all reinforcement learning experiments, and supported manuscript writing. T.D.M. secured funding and guided the research program.
\end{bibunit}

\clearpage
\newpage

\begin{bibunit}
\setcounter{section}{1}
\setcounter{subsection}{0}
\setcounter{equation}{0}
\setcounter{figure}{0}
\renewtheorem{theorem}{Theorem}[subsection]
\renewtheorem{proposition}{Proposition}[subsection]
\renewtheorem{definition}{Definition}[subsection]
\newtheorem{remark}{Remark}[subsection]
\newtheorem{corollary}{Corollary}[theorem]
\newtheorem{lemma}[theorem]{Lemma}
\renewcommand{\figurename}{Supplementary Figure}
\renewcommand{\tablename}{Supplementary Table}
\renewcommand{\bibsection}{\section*{Supplementary references}}
\renewcommand{\thesubsection}{\arabic{subsection}}

\noindent\makebox[\linewidth]{\rule{\linewidth}{1.5pt}}
\vspace{-0.25in}
\begin{center}
    {\Large Supplementary information}
\end{center}
\vspace{-0.2in}
\noindent\makebox[\linewidth]{\rule{\linewidth}{1.5pt}}

\tableofcontents

\pagebreak

\addtocontents{toc}{\protect\setcounter{tocdepth}{3}}

\addcontentsline{toc}{section}{Supplementary notes}
\section*{Supplementary notes}
\label{sec:notes_SI}
\subsection{Introduction}
\label{sec:SI_introduction}

In the following supplementary notes, we lay out the theoretical framework of maximum diffusion reinforcement learning (MaxDiff RL). MaxDiff RL is a generalization of maximum entropy (MaxEnt) RL in a similar sense as the principle of maximum caliber~\cite{Dill2018} is a generalization of the principle of maximum entropy~\cite{Kapur1989}. This requires a deliberate shift in the way we interpret the underlying goal of RL algorithms: from reaching desirable states to realizing desirable trajectories. By assigning conceptual importance to the ``state trajectory'' as a mathematical abstraction, our approach has an explicit focus on the way that properties of the underlying agent-environment state transition dynamics impact the performance of RL algorithms. In particular, we consider the impact that temporal correlations in the trajectories of RL agents can have on their performance and design MaxDiff RL to overcome this impact.

The broad structure of the supplement is the following: first, in Supplementary Note~\ref{sec:maxdiff_theory}, we derive a novel understanding of exploration through the lens of maximum caliber trajectory sampling. In doing so, we are able to derive analytical expressions that describe the trajectories of optimally exploring agents in settings where there is no goal or reward, as well as in settings where there is one. Then, in Supplementary Note~\ref{sec:maxdiff_synthesis}, we provide a mathematical framework for synthesizing agent behavior that satisfies optimal exploration statistics, which we show to be formally equivalent to the usual stochastic optimal control formulation of RL problems. Finally, Supplementary Note~\ref{sec:implementation} provides implementation details and statistical analyses of our empirical results. We will now provide a per-section summary of the results provided in Supplementary Notes~\ref{sec:maxdiff_theory} and~\ref{sec:maxdiff_synthesis}.

Supplementary Note~\ref{sec:maxdiff_theory} establishes the theoretical foundations of our approach. First, Supplementary Note~\ref{sec:controllability} motivates our primary conceptual point in a restricted class of systems---that temporal correlations in the state transition dynamics of embodied agents can have an impact on effective exploration and learning performance. Then, Supplementary Note~\ref{sec:exploration_sampling} establishes some mathematical preliminaries, such as how to think of an agent's experiences or state trajectories as collections of random variables parametrized by a time-like variable, and how to measure temporal correlations. Supplementary Note~\ref{sec:exploration_opt_problem} formalizes the problem of undirected state exploration through the lens of maximum caliber trajectory sampling. In doing so, we pay particular attention to realizing exploration with continuous trajectories. In Supplementary Note~\ref{sec:exploration_diffusion}, we prove that optimal exploration with continuous trajectories is achieved by state-space diffusion (Theorem~\ref{thm:diffusion}). Moreover, we prove that agents who satisfy optimal exploration statistics are Markovian (Corollary~\ref{cor:markov}) and ergodic (Corollary~\ref{cor:ergodicity}). Notably, this is not something we assumed a priori. In Supplementary Note~\ref{sec:directed_exploration}, we extend our results to directed exploration settings where there is a cost or reward function that assigns some notion of preference to particular states. While the Markov property holds in this setting automatically, we prove that optimal directed exploration is still ergodic (Theorem~\ref{thm:ergodicity}). Finally, Supplementary Note~\ref{sec:descent_directions} provides additional motivation that illustrates the sense in which maximum caliber directed exploration leads to goal-directed behavior. In doing so, we analyze the maximum likelihood trajectories of optimally exploring path distribution and find that they have inertial dynamics resembling gradient descent.

Supplementary Note~\ref{sec:maxdiff_synthesis} establishes the computational foundations of our approach. First, in Supplementary Note~\ref{sec:KL_control} we define MaxDiff trajectory synthesis more broadly in terms of KL control. In short, we define the objective of MaxDiff trajectory synthesis as finding controllers or policies that minimize the distance between an agent's trajectory distribution and the optimal trajectory distributions (as derived in Supplementary Note~\ref{sec:maxdiff_theory}). In Supplementary Note~\ref{sec:exploration_soc}, we show that the KL control objective from the previous section can be written as an equivalent stochastic optimal control problem, which allows us to formally state the MaxDiff RL objective. Supplementary Note~\ref{sec:maxdiff_RL} explores the formal properties of MaxDiff RL: the sense in which it generalizes MaxEnt RL (Theorem~\ref{thm:maxdiff_generalization} and Main Text Theorem~\ref{thm:thm1}), its single-shot learning capabilities (Theorem~\ref{thm:maxdiff_single} and Main Text Theorem~\ref{thm:thm3}), and their robustness to seeds and initializations (Theorem~\ref{thm:maxdiff_robust} and Main Text Theorem~\ref{thm:thm2}). Supplementary Notes~\ref{sec:exploration_entropymax} and~\ref{sec:exploration_controllers} introduce alternative formulations of the MaxDiff RL objective that are easier to compute, as well as more amenable to model-free RL implementations. Finally, Supplementary Note~\ref{sec:maxdiff_exploration} provides some examples of MaxDiff trajectory synthesis outside of RL.

\newpage
\clearpage

\subsection{Theoretical framework for maximum diffusion}
\label{sec:maxdiff_theory}
Throughout this section we analytically derive and establish the theoretical properties of maximally diffusive agents and their trajectories, as well as their relationship to \textit{i.i.d.} data, temporal correlations, controllability, and exploration. We do not directly discuss reinforcement learning within this section beyond framing our results, but rather establish mathematical foundations that elucidate the relationship between an agent's properties and its ability to explore and learn. For our implementation of these principles within a reinforcement learning framework, refer to Supplementary Note~\ref{sec:maxdiff_synthesis}.

\subsubsection{The role of temporal correlations in exploration and learning}
\label{sec:controllability}
Exploration is a process by which agents become exposed to new experiences, which is of broad importance to their learning performance. While many learning systems can function as abstract processes insulated from the challenges and uncertainties associated with embodied operation~\cite{LeCun2015}, physical agents---simulated or otherwise---have no such luxury~\cite{Taylor2021,Levine2021,Miki2022,Bloesch2022}. The laws of physics, material properties, and dynamics all impose fundamental constraints on what can be achieved by a learning system. The main conceptual point of this work is that the state transition dynamics of embodied learning agents can introduce temporal correlations that hinder their performance. In this section, we provide a formal mathematical argument in favor of this point in a particular class of systems. We do this in hopes of motivating how the controllability properties of the agent-environment state transition dynamics---and the temporal correlations these induce---may have an impact on the efficacy of action randomization as an exploration strategy more generally, and as a result on performance. 

Drawing inspiration from the study of multi-armed bandits~\cite{Auer2002}, the most common exploration strategy in reinforcement learning is randomized action exploration. The simplest of these methods merely requires that agents randomly sample actions from either uniform or Gaussian distributions to produce exploration. More sophisticated methods, such as maximum entropy reinforcement learning~\cite{Haarnoja2017,Haarnoja2018,So2022}, elaborate on this basic idea by learning a distribution from which to sample random actions. For the purpose of our analysis, these more advanced methods are functionally equivalent to each other---they assume that taking random actions produces effective state exploration. However, from the perspective of control theory we know that this is not necessarily the case. For a system to be able to reach desired states arbitrarily, it must be controllable~\cite{Sontag2013}.

To illustrate how the controllability properties of the agent-environment state transition dynamics can determine the structure and magnitude of temporal correlations, we will briefly consider randomized action exploration in linear time-varying (LTV) control systems. This is a broad class of systems for which we can provide formal mathematical arguments in favor of our main point. LTV dynamics can be expressed in terms of continuous-time deterministic trajectories in the following way:
\begin{equation}
    \label{eq:supp_lindyn}
    \dot{x}(t) = A(t)x(t)+B(t)u(t),
\end{equation}
where $A(t)$ and $B(t)$ are appropriately dimensioned matrices with state and control vectors $x(t)\in\mathcal{X}\subset\mathbb{R}^d$ and $u(t)\in\mathcal{U}\subset\mathbb{R}^m$, and $x(t_0)=x^*$ for $\mathcal{T}=[t_0,t] \subset \mathbb{R}$. The general form of solutions to this system of linear differential equations is expressed in terms of a convolution with the system's state-transition matrix, $\Psi(t,t_0)$, in the following way:
\begin{equation}
    \label{eq:supp_lindyn_sol}
    x(t) = \Psi(t,t_0)x^*+\int_{t_0}^{t}\Psi(t,\tau)B(\tau)u(\tau)d\tau.
\end{equation}
We consider these dynamics because by working with LTV dynamics we implicitly consider a very broad class of systems---all while retaining the simplicity of linear controllability analysis~\cite{Hespanha2018}. This is due to the fact that the dynamics of any nonlinear system that is locally linearizable along its trajectories can be effectively captured by LTV dynamics. Hence, any results applicable to the dynamics in Eq.~\ref{eq:supp_lindyn} will apply to linearizable nonlinear systems. However, we note that our derivations in subsequent sections do \textit{not} assume dynamics of this form. We only consider them to motivate our approach in this section.

To develop an understanding of the exploration capabilities of a given LTV system, we may ask what states are reachable by this system. After all, states that are not reachable cannot be explored or learned from. This is precisely what controllability characterizes:
\begin{definition}
\label{def:controllability}
    A system is said to be controllable over a time interval $[t_0, t] \subset \mathcal{T}$ if given any states $x^*,x_1\in \mathcal{X}$, there exists a controller $u(t):[t_0,t]\rightarrow \mathcal{U}$ that drives the system from state $x^*$ at time $t_0$ to $x_1$ at time $t$.
\end{definition}
\noindent While this definition intuitively captures what is meant by controllability, it does not immediately seem like an easily verifiable property. To this end, different computable metrics have been developed that equivalently characterize the controllability properties of certain classes of systems (e.g., the Kalman controllability rank condition~\cite{Sontag1991}). In particular, here we analyze the controllability Gramian of our system, as well as its rank and determinant as metrics on system controllability. 

For our class of LTV systems, characterizing controllability with this method is simple:
\begin{equation}
    \label{eq:supp_controllability_gramian}
    W(t_0,t) = \int_{t_0}^{t} \Psi(t,\tau)B(\tau)B(\tau)^T\Psi(t,\tau)^Td\tau,
\end{equation}
where the Gramian is a symmetric positive semidefinite matrix that depends on the state-control matrix $B(t)$ and the state-transition matrix $\Psi(t,t_0)$. The Gramian is a controllability metric that quantifies the amount of energy required to actuate the different degrees of freedom of the system~\cite{Cortesi2014,Summers2016}. For any given finite time interval, the controllability Gramian also characterizes the set of states reachable by the system. Importantly, when the controllability Gramian is full-rank, the system is provably controllable in the sense of Definition~\ref{def:controllability}~\cite{Sontag2013}, and capable of fully exploring its environment. However, when the controllability Gramian is poorly conditioned, substantial temporal correlations are introduced into the agent's state transitions, which can prevent effective exploration and---as a direct consequence---learning, as we will show.

To draw the connection between naive random exploration, controllability, and temporal correlations explicitly, we will now revisit the dynamics in Eq.~\ref{eq:supp_lindyn} under a slight modification. Let us design a controller that performs naive action randomization, i.e., let $u(t)=\xi$, where $\xi\sim\mathcal{N}(\mathbf{0},\text{Id})$ and Id is an identity matrix with diagonal of the same dimension as the control inputs, and $\mathbf{0}$ is the zero vector of the same dimension. Note that the system trajectories are now random variables---or rather, collections of random variables, which we define formally in the following section. Then, we have:
\begin{equation}
    \label{eq:supp_lindyn_random}
    \dot{x}(t) = A(t)x(t)+B(t)\cdot \xi.
\end{equation}
Here, we abuse notation slightly to minimize the difference between this equation and Eq.~\ref{eq:supp_lindyn}, but we can interpret the system as having linear Langevin dynamics~\cite{Kardar2007fields}. With these modifications in mind, we are now interested in examining the mean and covariance trajectory statistics in hopes of characterizing the structure of temporal correlations induced by the agent dynamics. We begin by taking the expectation over system trajectories described by Eq.~\ref{eq:supp_lindyn_sol}:
\begin{equation}
    \label{eq:supp_lindyn_mean}
    \begin{split}
    E[x(t)] &= E\Big[\Psi(t,t_0)x^*+\int_{t_0}^{t}\Psi(t,\tau)B(\tau)\cdot \xi d\tau\Big]\\
    &= \Psi(t,t_0) x^*+ E\Big[\int_{t_0}^{t}\Psi(t,\tau)B(\tau)\cdot \xi d\tau\Big]\\
    &=\Psi(t,t_0)x^*.
    \end{split}
\end{equation}
Hence, the expected sample paths of the dynamics will be centered around the autonomous paths of the system---that is, the paths the system takes in the absence of control inputs. We may now characterize the covariance of our system's sample paths. To do so, let $\mathbf{C}[x^*]=E\big[(x(t)-E[x(t)])(x(t)-E[x(t)])^T\big| x(t_0)=x^* \big]$ be the trajectory autocovariance. Although we will formalize this idea in the following section, for now we note that the trajectory-wise expectation is taken as the time-integration of point-wise autocovariances. With these preliminaries taken care of, we have: 
\begin{align}\label{eq:supp_lindyn_cov}
    \mathbf{C}[x^*] = E\big[&(x(t)-E[x(t)])(x(t)-E[x(t)])^T\big| x(t_0)=x^* \big]\nonumber \\
    = E\Big[&\Big(\Psi(t,t_0)x^*+\int_{t_0}^{t}\Psi(t,\tau)B(\tau)\cdot \xi d\tau-E[x(t)]\Big)\nonumber \\
    &\times \Big(\Psi(t,t_0)x^*+\int_{t_0}^{t}\Psi(t,\tau)B(\tau)\cdot \xi d\tau-E[x(t)]\Big)^T\Big]\nonumber \\
    = E\Big[&\Big(\int_{t_0}^{t}\Psi(t,\tau)B(\tau)\cdot \xi d\tau\Big)\Big(\int_{t_0}^{t}\Psi(t,\tau)B(\tau)\cdot \xi d\tau\Big)^T\Big]\nonumber\\
    =E\Big[&\int_{t_0}^{t}\Psi(t,\tau)B(\tau)\cdot (\xi\xi^T)\cdot  B(\tau)^T \Psi(t,\tau)^T d\tau\Big]\nonumber\\
    =\int_{t_0}^{t}&\Psi(t,\tau)B(\tau)B(\tau)^T \Psi(t,\tau)^T d\tau.
\end{align}
By inspection of the above expression and Eq.~\ref{eq:supp_controllability_gramian}, we arrive at the following important connection:
\begin{equation}
    \label{eq:supp_controllability_covariance}
    \mathbf{C}[x^*] = W(t_0,t)
\end{equation}
which tells us that for LTV dynamics (and by extension for linearizable nonlinear dynamics), a measure of temporal correlations---the trajectory autocovariance $\mathbf{C}[x^*]$ (see Supplementary Note~\ref{sec:exploration_sampling})---is exactly equivalent to the controllability Gramian of the system. Thus, for a broad class of systems, an agent's controllability properties are given by a measure of temporal correlations along their state trajectories. Moreover, in LTV systems these are not state-dependent properties. In other words,
\begin{equation}
    \label{eq:supp_controllability_gradient}
    \nabla_x \mathbf{C}[x^*] = \nabla_x W(t_0,t) = \mathbf{0},
\end{equation}
where $\mathbf{0}$ is an appropriately dimensioned zero matrix. However, for linearizable nonlinear systems, as well as more general nonlinear systems, these properties will be state-dependent. While our controllability analysis has been restricted to the class of dynamics describable by linear differential equations with time-varying parameters, we note that the connections we observe between trajectory autocovariance and controllability Gramians have been shown to hold for even more general classes of nonlinear systems through more involved analyses~\cite{Kashima2016}. Nonetheless, the results of our manuscript hold regardless of whether there is a formal and easily characterizable relationship between controllability and temporal correlations.

\begin{figure}[tp!]
    \centering
    \includegraphics[width=0.7\linewidth]{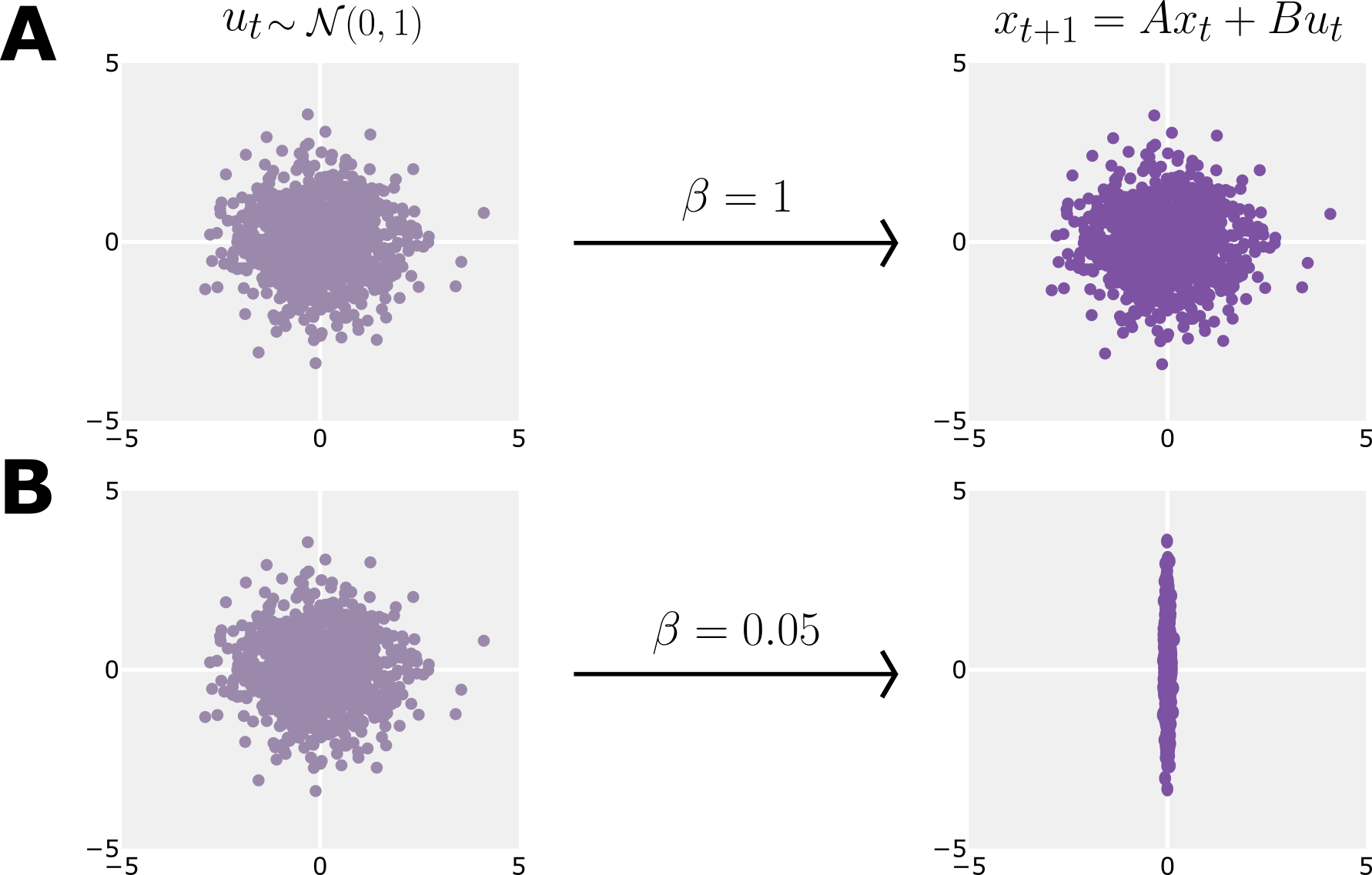}
    \caption{\textbf{Effect of controllability on the distribution of reachable states}. 
    \textbf{a,} For a linear system with dynamics like those in Figure 1 of the main text initialized with an $x_t$ of all zeroes, we depict the effect of controllability on a naive random action exploration strategy. For a linear system with ideal controllabilty properties, isotropic distributions of actions map onto isotropic distributions of states.
    \textbf{b,} However, when the system is poorly conditioned the system dynamics distort the isotropy of the original input distribution, introducing temporal correlations induced by the controllability properties of the system, and fundamentally changing its properties as an exploration strategy.
    }
    \label{fig:supp_isotropic_dists}
\end{figure}

From Eq.~\ref{eq:supp_lindyn_random} we can describe the system's reachable states by analyzing its state probability density function, which can be found analytically by solving its associated Fokker-Planck equation~\cite{Risken1996}. To do this, we only require the mean and covariance statistics of the process, in Eqs.~\ref{eq:supp_lindyn_mean} and~\ref{eq:supp_lindyn_cov}. Hence, the system's time-dependent state distribution is
\begin{equation}
    \label{eq:supp_fokker_planck_stationary}
    p(x, t, t_0) = \frac{1}{\sqrt{(2\pi)^d \det[W(t_0,t)]}} \exp \Big[-\frac{1}{2}\big(x-\Psi(t,t_0)x^*\big)^TW^{-1}(t_0,t)\big(x-\Psi(t,t_0)x^*\big)\Big]
\end{equation}
for some choice of initial conditions at $t_0$, where we have substituted Eq.~\ref{eq:supp_controllability_covariance} to highlight the role of controllability in the probability density of states reachable by the system through naive random exploration. Thus, how easy or hard it is to explore in a given direction (as characterized by Eq.~\ref{eq:supp_fokker_planck_stationary}) is entirely determined by the controllability properties of the system---or, equivalently, by a measure of temporal correlations of its state trajectories. Supplementary Fig.~\ref{fig:supp_isotropic_dists} illustrates this concept for the toy dynamical system introduced in the main text. We observe that changes in $\beta$ have an effect on the distribution of reachable states for the system that are consistent with Eq.~\ref{eq:supp_fokker_planck_stationary}, where we note we recentered the distribution mean. 

On the basis of these results, which have been known for decades~\cite{Mitra1969}, we can clearly see that controllability and temporal correlations play a key role in exploration and data acquisition. We cannot assume that random inputs are capable of producing effective exploration of system states without an understanding of its controllability. For example, if $W(t_0,t)$ is not full-rank, then exploration would be restricted to a linear subspace of an agent's exploration domain. This amounts to a complete collapse of the \textit{i.i.d.} assumption on the experiences of an agent, because its state transitions become deterministically correlated as a result of the degeneracy of Eq.~\ref{eq:supp_fokker_planck_stationary}. However, if we were to instead design $u(t)$ by exploiting knowledge about $\mathbf{C}[x^*]$, these limitations can be overcome. For example, if we let $u(t)=B(t)^T\mathbf{C}^{-1}[x^*]B(t)\cdot \xi$ instead, then better exploration can be achieved by can be reshaping $p(x,t,t_0)$ in a way that accounts for the system's temporal correlations. This is in fact the conceptual crux of our entire reinforcement learning framework, as we will show.

In more complex settings, where the input distribution is not Gaussian and the dynamics are strongly nonlinear, analyzing controllability may be more challenging. However, insofar as learning requires an embodied agent to either collect data or visit desirable states to optimize some objective, temporal correlations and controllability will continue to play an important role.
\begin{remark}
    \label{rem:controllability}
    Temporal correlations and controllability can determine whether it is possible and how challenging it is to learn.
\end{remark}
\noindent While one can construct proofs that illustrate this in a variety of simplified settings---as others have recently shown in the case of controllability~\cite{Tsiamis2021,Tsiamis2022}---we leave the more general claim as a remark to frame the motivation behind our upcoming derivations. Hence, we should strive to develop exploration and learning strategies that reflect---and try to overcome---the effect of controllability and its induced temporal correlations, as we do in the following sections.

\subsubsection{Exploration as trajectory sampling}
\label{sec:exploration_sampling}
In this section, we develop the mathematical formalism necessary for framing exploration in a controllability-aware manner that may allow us to overcome temporal correlations. While exploration with disembodied agents can be quite simple (e.g., sampling from a distribution, or performing a random walk), embodied agents must achieve exploration by changing the state of the environment through action. Our goal is to achieve state exploration in an embodied system, such as a robotic agent or otherwise, where their embodiment constrains the ways they can explore the states of an environment. While this motivation is most natural for physical systems, our framing is relevant to any setting in which the underlying agent-environment dynamics obey some notion of continuity of experience. To this end, we will need to define a formal notion of control system from which we can begin to model the experiences of agents. 

First, we formally define stochastic processes by adapting the definition provided in~\cite{Oksendal2010} to our use case.
\begin{definition}
    \label{def:stochpross_oksendal}
    A stochastic process is a family of random variables parametrized by a totally ordered indexing set $\mathcal{T}$,
    \begin{equation}
        \{X_t\}_{t\in\mathcal{T}}\notag \text{       when $\mathcal{T}$ is discrete, or       } \{X(t)\}_{t\in\mathcal{T}} \text{       when $\mathcal{T}$ is continuous},
    \end{equation}
    defined on a probability space $(\Omega,\mathcal{F},\mathbb{P})$. We take the sample space $\Omega$ to be measurable, $\mathcal{F}$ to be a Borel $\sigma$-algebra, and $\mathbb{P}$ to be a probability measure. We note that the random variables assume values in a compact state space $\mathcal{X}\subset\mathbb{R}^d$, and that each sample path takes value in a measurable space $\mathcal{X}^{\mathcal{T}}$ with Borel $\sigma$-algebra $\mathcal{B}(\mathcal{X}^{\mathcal{T}})$.
\end{definition}
\noindent Thus, stochastic processes are families of random variables indexed according to some ``time-like'' set, $\mathcal{T}$. For each $\omega \in \Omega$, the sample paths of the stochastic process, $x_{\mathcal{T}}(\omega) = \{X(t,\omega)\}_{t\in\mathcal{T}}$, take value in $\mathcal{X}^\mathcal{T}$. We note that we often take $\mathcal{T}$ to be an interval, e.g., $[t_0,t]$ or a halfline. When $\mathcal{T}$ is discrete, e.g., $\{1,\cdots,N\}$, we have $x_{1:N}(\omega) = \{X_t(\omega)\}_{t\in\{1,\cdots,N\}}$ instead. Then, we can define the pushforward measure of $x_{\mathcal{T}}:\Omega\rightarrow\mathcal{X}^{\mathcal{T}}$ in the usual way. That is, $P_F:\mathcal{B}(\mathcal{X}^{\mathcal{T}})\rightarrow [0,1]$ is given by $P_F[x_{\mathcal{T}}\in A] = \mathbb{P}(x_{\mathcal{T}}^{-1}(A))$ for some $A\subset\mathcal{X}^{\mathcal{T}}$. Finally, for a each $\omega$, we use $x(t)=x_{\mathcal{T}}(\omega)\in\mathcal{X}^\mathcal{T}$ to denote individual realizations of the stochastic process, and refer to $x(t)$ as an agent's experiences, \textit{state trajectories}, or \textit{paths}. To describe the likelihoods of individual state trajectories, we assume that the probability density function associated with the pushforward measure exists and is given by $P:\mathcal{X}^\mathcal{T}\rightarrow [0,\infty)$, such that
\begin{equation}
    P_F[x_{\mathcal{T}}\in A] = \mathbb{P}(x_{\mathcal{T}}^{-1}(A)) = \int_{x_{\mathcal{T}}^{-1}(A)}d\mathbb{P}(\omega) = \int_{A}P[x(t)]\mathcal{D}x(t)
\end{equation}
where $\mathcal{D}x(t)$ denotes integration over sample paths, as in the Feynman path integral formalism~\cite{Feynman2010}. Thus, we will refer to this density over paths as the \textit{path} or \textit{trajectory distribution}, and use $P[x(t)]$ to express the probability density of a given state trajectory of the stochastic process $x(t)\in\mathcal{X}^\mathcal{T}$. Alternatively, we use $P[x_{1:N}]$ when $\mathcal{T}$ is discrete.

To quantify correlations along sample paths or state trajectories, we evaluate a local measure of temporal correlations, $\mathbf{C}[x^*]$, over particular time intervals of a given stochastic process, $[t_i,t_i+\Delta t]\subset\mathcal{T}$. If $\{X(t)\}_{t\in\mathcal{T}}$ is a stochastic process defined according to Definition~\ref{def:stochpross_oksendal}, then an autocovariance function, $K_{XX}(t_1,t_2)$,  expresses the covariance of the process with itself at any two points in time $t_1,t_2\in\mathcal{T}$, or
\begin{equation}
    \label{eq:autocov_fn}
    K_{XX}(t_1,t_2) = E\big[(X(t_1)-E[X(t_1)])(X(t_2)-E[X(t_2)])^T \big].
\end{equation}
With these preliminaries taken care of, we define our measure of temporal correlations with respect to an initial condition $X(t_i)=x^*$ for some $x^*\in\mathcal{X}$ in the following way:
\begin{align}
    \mathbf{C}[x^*] &=  E\big[(X(t)-E[X(t)])(X(t)-E[X(t)])^T\big| X(t_i)=x^* \big] \notag \\
    &= \int_{t_i}^{t_i+\Delta t}K_{XX}(t_i,\tau)d\tau.
    \label{eq:temporal_correlations}
\end{align}
Thus, our measure of temporal correlations $\mathbf{C}[x^*]$ could also be characterized as an integrated autocovariance function along the state trajectories of a stochastic process. This usage of the term ``temporal correlations'' is in line with its broad usage in statistical mechanics (see Ch.~10 of~\cite{Sethna2021}), and we note that in practice one can divide Eq.~\ref{eq:temporal_correlations} by $\Delta t$ to prevent numerical estimates from strongly depending on the duration of the time-interval under consideration.

Lastly, it is important to note that the probability densities over state trajectories are strongly dependent on the dynamics that govern the agent-environment's time-evolution through state space. However, when the dynamics are nonautonomous, as is the case in control systems, this distribution will also depend on the choice of controller and the effect it has on the state transitions of the process. We define a controller as a function, $u(t):\mathcal{T} \rightarrow \mathcal{U}$, that produces an input to the system dynamics at every point in the index set, where $\mathcal{U}$ is usually a subset of $\mathbb{R}^m$. At this point, we are not considering the system dynamics themselves, how controllers are synthesized, or how much influence either of these can have in shaping the sample paths of the underlying control system. All we care about is acknowledging the fact that a choice of controller induces a different probability density over sample paths. With these definitions we can now establish our notion of control system, or stochastic control process. 

\begin{definition}
\label{def:control_system}
    A stochastic control process is a stochastic process (Definition~\ref{def:stochpross_oksendal}) on a probability space $(\Omega, \mathcal{F}, \mathbb{P}_{u(t)})$, with indexing set $\mathcal{T}$, where sample paths take value in a measurable space $(\mathcal{X}^{\mathcal{T}},\mathcal{B}(\mathcal{X}^{\mathcal{T}}))$, and the resulting density $P_{u(t)}:\mathcal{X}^\mathcal{T}\rightarrow [0,\infty)$ is parametrized by a controller $u(t):\mathcal{T} \rightarrow \mathcal{U}$.
\end{definition}
\noindent Thus, we think of control systems as stochastic processes that are parametrized by their controllers, or equivalently as a collection of distinct stochastic processes for each choice of controller.

In a stochastic control process the controller plays an important role in structuring the sample paths of a system---clearly, the sample path distribution of a robot with a controller that resists all movements is very different than one with a controller that encourages the robot to explore (see Supplementary Fig.~\ref{fig:supp_dists} for an illustration). Hence, controllers determine which regions of the state space the system is capable of sampling from. With this in mind, we can express the problem of exploration in control systems: to design a controller that maximizes the regions of the exploration domain from which we can sample trajectories. In part, this requires the use of control actions in order to maximize the support of the agent's sample path distribution. The support of a probability distribution is the subset of all elements in its domain with greater than zero probability. However, merely maximizing the sample path distribution's support is not enough to realize effective exploration in most settings. For directed state exploration, we would ideally also like to control \textit{how} probability mass is spread around the state space---if a given task demands that the agent's sample paths are biased towards a given goal, then our agent's path distribution should reflect this. In the following sections, we will work towards this goal of deriving path distributions for optimal undirected and directed exploration strategies. 

\begin{figure}[tp!]
    \centering
    \includegraphics[width=0.95\linewidth]{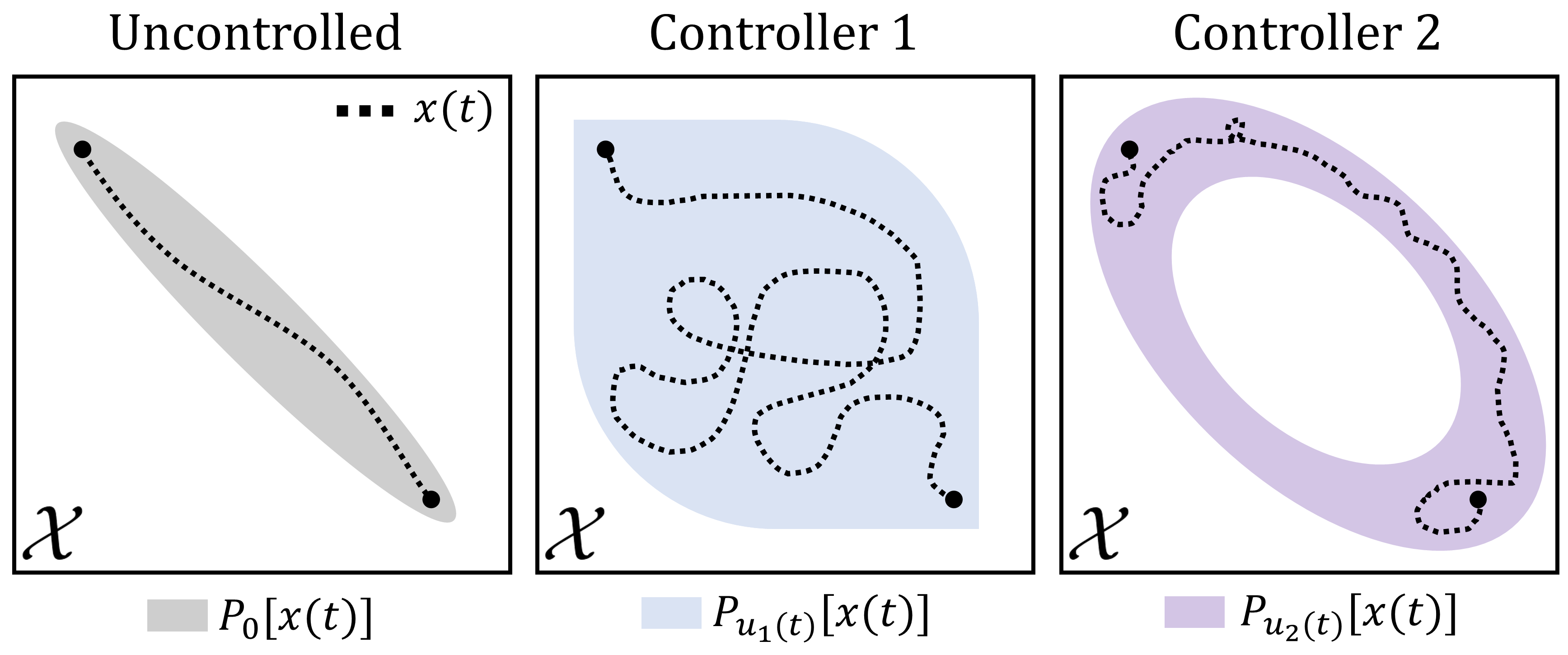}
    \caption{\textbf{Effect of controllers on the sample path distribution of stochastic control processes}. (left) Sample path and support of the probability density over the paths of an autonomous stochastic process (i.e., with null controller ``0''). (middle and right) Sample paths and distributions induced by two distinct controllers $u_1(t)$ and $u_2(t)$. Here, we illustrate that depending on the nature of the controller the distribution over sample paths can be nontrivial. Note that we do not illustrate the values of the probability densities, only their support. The reason for this is that so long as a regions of space have non-zero probability they will be sampled asymptotically.}
    \label{fig:supp_dists}
\end{figure}

\subsubsection{Undirected exploration as variational optimization}
\label{sec:exploration_opt_problem}
One way of simultaneously controlling the spread of probability mass and the support of a probability distribution is to optimize its entropy~\cite{Grendar2006}. For now, we consider the undirected exploration case, in which no task or objective biases the underlying agent's path distribution. As we will see in Supplementary Note~\ref{sec:directed_exploration}, this approach will also enable us to control the spread of probability mass in a more fine-grained manner in order to realize directed exploration with respect to an objective or task---and eventually to do reinforcement learning. 

Optimizing the entropy of an agent's path distribution through control synthesis can have a profound effect on the resulting behavior of the agent. This can be understood intuitively when there are no constraints on how we can increase the entropy of a sample path distribution. In this case, the maximum entropy distribution would be uniform over the entirety of the system's compact state space, leading to complete asymptotic exploration of the domain in a way that is equivalent to \textit{i.i.d.} uniform sampling. However, a process realizing the statistics described by such a path distribution would require teleportation---that is, that points in space be visited uniformly at random at every moment in time. While this may pose no problems for disembodied agents with unconstrained dynamics, this creates issues for any agent whose experiences are constrained by their embodiment or otherwise. For example, in physical control systems subject to the laws of physics, this is infeasible behavior. Hence, throughout the rest of this section we will take on the work of deriving the maximum entropy distribution for describing the state trajectories of agents with continuous experiences---a broad class of systems that includes all physical systems and many non-physical systems---as well as analyzing the formal properties of agents whose experiences satisfy such statistics. By maximizing trajectory entropy, this distribution will capture the statistics of an agent with minimally-correlated experiences. The analytical form of this distribution is crucial to the control and policy synthesis approach we derive in Supplementary Note~\ref{sec:maxdiff_synthesis}. However, we note that our results will also apply for disembodied agents with discontinuous paths when we consider the uniform distribution as the optimal distribution instead of the one we derive in this section.

We proceed by identifying the analytical form of the maximum entropy path distribution with no consideration given to the problem of generating actions that achieve such statistics. Hence, we begin by framing our exploration problem in the maximum caliber formalism of statistical mechanics~\cite{Jaynes1957,Dill2018,Chvykov2021}. Maximum caliber is a generalization of the principle of maximum entropy to function spaces, such as distributions over trajectories or sample paths. In order to apply the principle of maximum caliber within our stochastic process formalism, we first note that we interpret path integrals in the following way. Consider some real-valued function $f(\cdot)$ of $x_{\mathcal{T}}$, then we define its expectation over sample paths as
\begin{equation}
    \label{eq:stochastic_path_integral_interpretation}
    E[f(x_{\mathcal{T}})] = \int_{\Omega}f(x_{\mathcal{T}}(\omega))d\mathbb{P}(\omega) = \int_{\mathcal{X}^{\mathcal{T}}}P[x(t)]f(x(t))\mathcal{D}x(t),
\end{equation}
which is consistent with our definition of probability densities over state trajectories, $P[x(t)]$. Thus, path integrals integrate over the state trajectories of a stochastic process. Now, we are interested in finding a distribution which maximizes the entropy of sample paths, $S[P[x(t)]]$. Because we are looking for the unique analytical form of this distribution, we omit the controller-specific notation that was previously introduced---at least until we consider control and policy synthesis in the context of stochastic optimal control and reinforcement learning in Supplementary Note~\ref{sec:maxdiff_synthesis}. The general form of the maximum caliber variational optimization is then the following:
\begin{equation}\label{eq:supp_objective1}
    \underset{P[x(t)]}{\text{argmax}} \  -\int_{\mathcal{X}^{\mathcal{T}}} P[x(t)]\log P[x(t)] \mathcal{D}x(t)
\end{equation}
However, as written the optimization is ill-posed and leads to a trivial solution. We can see this by taking the variation with respect to the sample path distribution, where we would find that the optimal sample path distribution is uniform, yet not a valid probability density as it is unnormalized. Thus, we need to constrain the optimization problem so that we only consider behavior realizable by the class of agents we are interested in modeling.

Since we are interested in framing our exploration problem for application domains like optimal control and reinforcement learning, we tailor our modeling assumptions to these settings. What sorts of principled constraints could be applied? No constraints based on conservation of energy are applicable because autonomous systems are inherently nonequilibrium systems. Nonetheless, the behavior of many autonomous systems (especially physically embodied ones) is constrained by other aspects of their morphology, such as actuation limits and continuity of movement. In particular, the rates at which agent experiences or states can vary---and \textit{co-vary}---in time are typically bounded, which prevents them from discontinuously jumping between states by limiting their local rate of exploration. In fact, this is precisely what we found in Supplementary Note~\ref{sec:controllability}, where we saw that a system's ability to explore is closely tied to a measure of its temporal correlations, $\mathbf{C}[x^*]$, as defined in the previous section. Thus, we will choose to constrain the velocity fluctuations of our stochastic process so that they are finite and consistent with the integrated autocovariance statistics of the process, which may be determined empirically, and are related to a system's controllability properties in a broad class of systems. The use of an empirical (or learned) autocovariance estimate to quantify velocity fluctuations is important because different embodied agents have different limitations, which may additionally be spatially inhomogeneous and difficult to know a priori. Through this constraint, we can ensure that agent sample paths are continuous in time.

To formulate this path continuity constraint, we must first express the system's velocity fluctuations at each point in state space, $x^*\in\mathcal{X}$. We define the system's velocity fluctuations along sample paths $x(t)$ in the following way:
\begin{equation}\label{eq:supp_explore_rate}
    \langle\dot{x}(t)\dot{x}(t)^T\rangle_{x^*} = \int_{\mathcal{X}^{\mathcal{T}}} P[x(t)] \int_{\mathcal{T}} \dot{x}(\tau)\dot{x}(\tau)^T \delta (x(\tau)-x^*) d\tau\mathcal{D}x(t), 
\end{equation}
where $\delta(\cdot)$ denotes the Dirac delta function, and we note that the $\langle\cdot\rangle$ notation of statistical physics is equivalent to an expectation, i.e., $E[\cdot]$. We assume that the tensor described by Eq.~\ref{eq:supp_explore_rate} is full-rank so that the system's velocity fluctuations are not degenerate anywhere in the state space of the stochastic process. This assumption is crucial because it guarantees that our resulting path distribution is non-degenerate. If we had instead chosen to constrain the system by directly bounding the magnitude of its velocities, as opposed to its velocity fluctuations, we would not be able to guarantee the non-degeneracy of the resulting path distribution. Another important note is that the velocities of the trajectories of the stochastic process in this expression should be interpreted in the Langevin sense~\cite{Kardar2007fields}. That is to say, not as expressions of the differentiability of the sample paths of the underlying stochastic process, but as a shorthand for an integral representation of the stochastic differential equations describing the evolution of the sample paths of the system. 

We can now express our constraint as,
\begin{equation}\label{eq:supp_constraint1}
    \langle\dot{x}(t)\dot{x}(t)^T\rangle_{x^*} = \mathbf{C}[x^*], \quad \forall x^* \in \mathcal{X}.
\end{equation}
Crucially, these statistics are bounded everywhere in the exploration domain, and we assume them to satisfy Lipschitz continuity so that their spatial variations are bounded. We note that linearizability of the underlying agent dynamics is a sufficient condition to satisfy this property. Hence, we now have equality constraints on the system's velocity fluctuations that can vary at each point in the exploration domain---as one would expect for a complex embodied system, such as a robot. As an additional constraint, we require that $P[x(t)]$ integrates to 1 so that it is a valid probability density over trajectories.

With expressions for each of our constraints, we may now express the complete variational optimization problem using Lagrange multipliers:
\begin{align}
    \label{eq:supp_objective2}
   \underset{P[x(t)]}{\text{argmax}} \ -\int_{\mathcal{X}^\mathcal{T}} & P[x(t)] \log P[x(t)]\mathcal{D}x(t) -\lambda_0 \Big(\int_{\mathcal{X}^\mathcal{T}} P[x(t)] \mathcal{D}x(t)-1 \Big)  \\ 
   &-\int_{\mathcal{X}} Tr\Big(\Lambda(x^*)^T\big(\langle\dot{x}(t)\dot{x}(t)^T\rangle_{x^*} - \mathbf{C}[x^*]\big)\Big)dx^*, \notag
\end{align}
Here, we express the constraints at all points $x^*$ by taking an integral over all points in the domain. The $\lambda_0$ is a Lagrange multiplier enforcing our constraint that ensures valid probability densities, and $\Lambda(\cdot)$ is a matrix-valued Lagrange multiplier working to ensure that the rate of exploration constraints hold at every point in the domain. By solving this optimization we can obtain an expression for the maximum entropy distribution over sample paths. The solution to this problem will determine the distribution over sample paths with the greatest support, with the most uniformly spread probability mass, and with the least-correlated sample paths---thereby specifying the statistical properties of our optimal undirected exploration strategy, subject to a path continuity constraint.

\subsubsection{Maximizing path entropy produces diffusion}
\label{sec:exploration_diffusion}
In this section, we lay out the derivation of our solution to the variational optimization problem in Eq.~\ref{eq:supp_objective2}. We begin by stating our main result in the following theorem.

\begin{theorem}
\label{thm:diffusion}
    The maximum caliber sample paths of a stochastic control process (Definition~\ref{def:control_system}) with a maximum entropy exploration (in the sense of Eq.~\ref{eq:supp_objective2}) are given by diffusion with spatially-varying coefficients. 
\end{theorem}
\begin{proof}
    Letting $\mathcal{T}=[t_0,t]$, we begin by substituting Eq.~\ref{eq:supp_explore_rate} into Eq.~\ref{eq:supp_objective2}, taking its variation with respect to the probability density $\delta S[P[x(t)]]/\delta P[x(t)]$, and setting it equal to 0:
    \begin{equation}\nonumber
        \frac{\delta S}{\delta P[x(t)]} = -1 - \log P_{max}[x(t)] -  \lambda_0 -\int_{\mathcal{X}} \int_{t_0}^{t} Tr\Big(\Lambda(x^*)^T (\dot{x}(\tau)\dot{x}(\tau)^T)\Big) \delta(x(\tau)-x^*) d\tau dx^*=0.
    \end{equation}
    Then, taking advantage of the following linear algebra identity, $a^TBa=Tr(B^T(aa^T))$, for any $a\in \mathbb{R}^m$ and $B\in\mathbb{R}^{m\times m}$; as well as the properties of the Dirac delta, we can simplify our expression to the following: 
    \begin{equation}\nonumber
        \frac{\delta S}{\delta P[x(t)]} = -1 - \log P_{max}[x(t)] -  \lambda_0 -\int_{t_0}^{t} \dot{x}(\tau)^T\Lambda(x(\tau))\dot{x}(\tau) d\tau=0,
    \end{equation}
    which allows us to solve for the maximum entropy probability distribution over the sample paths of our stochastic control process. The solution will then be of the form:
    \begin{equation}\label{eq:supp_var_sol1}
        P_{max}[x(t)] = \frac{1}{Z} \exp\Big[-\int_{t_0}^{t} \dot{x}(\tau)^T\Lambda(x(\tau))\dot{x}(\tau) d\tau\Big],
    \end{equation}
    where we have subsumed the constant and Lagrange multiplier, $\lambda_0$, into a normalization factor, $Z$. We note that even without determining the form of our Lagrange multipliers, the maximum entropy probability density in Eq.~\ref{eq:supp_var_sol1} is already equivalent to the path probability of a diffusing particle with a (possibly anisotropic) spatially-inhomogeneous diffusion tensor (see~\cite{Kardar2007fields}, Ch. 9). While there is more work needed to characterize the diffusion tensor of this process, $\Lambda^{-1}(\cdot)$, this completes our proof.
\end{proof}

Thus, the least-correlated sample paths, which optimally sample from the exploration domain, are statistically equivalent to diffusion. This is to say that the distribution of paths with the greatest support over the state space describes the paths of a diffusion process. Hence, if the goal of some stochastic control process is to optimally explore and sample from its state space, the best strategy is to move randomly---that is, to decorrelate its sample paths. An additional benefit of our diffusive exploration strategy is that we did not have to presuppose that our agent dynamics were Markovian or ergodic. Instead, we find that these properties emerge through our derivation as intrinsic properties of the optimal exploration strategy itself. The following corollaries of Theorem~\ref{thm:diffusion} follow from the connection to diffusion processes and Markov chains, and as such more general forms of these proofs may be found in textbooks on stochastic processes and ergodic theory. Here, we assume that the diffusion tensor in Eq.~\ref{eq:supp_var_sol1}, $\Lambda^{-1}(\cdot)$, is full-rank and invertible everywhere in the state space. Additionally, for now we will assume that $\Lambda^{-1}(\cdot)$ is Lipschitz and bounded everywhere on $\mathcal{X}$. We will later find that these are not in fact different assumptions from those made in Eqs.~\ref{eq:supp_explore_rate} and~\ref{eq:supp_constraint1}.

\begin{corollary}
\label{cor:markov}
    The sample paths of a stochastic control process (Definition~\ref{def:control_system}) with a maximum entropy exploration strategy (in the sense of Eq.~\ref{eq:supp_objective2}) satisfy the Markov property. 
\end{corollary}
\begin{proof}
    This follows trivially from the temporal discretization of our path distribution in Eq.~\ref{eq:supp_var_sol1}, or alternatively from the properties of Langevin diffusion processes. Letting $x_t$ be the initial condition, we can see that,
    \begin{align}
        \label{eq:supp_path_discretization}
        p_{max}(x_{t+\delta t}|x_t) &= \frac{1}{Z}\exp\Big[-\int_{t}^{t+\delta t} \dot{x}(\tau)^T\Lambda(x(\tau))\dot{x}(\tau) d\tau\Big] \notag\\
        &\approx \frac{1}{Z_d} \exp\Big[-|x_{t+\delta t}-x_t|_{\Lambda(x_t)}^2\Big],
    \end{align}
    where we subsumed $\delta t$ into a new normalization constant $Z_d$ for convenience, and note that the support of $p_{max}(x_{t+\delta t}|x_t)$ is infinite. Importantly, our local Lagrange multiplier $\Lambda(x_t)$ enforces our velocity fluctuation constraint within a neighborhood of states reachable from $x_t$ for a sufficiently small time interval $\delta t$, which is guaranteed by our Lipschitz continuity assumption. In what remains of this manuscript we use $\delta t=1$ for notational convenience, but without loss of generality. Thus, our distribution in Eq.~\ref{eq:supp_path_discretization} depends only on the current state, which concludes our proof.
\end{proof}

\begin{corollary}
\label{cor:ergodicity}
    A stochastic control process (Definition~\ref{def:control_system}) in a compact and connected space $\mathcal{X}\subset\mathbb{R}^d$ with a maximum entropy exploration strategy (in the sense of Eq.~\ref{eq:supp_objective2}) is ergodic. 
\end{corollary}
\begin{proof}
    To prove the ergodicity of the process described by Eq.~\ref{eq:supp_var_sol1}, we use Corollary~\ref{cor:markov} and the properties of $\mathcal{X}$. We begin by discretizing our optimal stochastic control process in time and space such that $P_{max}[x_{1:N}]=\prod_{t=1}^{N-1} p_{max}(x_{t+1}|x_t)$, which we can do without loss of generality as a result of Corollary~\ref{cor:markov} and because $\mathcal{X}$ is compact, resulting in a finite space. Importantly, since $p_{max}(x_{t+1}|x_t)>0, \ \forall x_t,x_{t+1}\in\mathcal{X}, \ \forall t \in \mathcal{T}$, and $\mathcal{X}$ is finite and connected, then all states in $\mathcal{X}$ communicate. Moreover, because for all $x^*\in\mathcal{X}$, $p_{max}(x^*|x^*)>0$, the underlying Markov chain described by the transition kernel is aperiodic. Therefore, the Markov chain describing the stochastic control process is ergodic~\cite{Puterman2014}.
\end{proof}

To finish our derivation and fully characterize the nature of our maximum entropy exploration strategy, we must return to Eq.~\ref{eq:supp_var_sol1} and determine the form of the matrix-valued Lagrange multiplier $\Lambda(\cdot)$. Hence, we will return to our expression for $\langle\dot{x}(t)\dot{x}(t)^T\rangle_{x^*}$ in Eq.~\ref{eq:supp_explore_rate} and discretize our continuous sample paths, which we can do without loss of generality due to Corollary~\ref{cor:markov}. Since Eq.~\ref{eq:supp_explore_rate} represents a proportionality, we take out many constant factors throughout the derivation. Additionally, any constant factor of $\Lambda(\cdot)$ would be taken care of by the normalization constant $Z$ in the final expression for Eq.~\ref{eq:supp_var_sol1}. We proceed by discretizing Eq.~\ref{eq:supp_explore_rate}, using $i$ and $j$ as time indices and $p_{max}(\cdot|\cdot)$ as the conditional probability density defined in Eq.~\ref{eq:supp_path_discretization}. We do this by slicing the time interval $[t_0,t]$ into time indices $\{1,\cdots,N\}$. Our resulting expression is the following:
\begin{equation}\label{eq:supp_explore_rate_discrete}
    \langle\dot{x}(t) \dot{x}(t)^T\rangle_{x^*} = \prod_{i=1}^{N-1} \Big[\int_{\mathcal{X}} dx_{i+1} \  p_{max}(x_{i+1}|x_{i}) \Big] \sum_{j=1}^{N-1} (x_{j+1}-x_{j})(x_{j+1}-x_{j})^T\delta(x_j-x^*),
\end{equation}
where the path integrals are discretized according to the Feynman formalism~\cite{Feynman2010}, using the same discretization as in our proof of Corollary~\ref{cor:markov}.

From this expression in Eq.~\ref{eq:supp_explore_rate_discrete}, we take the following two steps. First, we switch out the order of summation and product by applying the Fubini-Tonelli theorem. Then, we factor out two integrals from the product expression---one capturing the probability flow \textit{into} $x_j$ and one capturing the flow \textit{out of} it: 
\begin{align}
    = &\sum_{j=1}^{N-1}\prod_{i\neq j, j-1} \Big[\int_{\mathcal{X}} dx_{i+1} \  p_{max}(x_{i+1}|x_{i})\Big] \nonumber\\
    &\times\int_{\mathcal{X}} p_{max}(x_{j}| x_{j-1})\int_{\mathcal{X}} p_{max}(x_{j+1}|x_{j}) (x_{j+1}-x_{j})(x_{j+1}-x_{j})^T\delta(x_j-x^*)dx_{j+1} dx_{j},\nonumber
\end{align}
where $\times$ denotes multiplication with the line above. Then we can apply the Dirac delta function to simplify our expression and get:
\begin{align}
    = \sum_{j=1}^{N-1} \prod_{i\neq j, j-1} &\Big[\int_{\mathcal{X}} dx_{i+1} \ p_{max}(x_{i+1}|x_{i})\Big] \nonumber\\
    &\times p_{max}(x^*| x_{j-1}) \int_{\mathcal{X}} p_{max}(x_{j+1}|x^*) (x_{j+1}-x^*)(x_{j+1}-x^*)^Tdx_{j+1}.
    \label{eq:supp_explore_rate_discrete2}
\end{align}
To simplify further we will tackle the following integral as a separate quantity:
\begin{equation}
    \label{eq:supp_explore_rate_int}
    I = \int_{\mathcal{X}} p_{max}(x_{j+1}|x^*) (x_{j+1}-x^*)(x_{j+1}-x^*)^Tdx_{j+1}.
\end{equation}
where we can substitute Eq.~\ref{eq:supp_path_discretization} into Eq.~\ref{eq:supp_explore_rate_int} to get:
\begin{equation}\nonumber
    I = \int_{\mathcal{X}} \frac{1}{Z_d} e^{-(x_{j+1}-x^*)^T\Lambda(x^*)(x_{j+1}-x^*)} (x_{j+1}-x^*)(x_{j+1}-x^*)^Tdx_{j+1}.
\end{equation}
This integral can then be tackled using integration by parts and closed-form Gaussian integration. Thus far, we have not had any need to specify the domain in which exploration takes place. However, in order to evaluate this multi-dimensional integral-by-parts we require integration limits. To this end, we will assume that the domain of exploration is large enough so that the distance between $x^*$ and $x_{j+1}$ makes the exponential term approximately decay to 0 at the limits, which we shorthand by placing the limits at infinity:
\begin{align}\label{eq:supp_explore_rate_int2}
    I = \frac{1}{Z_d}\Lambda  (x^*)^{-1}&\Big[\sqrt{\det(2\pi\Lambda^{-1}(x^*))} \nonumber\\
    &-(x_{j+1}-x^*)^T\mathbf{1}e^{-(x_{j+1}-x^*)^T\Lambda(x^*)(x_{j+1}-x^*)}\Big|_{x_{j+1}=-\infty}^{x_{j+1}=\infty} \Big],
\end{align}
where $\mathbf{1}$ is the vector of all ones, and the exponential term vanishes when evaluated at the limits. Note that our assumption on the domain of integration implies that we do not consider boundary effects, and that the quantity within the brackets is a scalar that can commute with our Lagrange multiplier matrix.

We are now ready to put together our final results. By combining Eq.~\ref{eq:supp_explore_rate_int2} and plugging it into Eq.~\ref{eq:supp_explore_rate_discrete2} we have
\begin{align}\label{eq:supp_explore_rate_discrete3}
    \langle \dot{x}(t) \dot{x}(t)^T\rangle_{x^*} = \frac{1}{Z_d}\sum_{j=1}^{N-1} \prod_{i\neq j, j-1} & \Big[\int_{\mathcal{X}} dx_{i+1} p_{max}(x_{i+1}|x_{i}) \Big] \nonumber\\ 
    &\times p_{max}(x^*| x_{j-1}) \sqrt{\det(2\pi\Lambda^{-1}(x^*))} \Lambda (x^*)^{-1}.
\end{align}
Since $\langle \dot{x}(t) \dot{x}(t)^T\rangle_{x^*}$ is everywhere full-rank, we can see that $\Lambda (x^*)^{-1}$ must be full-rank as well. Next, we recognize that $\sqrt{\det(2\pi\Lambda (x^*)^{-1})}$ cancels out with $Z_d$, and that we can re-expand  $p_{max}(x^*|x_{j-1})$ as an integral over $\delta(x_j-x^*)$ and fold it back into the integral product. Rearranging terms we have:
\begin{equation}\label{eq:supp_explore_rate_discrete4}
    \langle \dot{x}(t) \dot{x}(t)^T\rangle_{x^*} =  \prod_{i=1}^{N-1} \Big[\int_{\mathcal{X}}dx_{i+1}\ p_{max}(x_{i+1}|x_{i}) \Big] \sum_{j=1}^{N-1} \delta(x_j-x^*) \Lambda (x^*)^{-1}.
\end{equation}
At this point, we note that this expression merely computes the average of $\Lambda (x^*)^{-1}$ over all possible state trajectories that pass through $x^*$. However, because $\Lambda (x^*)^{-1}$ is a constant for any given $x^*$, this expression reduces down to $\langle \dot{x}(t) \dot{x}(t)^T\rangle_{x^*} = \Lambda (x^*)^{-1}$. Thus, using Eq.~\ref{eq:supp_constraint1}, we find that our Lagrange multiplier is given by:
\begin{equation}
    \label{eq:supp_lagrange_mult}
    \Lambda (x^*) = \mathbf{C}^{-1}[x^*].
\end{equation}
This result is significant because now we can relate a measure of temporal correlations to the sample path distribution of an optimally exploring agent. Taking this result and returning to Eq.~\ref{eq:supp_var_sol1}, we now have the final form of the maximum entropy exploration sample path distribution in terms of our measure of temporal correlations:
\begin{equation}
    \label{eq:supp_var_sol2}
    P_{max}[x(t)] = \frac{1}{Z} \exp\Big[-\frac{1}{2}\int_{t_0}^{t} \dot{x}(\tau)^T\mathbf{C}^{-1}[x(\tau)]\dot{x}(\tau) d\tau\Big],
\end{equation}
where we have added a factor of one half to precisely match the path probability of diffusive spatially-inhomogeneous dynamics. This final connection can be made rigorous by noting that $\mathbf{C}[x^*]$ is an estimator of a system's local diffusion tensor through the following relation: $\mathbf{C}[\cdot]=\frac{1}{2}\mathbf{D}[\cdot]\mathbf{D}[\cdot]^T$ for some diffusion tensor $\mathbf{D}[\cdot]$~\cite{Michalet2012,Boyer2012}. Lastly, we can discretize this distribution to arrive at the discrete-time maximum entropy sample path probability density:
\begin{equation}
    \label{eq:supp_var_sol2_discrete}
    p_{max}(x_{t+1}|x_t) = \frac{1}{Z_d} \exp\Big[-\frac{1}{2}|x_{t+1}-x_t|^2_{\mathbf{C}^{-1}[x_t]}\Big].
\end{equation}
Thus, when faced with path continuity constraints, the optimal exploration strategy is given by diffusion in state space, which concludes our derivation. In line with this, we describe systems that satisfy these statistics as \textit{maximally diffusive}.

Throughout this derivation, we have assumed for convenience that the velocity fluctuations of the stochastic control process are full-rank everywhere. This is equivalent to saying that the control system is capable of generating variability along all dimensions of its degrees of freedom---or equivalently, as shown in Supplementary Note~\ref{sec:controllability} for linearizable nonlinear systems, that our system is controllable. However, this assumption is somewhat artificial because typically we are not interested in exploring directly on the full state space of our control system. Instead, we often consider some differentiable coordinate transformation $y(t)=\psi(x(t))$ that maps our states in $\mathcal{X}$ onto the desired exploration domain $\mathcal{Y}$. In this case, all results described thus far will still hold and we will have a valid expression for $P_{max}[y(t)]$ with diffusion tensor $\mathbf{C}[y^*]$, so long as $\mathbf{C}[y^*]=\mathbf{J}_{\psi}[x^*]\mathbf{C}[x^*]\mathbf{J}_{\psi}[x^*]^T$ is everywhere full-rank, where $\mathbf{J}_{\psi}[\cdot]$ is the Jacobian matrix corresponding to the coordinate transformation $\psi$. Hence, we only require that the new system coordinates are controllable. This is particularly useful when we are dealing with high-dimensional systems with which we are interested in exploring highly coarse-grained domains.

\subsubsection{Directed exploration as variational optimization}
\label{sec:directed_exploration}
In the previous section we derived the analytical form of our maximum entropy exploration strategy, which describes agents with maximally-decorrelated experiences and whose path probabilities are equivalent to those of an ergodic diffusion process. Thus far, we have only discussed exploration as an undirected (or passive) process. This is to say, as a process that is blind to any notion of importance or preference ascribed to regions of the state space or exploration domain~\cite{Thrun1992}. However, under a simple reformulation of our exploration problem we will see that we can also achieve efficient directed exploration with theoretical guarantees on its asymptotic performance.

In many exploration problems, we have an a priori understanding of what regions of the exploration domain are important or informative. For example, in reinforcement learning this is encoded by the reward function~\cite{Haarnoja2017}, and in optimal control this is often encoded by a cost function or an expected information density~\cite{Miller2016,Mavrommati2018}. In such settings, one may want an agent to explore states while taking into account a measure of information or importance of that state, which is known as directed (or active) exploration. In order to realize directed exploration, we require a notion of the ``importance'' of states that is amenable to the statistical-mechanical construction of our approach. To this end, we reformulate our maximum entropy objective into a ``free energy'' minimization objective by introducing a bounded potential function $V[\cdot]$. Across fields, potential functions are used to ascribe (either a physical or virtual) cost to system states. A potential function is then able to encode tasks in control theory, learning objectives in artificial intelligence, desirable regions in spatial coverage problems, etc. Hence, we will extend the formalism presented in the previous sections to parsimoniously achieve goal-directed exploration by considering the effect of potential functions. 

Since our maximum entropy functional is an expression over all possible trajectories, we need to adapt our definition of a potential to correctly express our notion of ``free energy'' over possible system realizations. To this end, we define our potential over $\mathcal{T}=[t_0,t]$ in the following way,
\begin{equation}\label{eq:supp_path_potential}
    \langle V[x(t)]\rangle_{P} = \int_{\mathcal{X}^\mathcal{T}} P[x(t)] \int_{t_0}^{t} V[x(\tau)] d\tau\mathcal{D}x(t), 
\end{equation}
which captures the average cost over all possible system paths (integrated over each possible state and time for each possible path). Formally, we must assume that $\langle V[x(t)]\rangle_{P}$ is bounded, which in practice will be the case for policies and controllers derived from these principles. Our new free energy functional objective is
\begin{equation}\label{eq:supp_free_energy_functional}
    \underset{P[x(t)]}{\text{argmin}} \  \langle V[x(t)]\rangle_P-S[P[x(t)]],
\end{equation}
where we use $S[P[x(t)]]$ as a short-hand for the argument to Eq.~\ref{eq:supp_objective2}. Thankfully, to find the optimal path distribution all of the work carried out in Supplementary Notes~\ref{sec:exploration_opt_problem} and~\ref{sec:exploration_diffusion} remains unchanged. All that's needed is to take the variation of Eq.~\ref{eq:supp_path_potential} with respect to $P[x(t)]$ and integrate it into the optimal path distribution. As this arithmetic is very similar to the derivation provided in the proof of Theorem~\ref{thm:diffusion}, we omit it here. The resulting minimum free energy path distribution is then
\begin{equation}\label{eq:supp_free_energy_sol}
    P_{max}^V[x(t)] = \frac{1}{Z} \exp\Big[-\int_{t_0}^{t} \Big(V[x(\tau)]+\frac{1}{2}\dot{x}(\tau)^T\mathbf{C}^{-1}[x(\tau)]\dot{x}(\tau) \Big)d\tau\Big],
\end{equation}
which corresponds to the path distribution of a diffusion process in a potential field. Hence, the optimal directed exploration strategy is to scale the strength of diffusion with respect to the desirability of the state. In this sense, the net effect of the potential is merely to bias the diffusion process. We refer to systems satisfying such statistics as \textit{maximally diffusive with respect to the underlying potential}. As an aside, we note that,
\begin{equation}\label{eq:relationship_between_dists}
    P_{max}^V[x(t)] = P_{max}[x(t)]\cdot e^{-\int V[x(\tau)]d\tau}
\end{equation}
from which we can recover $P_{max}[x(t)]$ in the absence of a potential (i.e., $V[\cdot]=0$). We note that we can manipulate the above expression into a form amenable to Markov decision processes by letting $l(\cdot) = V[\cdot]$ be a standard cost function, which leads us to the following equivalent expression:
\begin{equation}\label{eq:relationship_between_dists_disc}
    P_{max}^l[x_{1:N}] = \prod_{t=1}^{N-1} p_{max}(x_{t+1}|x_t) e^{-l(x_t)},
\end{equation}
where we have discretized agent state trajectories without loss of generality. Remarkably, this path distribution resembles the form of those used in the control-as-inference literature~\cite{Rawlik2012}. We will find that the form of this distribution we derived is crucial to the approach we take in trajectory synthesis and reinforcement learning, particularly once we introduce a dependence on agent actions into the cost function.

What are the properties of such an exploration strategy? Since we already know that the sample paths of agents applying our exploration strategy are Markovian, as long as the potential function and its interactions with our agent are memory-less the sample paths generated by Eq.~\ref{eq:supp_free_energy_functional} will continue to be as well. However, ergodicity is a more challenging property to ascertain as it depends on the properties of the underlying potential function and of our diffusion process. Nonetheless, in the following theorem we show that the trajectories of an agent successfully diffusing according to our exploration strategy in a non-singular potential will continue to be ergodic under some mild assumptions.

\begin{theorem}
\label{thm:ergodicity}
    A stochastic control process (Definition~\ref{def:control_system}) in a compact and connected space $\mathcal{X}\subset\mathbb{R}^d$ with a maximum entropy exploration strategy in a potential (in the sense of Eq.~\ref{eq:supp_free_energy_sol}) is ergodic.
\end{theorem}
\begin{proof}
    The proof of this theorem can be easily arrived at by extending the proof of Corollary~\ref{cor:ergodicity}. As long as $V[\cdot]$ is bounded everywhere in the domain, we may discretize the stochastic control process in space and time everywhere in the domain, as in Corollary~\ref{cor:ergodicity}. Then, we can see that $p_{max}^V(x_{t+1}|x_t) = p_{max}(x_{t+1}|x_t)e^{-V[x_t]}>0, \ \forall x_t,x_{t+\delta t}\in \mathcal{X}, \ \forall t \in \mathcal{T}$. This is because we have already shown that $p_{max}(\cdot|\cdot)>0$ in Corollary~\ref{cor:ergodicity}, and because of the properties of the potential. Thus, the underlying Markov chain described by the $p_{max}^V(x_{t+1}|x_t)$ transition kernel is aperiodic and all states communicate, which guarantees ergodicity and concludes our proof.
\end{proof}

Hence, the net effect of the potential is to reshuffle probability mass in the stationary distribution of the agent's underlying Markov chain. We note that these proofs can be carried out without discretizations by instead invoking the physics of diffusion processes, as in~\cite{Wang2019} where the authors proved that heterogeneous diffusion processes in a broad class of non-singular potentials are ergodic when the strength of the potential exceeds the strength of diffusion-driven fluctuations. However, here we limit ourselves to methods from the analysis of stochastic processes. In short, minimum free energy exploration leads to ergodic coverage of the exploration domain with respect to the potential. We note that this is an important result when it comes to the applicability of our results in robotics and reinforcement learning, as it is effectively an asymptotic guarantee on learning when the learning task is encoded by the choice of potential function---as we will illustrate in the following sections.

\subsubsection{Minimizing path free energy produces diffusive gradient descent}
\label{sec:descent_directions}
To develop further intuition about the sense in which the statistics of Eq.~\ref{eq:supp_free_energy_sol} describe goal-directed exploratory behavior, we can examine the maximum likelihood trajectory of our minimum free energy path distribution under the assumption of path differentiability, which the rest of our analysis does not require. To do this, we begin by calculating the negative log-likelihood of $P_{max}^V[x(t)]$ neglecting the normalization factor:
\begin{equation}
    \label{eq:supp_loglike}
    \begin{split}
    -\log [P_{max}^V[x(t)]] &= \int_{t_0}^{t} V[x(\tau)]+\frac{1}{2}\dot{x}(\tau)^T\mathbf{C}^{-1}[x(\tau)]\dot{x}(\tau) d\tau\\
    &= \int_{t_0}^{t} \mathcal{H}(\tau,x(\tau),\dot{x}(\tau)) d\tau\\
    &= \int_{t_0}^{t} \dot{x}(\tau)^T\mathbf{C}^{-1}[x(\tau)]\dot{x}(\tau)-\mathcal{L}(\tau,x(\tau),\dot{x}(\tau)) d\tau,
    \end{split}
\end{equation}
where we noted that the integral's argument is a Hamiltonian whose Legendre transform we can take, and arrive at an equivalent Lagrangian description of the system. Then, to derive an expression for the maximum likelihood trajectories of our path distribution we can extremize the Lagrangian's associated action functional:
\begin{equation}
    \label{eq:supp_action}
    \mathcal{A} = \int_{t_0}^{t}  \mathcal{L}(\tau,x(\tau),\dot{x}(\tau)) d\tau = \int_{t_0}^{t} V[x(\tau)]-\frac{1}{2}\dot{x}(\tau)^T\mathbf{C}^{-1}[x(\tau)]\dot{x}(\tau)d\tau.
\end{equation}
Assuming that our potential is differentiable, we can find the dynamics of the maximum likelihood trajectory by using the Euler-Lagrange equations:
\begin{equation}
    \label{eq:supp_euler_lagrange}
    \begin{split}
    0 &= \nabla_{x}\mathcal{L}-\frac{d}{dt}\big[\nabla_{\dot{x}}\mathcal{L}\big]\\
     &=\nabla_x V[x(t)] -\frac{1}{2}\dot{x}(t)^T\nabla_x \mathbf{C}^{-1}[x(t)]\dot{x}(t) - \Big[-\ddot{x}(t)^T\mathbf{C}^{-1}[x(t)]-\dot{x}(t)^T\nabla_x \mathbf{C}^{-1}[x(t)]\dot{x}(t)\Big]\\
     &= \nabla_x V[x(t)]+\ddot{x}(t)^T\mathbf{C}^{-1}[x(t)]+\frac{1}{2}\dot{x}(t)^T\nabla_x \mathbf{C}^{-1}[x(t)]\dot{x}(t)
    \end{split}
\end{equation}
which we can rearrange into our final expression,
\begin{equation}
    \label{eq:supp_maxlikelihood_dyn_full}
    \ddot{x}(t) = -\mathbf{C}[x(t)]\Big[\nabla_x V[x(t)]+\frac{1}{2}\dot{x}(t)^T\nabla_x \mathbf{C}^{-1}[x(t)]\dot{x}(t)\Big]
\end{equation}
This last expression represents the maximum likelihood dynamics of a system whose trajectories satisfy our minimum free energy path distribution. We note that $\nabla_x \mathbf{C}^{-1}[x(t)]=-\mathbf{C}^{-1}[x(t)]\nabla_x \mathbf{C}[x(t)]\mathbf{C}^{-1}[x(t)]$, which we omitted from Eq.~\ref{eq:supp_maxlikelihood_dyn_full} for notational simplicity. Our expression is comprised of two gradient-like terms. The first of these terms points in directions of descent for the potential, and the second in directions that increase the system's autocovariance statistics (or controllability, when applicable).

To simplify these dynamics further, we can make one of two assumptions: either that our measure of temporal correlations varies slowly over space (at least relative to $\nabla_x V$), or that our system dynamics are LTV. Taken together, our assumptions imply that $\nabla_x \mathbf{C} \approx \mathbf{0}$, which leads to a simplification of the final expression in Eq.~\ref{eq:supp_maxlikelihood_dyn_full}. For the sake of making a connection to controllability, consider simplifying the maximum likelihood dynamics by assuming they are LTV. For this class of dynamics, Eq.~\ref{eq:supp_controllability_gradient} tells us that our system's controllability properties do not vary over state-space, as discussed in Supplementary Note~\ref{sec:controllability}. For systems with fixed or quasi-static morphologies this assumption holds well. Then, we have the following simplified dynamics:
\begin{equation}
    \label{eq:supp_descent}
    \begin{split}
        \ddot{x}(t) &= -\mathbf{C}[x(t)]\nabla_x V[x(t)]\\
        &= -W(t,t_0)\nabla_x V[x(t)].
    \end{split}
\end{equation}
By inspection we see that these second order dynamics resemble those of inertial gradient descent~\cite{Nesterov1983,Qian1999,Attouch2017,Attouch2019}, with two key differences. First, the absence of a damping term in the expression, which can be artificially introduced and tuned to guarantee and optimize convergence. Alternatively, we can note that any physical system approximately satisfying maximally diffusive trajectory statistics will experience dissipation, which means there may be no need to introduce it artificially. Second, and more importantly, that our system's ability to produce descent directions that optimize the potential is affected by its controllability properties. Thus, our results show that controllable agents can minimize arbitrary potentials merely through diffusive state exploration, which forms the conceptual basis of our approach to optimization and learning in the following sections.

As a final note on the derivations carried out throughout all of Supplementary Note~\ref{sec:maxdiff_theory}, we point out that most of the work we have done largely amounts to formulating an exploration problem and deriving the optimal trajectory distribution for a sufficiently broad class of agents---those with continuous paths. However, it is entirely possible for agents with discontinuous paths to have constraints on their state transitions as well. In other words, just because an agent may be capable of teleporting from one state to another (e.g., in a digital environment) it does not mean that it is equally easy to teleport to and from every state in the environment. Thus, as a final note we point out that every formal result we have proven throughout Supplementary Note~\ref{sec:maxdiff_theory} still holds when we remove the path continuity constraint (except for the analysis in this section). However, in this case the optimal distribution will be uniform over the state space, i.e., $p^U_{max}(x_{t+1}|x_t) = 1/|\mathcal{X}|$. In the presence of a potential our agent would also provably realize ergodic Markov exploration with respect to a cost or potential function. In this case, the optimal path distribution would take a similar form as Eq.~\ref{eq:relationship_between_dists_disc}, i.e., $p^{U,l}_{max}(x_{t+1}|x_t) = p^U_{max}(x_{t+1}|x_t)e^{- l(x_t)}$. However, realizing these path statistics is only possible when the underlying agent is fully controllable in the sense of Definition~\ref{def:full_controllability}, as we discuss in the following section. We note that it is exclusively under these conditions that agents can completely overcome correlations between state transitions. Agents satisfying these statistics will achieve \textit{i.i.d.} sequential sampling; however, the connection to the statistical mechanics of diffusion processes will no longer hold.

\clearpage
\newpage

\subsection{Synthesizing maximally diffusive trajectories}
\label{sec:maxdiff_synthesis}
Throughout the previous section, we have been studying the properties of a theoretical agent whose experiences spontaneously satisfy the path statistics of a maximally diffusive stochastic control process. However, the autonomous dynamics of control systems will typically not satisfy these statistics on their own. Hence, we require an approach from which to synthesize controllers (and policies) that generate maximally diffusive trajectories. In this section, we provide a general formulation of such an approach as well as simplifications amenable to use in real-time optimal control synthesis and reinforcement learning. All results derived herein form part of what we refer to as \textit{maximum diffusion (MaxDiff) trajectory synthesis}.

\subsubsection{Maximally diffusive trajectories via KL control}
\label{sec:KL_control}
In previous sections, we derived the maximally diffusive path distribution, $P_{max}^V[x(t)]$, and characterized the properties of sample paths drawn from it in the presence of a potential that ascribes a cost to system states, $V[\cdot]$. Now, we turn to the question of synthesizing policies and controllers that can actually satisfy these trajectory distributions. To this end, we recall that in Supplementary Note~\ref{sec:exploration_sampling} we defined a path probability density for an arbitrary stochastic control process, $P_{u(t)}[x(t)]$. Equipped with this distribution, we are able to express the most general form of the MaxDiff trajectory synthesis objective. To synthesize maximally diffusive trajectories, it suffices to generate policies and controllers that minimize the Kullback-Leibler (KL) divergence between the analytical optimum we derived in Supplementary Note~\ref{sec:maxdiff_theory} and the system's current path distribution. Equivalently, we can express this as,
\begin{equation}\label{eq:KL_control_obj}
    \underset{u(t)}{\text{argmin}}  \ D_{KL}(P_{u(t)}[x(t)]||P_{max}^V[x(t)]),
\end{equation}
which we can reformulate into many alternative forms through simple manipulations, as we illustrate throughout the following sections. Here, we first manipulate the objective into a form that highlights the different roles of the terms comprising it. Importantly, we note that taking the KL divergence is a well-defined operation in this context because the support of $P_{max}^V[x(t)]$ is infinite, and we have assumed that $\mathcal{X}$ is a compact domain. Using the definition of the KL divergence over path distributions and taking $\mathcal{T}=[t_0,t]$, we can factor our objective in the following way:
\begin{equation}
    \label{eq:supp_KL_control_obj_deriv}
    \begin{split}
    D_{KL}(P_{u(t)}||P_{max}^V) &= \int_{\mathcal{X}^\mathcal{T}} P_{u(t)}[x(t)] \log \frac{P_{u(t)}[x(t)]}{P_{max}^V[x(t)]}\mathcal{D}x(t)\\
     &=\int_{\mathcal{X}^\mathcal{T}} P_{u(t)}[x(t)] \Big[\log P_{u(t)}[x(t)]-\log P_{max}^V[x(t)]\Big]\mathcal{D}x(t)\\
     &= \int_{\mathcal{X}^\mathcal{T}} P_{u(t)}[x(t)] \Big[\log P_{u(t)}[x(t)]-\log P_{max}[x(t)]+\int_{t_0}^{t}V[x(\tau)]d\tau\Big]\mathcal{D}x(t)\\
     &= \langle V[x(t)] \rangle_{P_{u(t)}}+D_{KL}(P_{u(t)}[x(t)]||P_{max}[x(t)]),
    \end{split}
\end{equation}
where we used Eq.~\ref{eq:relationship_between_dists} to arrive at our final expression. Now, we can rewrite our control synthesis problem as the following
\begin{equation}\label{eq:KL_control_obj_freeen}
    \underset{u(t)}{\text{argmin}}  \ \langle V[x(t)] \rangle_{P_{u(t)}}+D_{KL}(P_{u(t)}[x(t)]||P_{max}[x(t)]),
\end{equation}
or equivalently
\begin{equation}\label{eq:KL_control_obj_freeen2}
    \underset{u(t)}{\text{argmin}}  \ E_{P_{u(t)}}\big[ L[x(t),u(t)]\big] +D_{KL}(P_{u(t)}[x(t)]||P_{max}[x(t)]),
\end{equation}
where we replace our potential with a cost function $L[x(t),u(t)]= \int_{\mathcal{T}} l(x(t),u(t)) dt$ in terms of the running cost $l(\cdot,\cdot)$. While potential functions are a natural way to ascribe thermodynamic costs to the states of physical systems, such as diffusion processes, there is no reason to restrict ourselves to that formalism now that we are focused on control synthesis. We also replaced our physics-based expected value notation, but note that they are formally equivalent (i.e., $\langle \cdot \rangle_p = E_{p}[\cdot]$). Finally, we note that we can introduce a temperature-like parameter $\alpha>0$ to balance between the two terms in our objective: the first, which optimizes task performance; and the second, which optimizes the statistics of the system's state space diffusion. Thus, when the system is able to achieve maximally diffusive trajectory statistics, our approach reduces to solving the task with thorough exploration of the cost landscape. 

An interesting property of this result is that in our theoretical approach there is no formal trade-off between exploration and exploitation---at least asymptotically. This is because when a system is capable of realizing maximally diffusive trajectories, the KL divergence term goes to zero. That being said, in practice this is not the case and the introduction of $\alpha$ will be of practical use in balancing between exploration and exploitation. Moreover, when maximally diffusive statistics are satisfied the expected value of the objective is taken with respect to the optimal maximum entropy trajectory distribution (i.e., $E_{P_{max}}\big[L[x(t),u(t)]\big]$), which is a bias-minimizing estimator of the cost function asymptotically equivalent to \textit{i.i.d.} sampling of state-action costs (or rewards) as a result of the ergodic properties of $P_{max}[x(t)]$. This is particularly useful in applications like reinforcement learning where the cost (or reward) function is unknown.

\subsubsection{Maximally diffusive trajectories via stochastic optimal control}
\label{sec:exploration_soc}
We can formulate our KL control problem as an equivalent stochastic optimal control (SOC) problem by making use of their well-known connections~\cite{Rawlik2012}. SOC problems are typically framed as Markov decision processes (MDPs) where the objective is to find a policy that optimizes the expected cost of a given cost-per-stage function over some time-horizon. More formally, an MDP is a 5-tuple $(\mathcal{X},\mathcal{U},p,r,\gamma)$, with state space, $\mathcal{X}$, and action space, $\mathcal{U}$. Then, $p: \mathcal{X}\times\mathcal{X}\times\mathcal{U}\rightarrow[0,\infty)$ represents the probability density of transitioning from state $x_t\in\mathcal{X}$ to state $x_{t+1}\in\mathcal{X}$ after taking action $u_t\in\mathcal{U}$. At every point in time, for each state, and for each action taken, the environment emits a bounded loss $l:\mathcal{X}\times\mathcal{U}\rightarrow [l_{min},l_{max}]$ discounted at a rate $\gamma\in [0,1)$. Given this formalism, the standard discrete time formulation of the SOC objective is
\begin{equation}\label{eq:soc_standard}
    \pi^* = \underset{\pi}{\text{argmin}}  \ E_{(x_{1:N},u_{1:N})\sim P_{\pi}} \Bigg[\sum_{t=1}^{N} \gamma^t l(x_t,u_t) \Bigg],
\end{equation}
where $l(\cdot,\cdot)$ is a discretized running cost and the expectation is taken with respect to the trajectory distribution induced by the policy $P_{\pi}$, which we will now motivate and define.

To translate our KL control results from the previous section into an equivalent SOC problem, we will have to make some modifications to our approach. In particular, the introduction of a policy $\pi(\cdot|\cdot)$ that replaces our notion of a controller (as defined in Supplementary Note~\ref{sec:exploration_sampling}) requires careful treatment. Whereas our definition of a path distribution allowed us to express a distribution directly over the state trajectories of the agent-environment dynamical process, the introduction of a policy induces a distribution over actions as well. In other words, instead of $P_{u_{1:N}}[x_{1:N}]$, we will now have $P_{\pi}[x_{1:N},u_{1:N}]$. This creates a complication by making the KL divergence in Eq.~\ref{eq:KL_control_obj} ill-posed---the agent's state-action path distribution and our maximally diffusive distribution are now defined over different domains. To solve this issue, we introduce the following distributions:
\begin{equation}
    \label{eq:policy_dists}
    \begin{split}
    P_{\pi}[x_{1:N},u_{1:N}] &= \prod_{t=1}^N p(x_{t+1}|x_t,u_t)\pi(u_t|x_t) \\
    P_{max}^{l}[x_{1:N},u_{1:N}] &= \prod_{t=1}^N p_{max}(x_{t+1}|x_t) e^{-l(x_t,u_t)},
    \end{split}
\end{equation}
where $p_{max}(x_{t+1}|x_t)$ is the discretized maximally diffusive conditional density in Eq.~\ref{eq:supp_var_sol2_discrete}. The second of these distributions was analytically derived in Eq.~\ref{eq:relationship_between_dists_disc}, and we can formally introduce an action dependence because the maximally diffusive path distribution is action-independent. Note that for the first time in our derivation we are making use of the Markov property to express our system's dynamics. However, since the analytically-derived optimal transition dynamics are Markovian, the synthesized controller will attempt to make the agent's true dynamics satisfy the Markov property as a result of the underlying optimization, which makes this a benign assumption under our framework. We note that the more general problem description in Eq.~\ref{eq:KL_control_obj} does not require us to assume that our dynamics are Markovian because we are minimizing the KL divergence between the trajectory distributions directly.

Taken together, these modifications allow us to rewrite Eq.~\ref{eq:KL_control_obj} as,
\begin{equation}\label{eq:KL_control_obj_disc}
    \underset{\pi}{\text{argmin}}  \ D_{KL}(P_{\pi}[x_{1:N},u_{1:N}]||P_{max}^l[x_{1:N},u_{1:N}]).
\end{equation}
Then, working from the definition of the KL divergence we have
\begin{align*}
    D_{KL}(P_{\pi}[x_{1:N},u_{1:N}]||P_{max}^l[x_{1:N},u_{1:N}]) &=  E_{P_{\pi}}\Bigg[ \log \frac{P_{\pi}[x_{1:N},u_{1:N}]}{P_{max}^l[x_{1:N},u_{1:N}]}\Bigg] \nonumber\\
    &= E_{P_{\pi}} \Bigg[ \log \prod_{t=1}^N \frac{p(x_{t+1}|x_t,u_t)\pi(u_t|x_t)}{p_{max}(x_{t+1}|x_t)e^{-l(x_t,u_t)}}\Bigg]\nonumber\\
    &= E_{P_{\pi}} \Bigg[ \sum_{t=1}^N \log \frac{p(x_{t+1}|x_t,u_t)\pi(u_t|x_t)}{p_{max}(x_{t+1}|x_t)e^{-l(x_t,u_t)}}\Bigg]\nonumber\\
    &= E_{P_{\pi}} \Bigg[ \sum_{t=1}^N l(x_t,u_t) +\log \frac{p(x_{t+1}|x_t,u_t)\pi(u_t|x_t)}{p_{max}(x_{t+1}|x_t)}\Bigg].
\end{align*}
At this point, we explicitly introduce a temperature-like parameter, $\alpha>0$, to balance between the terms of our objective, as mentioned in the previous section and in the main text. We note that this is a benign modification because equivalent to scaling our costs or rewards by $1/\alpha$, and leads to the following result:
\begin{equation}
    \label{eq:soc_maxdiff_workingpoint}
    D_{KL}(P_{\pi}||P_{max}^l) = E_{P_{\pi}} \Bigg[ \sum_{t=1}^N l(x_t,u_t) +\alpha\log \frac{p(x_{t+1}|x_t,u_t)\pi(u_t|x_t)}{p_{max}(x_{t+1}|x_t)}\Bigg].
\end{equation}

With this result we are now able to write our final expression for an equivalent SOC representation of the KL control problem in Eq.~\ref{eq:KL_control_obj}:
\begin{equation}\label{eq:supp_soc_maxdiff}
    \pi_{\text{MaxDiff}}^* = \underset{\pi}{\text{argmin}}  \ E_{(x_{1:N},u_{1:N})\sim P_{\pi}} \Bigg[\sum_{t=1}^{N} \gamma^t\hat{l}(x_t,u_t) \Bigg],
\end{equation}
where we introduced the discounting factor $\gamma$, and with
\begin{equation}\label{eq:supp_soc_maxdiff_running_cost}
    \hat{l}(x_t,u_t) = l(x_t,u_t)+\alpha \log \frac{p(x_{t+1}|x_t,u_t)\pi(u_t|x_t)}{p_{max}(x_{t+1}|x_t)},
\end{equation}
as our modified running cost function, which concludes our derivation of the formal equivalence between the KL control and SOC MaxDiff trajectory synthesis problems. When we modify the objective above by instead maximizing a reward function $\hat{r}(x_t,u_t)$ with $r(x_t,u_t)=-l(x_t,u_t)$, we refer to this objective as the \textit{MaxDiff RL} objective as we have done in the main text, and whose properties we will explore in the following section.

\subsubsection{Maximum diffusion reinforcement learning}
\label{sec:maxdiff_RL}

The objective we derived in the previous section in Eq.~\ref{eq:supp_soc_maxdiff_running_cost} is the standard form of the MaxDiff RL objective, as discussed in the main text. In this section, we will explore the relationship between MaxDiff RL and MaxEnt RL, and prove some properties of MaxDiff RL agents. Central to these discussions is the way that temporal correlations and controllability play a role in our theoretical framework. For this reason, we first formalize and define a particular notion of controllability in the context of MDPs that was partially introduced in~\cite{Todorov2007}, implicit in the results of~\cite{Todorov2009}, and explicitly called out in~\cite{Rawlik2012}.
\begin{definition}
    \label{def:full_controllability}
    The state transition dynamics, $p(x_{t+1}|x_t,u_t)$, in an  MDP, $(\mathcal{X},\mathcal{U},p,r,\gamma)$, are fully controllable when there exists a policy, $\pi: \mathcal{U}\times\mathcal{X}\rightarrow [0,\infty)$, such that:
    \begin{equation}
        \label{eq:full_control_transition}
        p_{\pi}(x_{t+1}|x_t) = E_{u_t\sim \pi(\cdot|x_t)}[p(x_{t+1}|x_t,u_t)]
    \end{equation}
    and
    \begin{equation}
        \label{eq:full_control_KL}
        D_{KL}\Big(p_{\pi}(x_{t+1}|x_t) \Big|\Big| \nu(x_{t+1}|x_t)\Big) = 0, \ \ \ \forall t \in \mathbb{Z}^+
    \end{equation}
    for any arbitrary choice of state transition probabilities, $\nu:\mathcal{X}\times\mathcal{X}\rightarrow [0,\infty)$.
\end{definition}
\noindent Thus, a system is \textit{fully controllable} when it is simultaneously capable of reaching every state and controlling \textit{how} each state is reached. In other words, a fully controllable agent can arbitrarily manipulate its state transition probabilities, $p_{\pi}(x_{t+1}|x_t)$, by using an optimized policy to match any desired transition probabilities, $\nu(x_{t+1}|x_t)$. Whether the underlying policy is deterministic or stochastic is irrelevant to Definition~\ref{def:full_controllability}. However, our interpretation of $p_{\pi}(x_{t+1}|x_t)$ is different in either setting. When the policy is stochastic we interpret the agent's state transition dynamics due to a policy as
\begin{equation}
    \label{eq:controlled_transitions_stoch}
    p_{\pi}(x_{t+1}|x_t)=\int_{\mathcal{U}}p(x_{t+1}|x_t,u_t)\pi(u_t|x_t)du_t,
\end{equation}
where the integral over control actions arises from the expectation in Eq.~\ref{eq:full_control_transition}. Alternatively, in the deterministic case the agent's state transition dynamics are given by
\begin{equation}
    \label{eq:controlled_transitions_det}
    p_{\pi}(x_{t+1}|x_t)=\int_{\mathcal{U}}p(x_{t+1}|x_t,u_t)\delta(u_t-\tau_{\pi}(x_t))du_t = p(x_{t+1}|x_t,\tau_{\pi}(x_t)),
\end{equation}
where action sequences are drawn from $\pi(u_t|x_t) = \delta(u_t-\tau_{\pi}(x_t))$, which is a Dirac delta where $u_t=\tau_{\pi}(x_t)$ is some deterministic function of the current state~\cite{Rawlik2012}.

Equipped with our definition of full controllability, we may now shed a light on the relationship between our MaxDiff RL framework and the broader MaxEnt RL literature~\cite{Haarnoja2017,Haarnoja2018,Eysenbach2022}, and present one of our main theorems.
\begin{theorem}
    \label{thm:maxdiff_generalization}
    (Theorem 1 of Main Text) Let the state transition dynamics due to a policy $\pi$ be $p_{\pi}(x_{t+1}|x_t)$. If the state transition dynamics are assumed to be decorrelated, then the optimum of Eq.~\ref{eq:supp_soc_maxdiff_running_cost} is reached when $D_{KL}(p_{\pi}||p_{max}) = 0$ and the MaxDiff RL objective reduces to the MaxEnt RL objective.
\end{theorem}
\begin{proof}
    Our goal in this proof will be to take the MaxDiff RL objective function in Eq.~\ref{eq:supp_soc_maxdiff} and explore its relationship to the MaxEnt RL objective. Neglecting the discounting factor $\gamma$ but without loss of generality, we begin our proof by algebraically manipulating the MaxDiff RL objective function in Eq.~\ref{eq:supp_soc_maxdiff}:
    \begin{align}
        E_{P_{\pi}} \Bigg[\sum_{t=1}^{N} \hat{l}(x_t,u_t) \Bigg] = E_{P_{\pi}} \Bigg[ \sum_{t=1}^N &l(x_t,u_t) + \alpha\log \frac{p(x_{t+1}|x_t,u_t)\pi(u_t|x_t)}{p_{max}(x_{t+1}|x_t)}\Bigg] \nonumber\\
        = E_{P_{\pi}} \Bigg[ \sum_{t=1}^N &l(x_t,u_t) \Bigg] + \sum_{t=1}^N E_{(x_t,u_t)\sim p,\pi} \Bigg[\alpha \log \frac{p(x_{t+1}|x_t,u_t)\pi(u_t|x_t)}{p_{max}(x_{t+1}|x_t)}\Bigg] \nonumber\\
        = E_{P_{\pi}} \Bigg[ \sum_{t=1}^N &l(x_t,u_t) + \alpha\log \pi(u_t|x_t)\Bigg] \nonumber\\
        &+ \sum_{t=1}^N E_{(x_t,u_t)\sim p,\pi} \Bigg[\alpha\log \frac{p(x_{t+1}|x_t,u_t)}{p_{max}(x_{t+1}|x_t)}\Bigg]. \nonumber
    \end{align}
    So far, we have merely rearranged the terms in the MaxDiff RL objective by taking advantage of the linearity of expectations and the definition of $P_{\pi}$ in Eq.~\ref{eq:policy_dists}. Now, we proceed by applying Jensen's inequality to the last term of our expression above---bringing in the expectation over control actions into the logarithm, noting that $E_{u_t\sim\pi}[p_{max}(x_{t+1}|x_t)]=p_{max}(x_{t+1}|x_t)$, and doing more algebraic manipulations:
    \begin{align}
        &\leq E_{P_{\pi}} \Bigg[ \sum_{t=1}^N l(x_t,u_t) + \alpha\log \pi(u_t|x_t)\Bigg] + \sum_{t=1}^N E_{x_t\sim p} \Bigg[\alpha\log \frac{E_{u_t\sim \pi}[p(x_{t+1}|x_t,u_t)]}{p_{max}(x_{t+1}|x_t)}\Bigg] \nonumber\\
        &\leq E_{P_{\pi}} \Bigg[ \sum_{t=1}^N l(x_t,u_t) + \alpha\log \pi(u_t|x_t)\Bigg] + \sum_{t=1}^N E_{x_t\sim p} \Bigg[\alpha\log \frac{p_{\pi}(x_{t+1}|x_t)}{p_{max}(x_{t+1}|x_t)}\Bigg]\nonumber\\
        &\leq E_{P_{\pi}} \Bigg[\sum_{t=1}^N l(x_t,u_t) + \alpha\log \pi(u_t|x_t) + \alpha D_{KL}\big(p_{\pi}(x_{t+1}|x_t) \big|\big| p_{max}(x_{t+1}|x_t)\big)   \Bigg],\label{eq:proof_ineq_maxdiff}
    \end{align}
    where we also used the definition of $p_{\pi}(x_{t+1}|x_t)$ from Eq.~\ref{eq:full_control_transition}. 
    
    To conclude our proof, we must show that the MaxEnt RL objective emerges from the MaxDiff RL objective under the assumption that an agent's state transitions are decorrelated. We can formalize what decorrelation requires of an agent in one of two contexts---that of agents with continuous experiences, or in general. Our derivation throughout Supplementary Note~\ref{sec:maxdiff_theory} achieves this in the context of agents with continuous experiences. Therein, we proved that the least-correlated continuous agent paths uniquely satisfy maximally diffusive trajectory statistics, which requires that $D_{KL}(p_{\pi}|| p_{max}) = 0$ when there exists an optimizing policy $\pi$. Alternatively, completely decorrelating the state transitions of an agent in general requires being able to generate arbitrary jumps between states---as discussed in the main text---which requires full controllability (see Definition~\ref{def:full_controllability}). Given full controllability, the optimum of Eq.~\ref{eq:proof_ineq_maxdiff} is also reached when $D_{KL}(p_{\pi}|| p_{max}) = 0$.
    
    Applying the assumption of decorrelated state transitions in either of the two senses expressed above not only simplifies Eq.~\ref{eq:proof_ineq_maxdiff} by removing the KL divergence term but also by saturating Jensen's inequality, which recovers the equality between the left and right hand sides of our equations:
    \begin{equation}
        E_{P_{\pi}} \Bigg[\sum_{t=1}^{N} \hat{l}_c(x_t,u_t) \Bigg] = E_{P_{\pi}} \Bigg[\sum_{t=1}^N l(x_t,u_t) + \alpha\log \pi(u_t|x_t) \Bigg],\nonumber
    \end{equation}
    where we added the subscript $c$ to indicate that this applies under the assumption of decorrelated state transitions---either in the context of agents with continuous paths (with maximum diffusivity as a necessary condition) or in general (with full controllability as a sufficient condition). Putting together our final results, we may now write down the simplified MaxDiff RL optimization objective with the added assumption of decorrelated state transitions:
    \begin{equation}\label{eq:soc_maxent}
        \pi^* = \underset{\pi}{\text{argmin}}  \ E_{(x_{1:N},u_{1:N})\sim P_{\pi}} \Bigg[\sum_{t=1}^{N} \gamma^t\hat{l}_c(x_t,u_t) \Bigg],
    \end{equation}
    where we reintroduced $\gamma$, and with 
    \begin{equation}\label{eq:soc_maxent_running_cost}
        \hat{l}_c(x_t,u_t) = l(x_t,u_t)+\alpha \log \pi(u_t|x_t),
    \end{equation}
    or equivalently, we can write Eq.~\ref{eq:soc_maxent} as a maximization by replacing the cost with a reward function:
    \begin{equation}\label{eq:soc_maxent_running_reward}
        \hat{r}_c(x_t,u_t) = r(x_t,u_t)+\alpha \mathcal{H}(\pi(u_t|x_t)),
    \end{equation}
    where we briefly changed our entropy notation, using $\mathcal{H}(\pi(u_t|x_t))=S[\pi(u_t|x_t)]$, to highlight similarities with other results in the literature. Crucially, we recognize this objective as the MaxEnt RL objective, which proves that MaxDiff RL is a strict generalization of MaxEnt RL to agents with temporally correlated experiences and concludes our proof. We note that this also proves that maximizing policy entropy does not decorrelate state transitions in general because maximizing policy entropy does not minimize $D_{KL}(p_{\pi}||p_{max})$. 
\end{proof}
\noindent In contrast to MaxEnt RL, when the agent-environment state transition dynamics introduce temporal correlations, the MaxDiff RL objective continues to prioritize effective exploration by decorrelating state transitions and encouraging the system to realize maximally diffusive trajectories. As we have shown above, MaxEnt RL's strategy of decorrelating action sequences is only as effective as MaxDiff RL's strategy of decorrelating state sequences when the underlying system's dynamics do not introduce temporal correlations on their own.

Now, we turn to analyzing the formal properties of MaxDiff RL agents. In particular, we will analyze how the ergodic properties of maximally diffusive trajectories (i.e., Theorem~\ref{thm:ergodicity}) can have an impact on the learning performance of MaxDiff RL agents. Namely, on their single-shot learning capabilities and their robustness to seeds and initializations. Prior to proceeding formally, we must first provide a framework from which to asses the learning performance of RL agents. To this end, we will make use of the representation-agnostic probably approximately correct in Markov decision processes (PAC-MDP) framework~\cite{Strehl2006,Strehl2009}. 
\begin{definition}
    \label{def:supp_pac-mdp}
    An algorithm $\mathcal{A}$ is said to be PAC-MDP (Probably Approximately Correct in Markov Decision Processes) if, for any $\epsilon>0$ and $\delta \in (0,1)$, a policy $\pi$ can be produced with $poly(|\mathcal{X}|, |\mathcal{U}|, 1/\epsilon, 1/\delta, 1/(1-\gamma))$ sample complexity that is at least $\epsilon$-optimal with probability at least $1-\delta$. In other words, if $\mathcal{A}$ satisfies
    \begin{equation}
        \label{eq:supp_pac_mdp_condition}
        {\normalfont\text{Pr}}\big(\mathcal{V}_{\pi^*}(x_0)-\mathcal{V}_{\pi}(x_0)\leq \epsilon\big) \geq 1-\delta
    \end{equation}
    with polynomial sample complexity for all $x_0\in\mathcal{X}$, where 
    \begin{equation}
        \label{eq:value_function}
        \mathcal{V}_{\pi}(x_t) = E_{p,\pi}\Bigg[\sum_{n=0}^{\infty}\gamma^n r(x_{n+t},u_{n+t})\Bigg|x_t=x\Bigg]
    \end{equation}
    is the value function and $\mathcal{V}_{\pi^*}(x_t)$ is the optimal value function, then $\mathcal{A}$ is PAC-MDP.
\end{definition}
\noindent Thus, this framework states that an algorithm $\mathcal{A}$ is PAC-MDP when it is capable of learning a policy with polynomial sample complexity that can get arbitrarily close to the optimal policy with arbitrarily high probability. We note that this framework is representation-agnostic in the sense that, regardless of whether $\mathcal{A}$ involves any kind of neural network representation, any algorithm that satisfies Definition~\ref{def:supp_pac-mdp} is guaranteed to be at least $\epsilon$-optimal.

Given our definition of the PAC-MDP framework, we will now consider the implications of our results on single-shot learning. Most applications of deep RL take place in episodic environments where after each execution of a given task, the agent and environment are reset, and their initial conditions are randomized. This is the setting that we have referred to as ``multi-shot'' learning throughout our manuscript. However, learning outside of episodic environments is crucial to real-world applications of deep RL~\cite{Chen2022_yolo, Lu2021_reset}. In non-episodic tasks, agents are expected to learn within a single deployment without resetting the task or environment, which we have referred to as the ``single-shot'' learning setting throughout our manuscript. With this in mind, we are now able to state our next formal result in terms of the PAC-MDP framework.
\begin{theorem}
    \label{thm:maxdiff_single}
    (Theorem 3 of Main Text) If there exists a PAC-MDP algorithm $\mathcal{A}$ with policy $\pi^{max}$ for the MaxDiff RL objective (Eq.~\ref{eq:supp_soc_maxdiff_running_cost}), then the Markov chain induced by $\pi^{max}$ is ergodic, and any individual initialization of $\mathcal{A}$ will asymptotically satisfy the same $\epsilon$-optimality as an ensemble of initializations.
\end{theorem}
\begin{proof}
    This theorem follows directly from the ergodicity of maximally diffusive trajectories (Theorem~\ref{thm:ergodicity}), some basic facts about MDPs~\cite{Puterman2014}, and the application of Birkhoff's ergodic theorem~\cite{Hairer2018} onto our definition of PAC-MDP (Definition~\ref{def:supp_pac-mdp}). First, since $\mathcal{A}$ is capable of producing an $\epsilon$-optimal policy, $\pi^{max}$, we take $D_{KL}(p_{\pi^{max}}||p_{max})\approx 0$ for some choice of $\epsilon$, given that $p_{\pi^{max}}(x_{t+1}|x_t) = \int_{\mathcal{U}}p(x_{t+1}|x_t,u_t)\pi^{max}(u_t|x_t)du_t$. Then, it is well-known that any given policy in an MDP gives rise to a Markov chain on the state-space of the MDP~\cite{Puterman2014}. Naturally, the properties of the policy-induced Markov chain depend on the properties of the resulting state transition kernel (e.g., $p_{\pi}(x_{t+1}|x_t)$).
    
    Now, let $\{x_t\}_{t\in\mathbb{N}}$ be a Markov chain with state transition properties determined by $p_{\pi^{max}}(x_{t+1}|x_t)$. Because we know that $D_{KL}(p_{\pi^{max}}||p_{max})\approx 0$, the Markov chain described by $p_{\pi^{max}}(x_{t+1}|x_t)$ is ergodic (per Theorem~\ref{thm:ergodicity}) with invariant measure $\rho$. To proceed further, we will now state Birkhoff's well-known ergodic theorem~\cite{Hairer2018,Moore2015}.
    \begin{theorem}
        \label{thm:birkhoff}
        (Birkhoff's ergodic theorem) Let $\{x_t\}_{t\in\mathbb{N}}$ be an aperiodic and irreducible Markov process on a state space $\mathcal{X}$ with invariant measure $\rho$ and let $f:\mathcal{X}\rightarrow\mathbb{R}$ be any measurable function with $E[|f(x)|]<\infty$. Then, one has 
        \begin{equation}
            \lim_{T\rightarrow\infty}\frac{1}{T}\sum_{t=0}^T f(x_t) = E_{x_0\sim\rho}[f(x_0)]
        \end{equation}
        almost surely.
    \end{theorem}
    \noindent In other words, Birkhoff's ergodic theorem states the the time-average of any function of an ergodic Markov chain is equal to its ensemble average. 

    Now, we return to the definition of PAC-MDP to slightly manipulate the expression:
    \begin{equation}
        \begin{split}
            {\normalfont\text{Pr}}\big(\mathcal{V}_{\pi^*}(x_0)-\mathcal{V}_{\pi^{max}}(x_0)\leq \epsilon\big) &\geq 1-\delta\notag\\
            E_{x_0\sim\rho}\Big[\mathbf{1}\{\mathcal{V}_{\pi^*}(x_0)-\mathcal{V}_{\pi^{max}}(x_0)\leq \epsilon\}\Big]&\geq 1-\delta,\notag
        \end{split}
    \end{equation}
    where $\mathbf{1}\{\cdot\}$ denotes an indicator function. In other words, to be PAC-MDP is equivalent to being at least $\epsilon$-optimal on average at least $100\times (1-\delta)\%$ of episodes. To conclude our proof, let
    \begin{equation}
        f(x_t) = \mathbf{1}\{\mathcal{V}_{\pi^*}(x_t)-\mathcal{V}_{\pi^{max}}(x_t)\leq \epsilon\} \notag
    \end{equation}
    be an observable, which we note is bounded, and apply Birkhoff's theorem. Then, we will have
    \begin{equation}
        \lim_{T\rightarrow\infty}\frac{1}{T}\sum_{t=0}^T \mathbf{1}\{\mathcal{V}_{\pi^*}(x_t)-\mathcal{V}_{\pi^{max}}(x_t)\leq \epsilon\} = E_{x_0\sim\rho}[\mathbf{1}\{\mathcal{V}_{\pi^*}(x_0)-\mathcal{V}_{\pi^{max}}(x_0)\leq \epsilon\}],\notag
    \end{equation}
    which proves that any individual initial condition will satisfy the ensemble average. In turn, we have
    \begin{equation}
        {\normalfont\text{Pr}}\big(\mathcal{V}_{\pi^*}(x_0)-\mathcal{V}_{\pi^{max}}(x_0)\leq \epsilon\big) \geq 1-\delta \implies \lim_{T\rightarrow\infty}\frac{1}{T}\sum_{t=0}^T  \mathbf{1}\{\mathcal{V}_{\pi^*}(x_t)-\mathcal{V}_{\pi^{max}}(x_t)\leq \epsilon\} \geq 1-\delta\notag
    \end{equation}
    almost surely, which proves that an algorithm that is PAC-MDP during multi-shot (episodic) learning is guaranteed to be PAC-MDP during single-shot (non-episodic) learning if the underlying Markov chain induced by the policy is ergodic. This concludes our proof.
\end{proof}
\noindent An important note about this proof is that it clarifies why ergodic sampling along continuous Markovian trajectories is the best possible alternative to \textit{i.i.d.} sampling (via Theorem~\ref{thm:birkhoff}). As a result of ergodicity, the sampling statistics of any individual realization of an ergodic process are indistinguishable from \textit{i.i.d.} sampling asymptotically, almost surely.

Next, we will prove that any PAC-MDP algorithm applied onto the MaxDiff RL objective will be robust to initial conditions and random seeds, which is a highly desirable property of deep RL agents.

\begin{theorem}
    \label{thm:maxdiff_robust}
    (Theorem 2 of Main Text) If there exists a PAC-MDP algorithm $\mathcal{A}$ with policy $\pi^{max}$ for the MaxDiff RL objective (Eq.~\ref{eq:supp_soc_maxdiff_running_cost}), then the Markov chain induced by $\pi^{max}$ is ergodic, and $\mathcal{A}$ will be asymptotically $\epsilon$-optimal regardless of initialization.
\end{theorem}
\begin{proof}
    The proof of this theorem is simple given the proof to Theorem~\ref{thm:maxdiff_single}. Once again, let 
    \begin{equation}
        f(x_t) = \mathbf{1}\{\mathcal{V}_{\pi^*}(x_t)-\mathcal{V}_{\pi^{max}}(x_t)\leq \epsilon\} \notag
    \end{equation}
    be an observable. Now, let $\{x_t\}_{t\in\mathbb{N}}$ and $\{x'_t\}_{t\in\mathbb{N}}$ both be ergodic Markov chains with identical transition kernels given by $p_{\pi^{max}}$, but with different initial conditions $x_0,x'_0\in\mathcal{X}$. Then, since Birkhoff's ergodic theorem guarantees that the time-averages of observables from $\{x_t\}_{t\in\mathbb{N}}$ and $\{x'_t\}_{t\in\mathbb{N}}$ will converge to the same unique ensemble average over the invariant measure $\rho$ (Theorem~\ref{thm:birkhoff}), the following is true:
    \begin{equation}
        \lim_{T\rightarrow\infty}\frac{1}{T}\sum_{t=0}^T|f(x_t)-f(x'_t)|= 0\notag
    \end{equation}
    for any $x_0,x'_0\in\mathcal{X}$ almost surely. This proves that any PAC-MDP algorithm is guaranteed to be robust to random seeds and environmental initializations if the underlying Markov chain induced by the policy is ergodic, which concludes our proof.
\end{proof}
\noindent Thus, we have now proven that MaxDiff RL generalizes MaxEnt RL, that MaxDiff RL agents are capable of single-shot learning, and that MaxDiff RL agents are robust to initial conditions and random seeds---all because of the deep connection between maximally diffusive trajectories and ergodicity.

An interesting aside is that the MaxDiff RL objective formally requires model-based techniques to optimize because of its dependence on the system's state transition dynamics. In this sense, MaxEnt RL is the best one can do in a model-free setting---yet, with model-based techniques better performance is attainable when the system dynamics introduce temporal correlations. However, if one has direct access to state transition entropy estimates, then by reformulating the objective function in Eq.~\ref{eq:supp_soc_maxdiff}, it is technically possible to extend our results to model-free algorithms, as we show in the following sections.

\subsubsection{Alternative synthesis approach via path entropy maximization}
\label{sec:exploration_entropymax}
In Supplementary Note~\ref{sec:KL_control}, we derived a synthesis approach based on KL control that optimizes exploration and task performance by making agents realize maximally diffusive trajectories. Alternatively, we can use the fact that in Supplementary Note~\ref{sec:maxdiff_theory} we derived the unique trajectory distribution $P_{max}[x(t)]$ with maximum entropy $S[P_{max}[x(t)]]$ that satisfies our constraints---which merely amount to prohibiting teleportation between states. As a result of this, we know that $S[P_{max}[x(t)]] \geq S[P_{u(t)}[x(t)]]$ with equality if and only if $P_{max}[x(t)]=P_{u(t)}[x(t)]$. Thus, instead of minimizing the KL divergence, we can instead maximize $S[P_{u(t)}[x(t)]]$, leading to the following equivalent optimization problem,
\begin{equation}\label{eq:supp_opt_obj1}
    \underset{u(t)}{\text{argmax}}  \ S[P_{u(t)}[x(t)]],
\end{equation}
whose optimum satisfies $S[P_{u^*(t)}[x(t)]]=S[P_{max}[x(t)]]$. Based on this specification, we can define several other equivalent MaxDiff trajectory synthesis problem specifications that may be more or less convenient depending on the details of the application domain:
\begin{align}\label{eq:supp_opt_obj1_forms}
    \underset{u(t)}{\text{max}}  \ S[P_{u(t)}[x(t)]], \ &\ \ \underset{u_{1:N-1}}{\text{max}}  \ S\Bigg[\prod_{t=1}^{N} p(x_{t+1}|x_t,u_t)\Bigg], \nonumber\\
    \underset{\pi}{\text{max}} \ S[P_{\pi}[x(t),u(t)]],& \ \ \ \underset{\pi}{\text{max}}  \ S\Bigg[\prod_{t=1}^{N} p(x_{t+1}|x_t,u_t)\pi(u_t|x_t)\Bigg],
\end{align}
where $P_{\pi}[x(t),u(t)]$ is a continuous-time distribution over states and control actions analogous to $P_{\pi}[x_{1:N},u_{1:N}]$, and we can think of a controller as a policy given by a Dirac delta distribution centered at $u_t$. The equivalence between the KL control and SOC formulations of the problem, and the formulation we have produced in this section, leads to
\begin{align}\label{eq:maxent_control_obj_freeen}
     &\underset{u(t)}{\text{argmin}}  \ E_{P_{u(t)}} \big[ L[x(t),u(t)] \big]-\alpha S[P_{u(t)}[x(t)]], \nonumber\\
     &\underset{\pi}{\text{argmin}}  \ E_{P_{\pi}} \big[ L[x(t),u(t)] \big]-\alpha S[P_{\pi}[x(t),u(t)]]
\end{align}
and 
\begin{align}\label{eq:maxent_control_obj_freeen2}
    \underset{u_{1:N}}{\text{argmin}}  \ E_{P_{u_{1:N}}} &\Bigg[ \sum_{t=1}^N l(x_t,u_t)-\alpha S[p(x_{t+1}|x_t,u_t)]\Bigg], \nonumber \\
    \underset{\pi}{\text{argmin}}  \ E_{P_{\pi}} \Bigg[ \sum_{t=1}^N &l(x_t,u_t)-\alpha S[p(x_{t+1}|x_t,u_t)\pi(u_t|x_t)]\Bigg]
\end{align}
also being formally equivalent to Eq.~\ref{eq:KL_control_obj}, where we omit $\gamma$. While the different objectives listed in Eqs.~\ref{eq:supp_opt_obj1_forms}-\ref{eq:maxent_control_obj_freeen2} may seem redundant, some of these may prove to be more readily applicable in particular domains, or to a given practitioner's preferred policy synthesis approach. In the following section, we derive an additional objective that attains the same optimum as Eqs.~\ref{eq:supp_opt_obj1_forms}-\ref{eq:maxent_control_obj_freeen2}, but is better suited to model-free optimizations.

\subsubsection{Simplified synthesis via local entropy maximization}
\label{sec:exploration_controllers}
As currently written, the objectives specified thus far require access to $p(x_{t+1}|x_t,u_t)$. To avoid this, we can simplify the problem by assuming that our agent's path statistics are already within a \textit{local} variational neighborhood of the optimal statistics. We formalize this optimistic assumption by asserting that our agent's path probability densities are of the following form,
\begin{equation}\label{eq:supp_dist_approx}
    P_{u(t)}^L[x(t)] = \frac{1}{Z} \exp\Big[-\frac{1}{2}\int_{t_0}^{t} \dot{x}(\tau)^T\mathbf{C}^{-1}_{u(t)}[x(\tau)]\dot{x}(\tau) d\tau\Big],
\end{equation}
where it is still the case that $S[P_{max}[x(t)]] \geq S[P_{u(t)}^L[x(t)]]$, and that the optimum can only be reached if and only if $P_{max}[x(t)]=P_{u(t)}^L[x(t)]$. Hence, by optimizing $S[P_{u(t)}^L[x(t)]]$ we merely change the direction from which our system approaches the true variational optimum of Eq.~\ref{eq:KL_control_obj}.

We proceed by analytically deriving the functional form of $S[P_{max}[x(t)]]$, and then using it to formulate our optimization of $S[P_{u(t)}^L[x(t)]]$. We begin by considering the path entropy along a finite path, $S[P_{max}[x_{1:N}]]$, where we can apply the chain rule of conditional entropies. For the reader's convenience, we state the chain rule as it is commonly formulated below:
\begin{equation}
    \label{eq:supp_cond_ent_chainrule}
    S[P[x_{1:N}]] = \sum_{t=1}^{N} S[p(x_{t+1}|x_{1:t})].
\end{equation}
Then, applying this property directly onto $P_{max}[x_{1:N}]$ we have,
\begin{equation}
    \label{eq:supp_disc_path_ent}
    S[P_{max}[x_{1:N}]]=\sum_{t=1}^{N} S[p_{max}(x_{t+1}|x_{t})] \propto \frac{1}{2}\sum_{t=1}^{N}\log \det \mathbf{C}[x_t],
\end{equation}
where we made use of the Markov property to simplify our sum over conditional entropies, and then the analytical form of the entropy of a Gaussian distribution (up to a constant offset) to reach our final expression. Thus, realizing maximally diffusive trajectories merely requires synthesizing a controller $u(t)$ or policy $\pi(\cdot|\cdot)$ that satisfies $\mathbf{C}_{u(t)}[x^*]=\mathbf{C}_{\pi}[x^*]=\mathbf{C}[x^*]$ for all $x^*\in\mathcal{X}$, which can be done through optimization.

In this way, we can arrive at the MaxDiff RL objective presented in the main text, which is expressed in terms of an instantaneous reward function, $r(x_t,u_t)$. Omitting $\gamma$, the implemented MaxDiff RL objective is the following,
\begin{equation}
    \label{eq:supp_maxdiff_ratt}
    \underset{\pi}{\text{argmax}}  \  E_{(x_{1:N},u_{1:N})\sim P_{\pi}}\Bigg[ \sum_{t=1}^{N} r(x_t,u_t)+\frac{\alpha}{2} \log \det \text{\textbf{C}}_{\pi}[x_t]\Bigg].
\end{equation}
This objective is the one that we used to produce all empirical results in the main text. Here, we would like to make a few practical notes regarding its implementation. First, in practice we may not always have guarantees on the full-rankness of $\mathbf{C}[x^*]$, which can make its determinant evaluate to zero and create numerical stability issues. To remedy this, we may take advantage of another property of the log-determinant and instead optimize $\sum_{i=1}^{M}\log \lambda_i$, where the sum is taken over the leading $M$ eigenvalues of $\mathbf{C}[x^*]$. However, it is important to note that this effectively restricts the exploration to an $M$-dimensional subspace of $\mathcal{X}$. Separately, we note that one can instead optimize the logarithm of the trace of $\mathbf{C}[x^*]$ as an approximation that drastically reduces the computational complexity of the determinant in high dimensional systems. However, this approximation only formally produces equivalent results to the log-determinant when system states vary independently from one another (i.e., when $\mathbf{C}[x^*]$ is diagonal), which is generally not the case. Nonetheless, it may be of help to a practitioner at the cost of some added distance to the assumptions underlying our formal guarantees.

Since the objective in Eq.~\ref{eq:supp_maxdiff_ratt} does not explicitly depend on $p(x_{t+1}|x_t,u_t)$, it provides us with a clear setting in which to frame model-free implementations of MaxDiff RL. Because our theoretical approach in Supplementary Note~\ref{sec:maxdiff_theory} assumes that there is a unique $\mathbf{C}[x^*]$ for all $x^*\in\mathcal{X}$, we can reinterpret $\frac{1}{2}\log\det\mathbf{C}[x^*]$ as an environmental property. In other words, because the function $S[x^*]=\frac{1}{2}\log\det\mathbf{C}[x^*]$ is a real-valued function of state that in principle does not require direct access to the state transition dynamics, we can think of $S[x^*]$ as a state-dependent property of the environment. By doing so, implementing model-free MaxDiff RL becomes very similar to existing state entropy maximization techniques in model-free RL~\cite{Lee2020_RL, Seo2021,Mutti2021}. The primary challenge is to learn a parametric estimator of $S[x^*]$---similar to those used in~\cite{Lee2020_RL, Seo2021,Mutti2021}, or~\cite{Prabhakar2022}---such that $\hat{S}_{\theta}[x^*] \approx S[x^*]$, where $\theta$ represents some vector of parameters (e.g., neural network weights). Then, given such an estimate model-free implementations are possible by augmenting the value function with $\hat{S}_{\theta}[x_t]$. Moreover, we note that our implementation of model-based MaxDiff RL makes use of data-driven estimates of $\mathbf{C}[x^*]$ during its optimization. While it may seem that evaluating $\mathbf{C}[x^*]$ still requires access to predictive system rollouts in a model-based fashion, $\mathbf{C}[x^*]$ can be empirically estimated from data moving backwards in time---in other words, $\mathbf{C}[x^*]=\int_{t_i-\Delta t}^{t_i}K_{XX}(\tau,t_i)d\tau$. However, we note that equality between the forwards and backwards data-driven estimates of  $\mathbf{C}[x^*]$ is only guaranteed for stationary processes. Nonetheless, this shows that non-parametric model-free data-driven estimates of $\mathbf{C}[x^*]$ are also possible.

\subsubsection{Example applications of MaxDiff trajectory synthesis}
\label{sec:maxdiff_exploration}
In this section, we implement MaxDiff trajectory synthesis across handful of applications outside of reinforcement learning that require both directed and undirected exploration. These should illustrate the sense in which our theoretical framework can extend beyond a particular algorithmic implementation, or even reinforcement learning as a problem setting. Moreover, here we will analyze the behavior of various dynamical systems made to follow maximally diffusive trajectories through the lens of statistical mechanics. 

\begin{figure}[t!]
    \centering
    \includegraphics[width=1.0\linewidth]{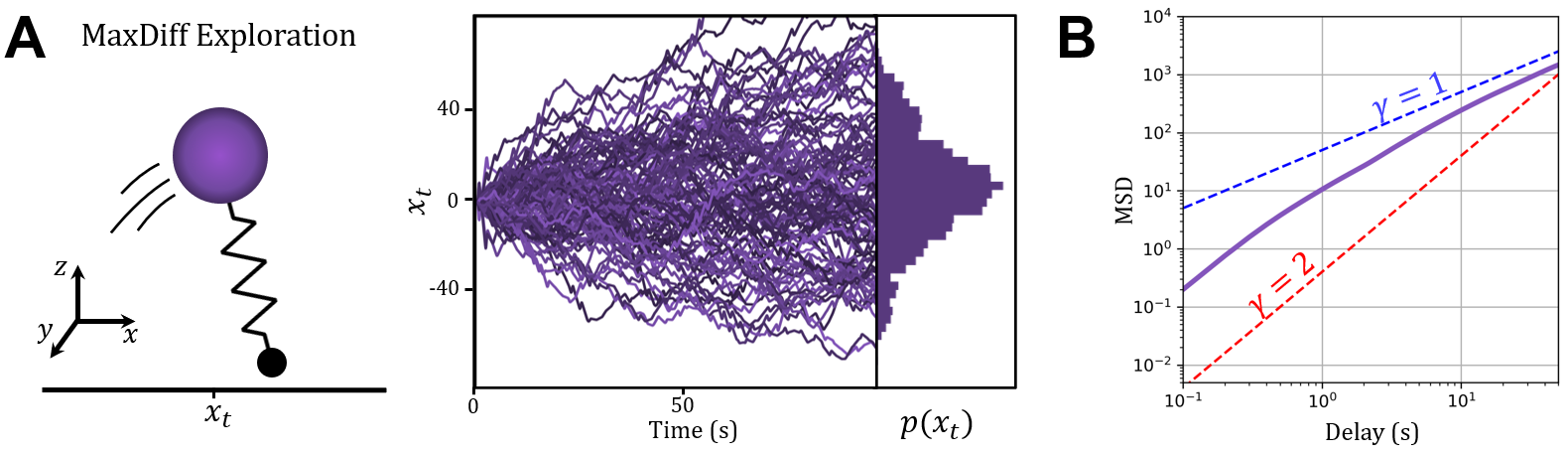}
    \caption{\textbf{Maximally diffusive trajectories of a spring-loaded inverted pendulum (SLIP)}. \textbf{a}, The SLIP model (left panel) is a 9-dimensional nonlinear and nonsmooth second-order dynamical system, which is used as a popular model of human locomotion. (right panel) We choose this system because it is far from the ideal assumptions under which our theory is formulated, and yet its sample paths behave as we expect. The sample paths of the SLIP model with MaxDiff trajectories in the one dimensional space determined by its $x$-coordinate approximately match the statistics of pure Brownian motion in one dimension. \textbf{b}, Mean squared displacement (MSD) plots give the deviation of the position of an agent over time with respect to a reference position. We can distinguish between diffusion processes by comparing the growth of their MSD over time. In general, we expect them to follow a relationship described by $\text{MSD}(x) \propto t^{\gamma}$, where $\gamma$ is an exponent that determines the different diffusion regimes (normal diffusion $\gamma=1$, superdiffusion $1<\gamma<2$, ballistic motion $\gamma\geq2$). As we can see, the behavior of the diffusing SLIP model is superdiffusive at short time-scales, but gradually becomes more like a standard diffusion process as we coarse-grain. Similar short-delay superdiffusion regimes have been observed in systems with nontrivial inertial properties~\cite{Scholz2018}, such as those of our macroscopic SLIP agent.}
    \label{fig:supp_diffusion}
\end{figure}

We begin by studying MaxDiff trajectory synthesis in the undirected exploration of a nontrivial control system---a spring-loaded inverted pendulum (SLIP) model. The SLIP model is a popular dynamic model of locomotion and encodes many important properties of human locomotion~\cite{Srinivasan2006}. In particular, we will implement the SLIP model as in~\cite{Ansari2016}, where it is described as a 9-dimensional nonlinear nonsmooth control system. The SLIP model is shown in Supplementary Fig.~\ref{fig:supp_diffusion}(a) and consists of a ``head'' which carries its mass, and a ``toe'' which makes contact with the ground. Its state-space is defined by the 3D velocities and positions of its head and toe, or $x = [x_h,\dot{x}_h,y_h,\dot{y}_h,z_h,\dot{z}_h,x_t,y_t,q]^T$, where $q=\{c,a\}$ is a variable that tracks whether the system is in contact with the ground or in the air. The SLIP dynamics are the following:
\begin{align}
    \label{eq:supp_slipdyn}
     &\dot{x} = f(x,u) = 
    \begin{cases}
        f_c(x,u),& \text{if } l_c < l_0\\
        f_a(x,u),& \text{otherwise}
    \end{cases},\notag \\ 
    f_c(x,u) = &
    \begin{bmatrix}
        \dot{x}_h \\ \frac{(k(l_0-l_s) +u_c)(x_h-x_t)}{ml_c} \\ \dot{y}_h \\ \frac{(k(l_0-l_c) +u_c)(y_h-y_t)}{ml_c} \\ \dot{z}_h \\ \frac{(k(l_0-l_c) +u_c)(z_h-z_t)}{ml_c}-g \\ 0 \\ 0
    \end{bmatrix}, \ 
    f_a(x,u) = \begin{bmatrix}
        \dot{x}_h \\ 0 \\ \dot{y}_h \\ 0 \\ \dot{z}_h \\ -g \\ \dot{x}_h+u_{t_x} \\ \dot{y}_h+u_{t_y}
    \end{bmatrix},
\end{align}
where $f_c(x,u)$ captures the SLIP dynamics during contact with the ground, and $f_a(x,u)$ captures them while in the air. During contact the SLIP can only exert a force, $u_c$, by pushing along the axis of the spring, whose resting length is $l_0$ and its stiffness is $k$. During flight the SLIP is subject to gravity, $g$, and is capable of moving the $x,y$-position of its toe by applying $u_{t_x}$ and $u_{t_y}$, respectively. To finish specifying the SLIP dynamics, and determine whether or not the spring is in contact with the ground, we define, 
\begin{equation}\nonumber
    l_c = \sqrt{(x_h-x_t)^2+(y_h-y_t)^2+(z_h-z_G)^2},
\end{equation}
which describes the distance along the length of the spring to the ground, and $z_G$ is the ground height. Rather than explore diffusively in the entirety of the SLIP model's 9-dimensional state-space, we will first demand that it only explores a 1-dimensional space described by its $x$-coordinate, starting from an initial condition of $x(0)=0$. We can think of this as a projection to a 1-dimensional subspace of the system, or equivalently as a coordinate transformation with a constant Jacobian matrix. We note that the system's nonsmoothness should break the path continuity constraint that our approach presumes to hold. However, since we use a coordinate transformation to formulate the exploration problem in terms of the system's $x$-coordinate we do not violate the assumptions of MaxDiff trajectory synthesis. This is because, while the system's velocities experience discontinuities, its position coordinates do not. In general, the use of coordinate transformations can extend the applicability of MaxDiff trajectory synthesis to even broader classes of systems than those claimed by our theoretical framework throughout Supplementary Note~\ref{sec:maxdiff_theory}. However, this will require a formal analysis of the observability properties of maximally diffusive agents, which lies outside the scope of this work.

\begin{figure}[t!]
    \centering
    \includegraphics[width=0.9999\linewidth]{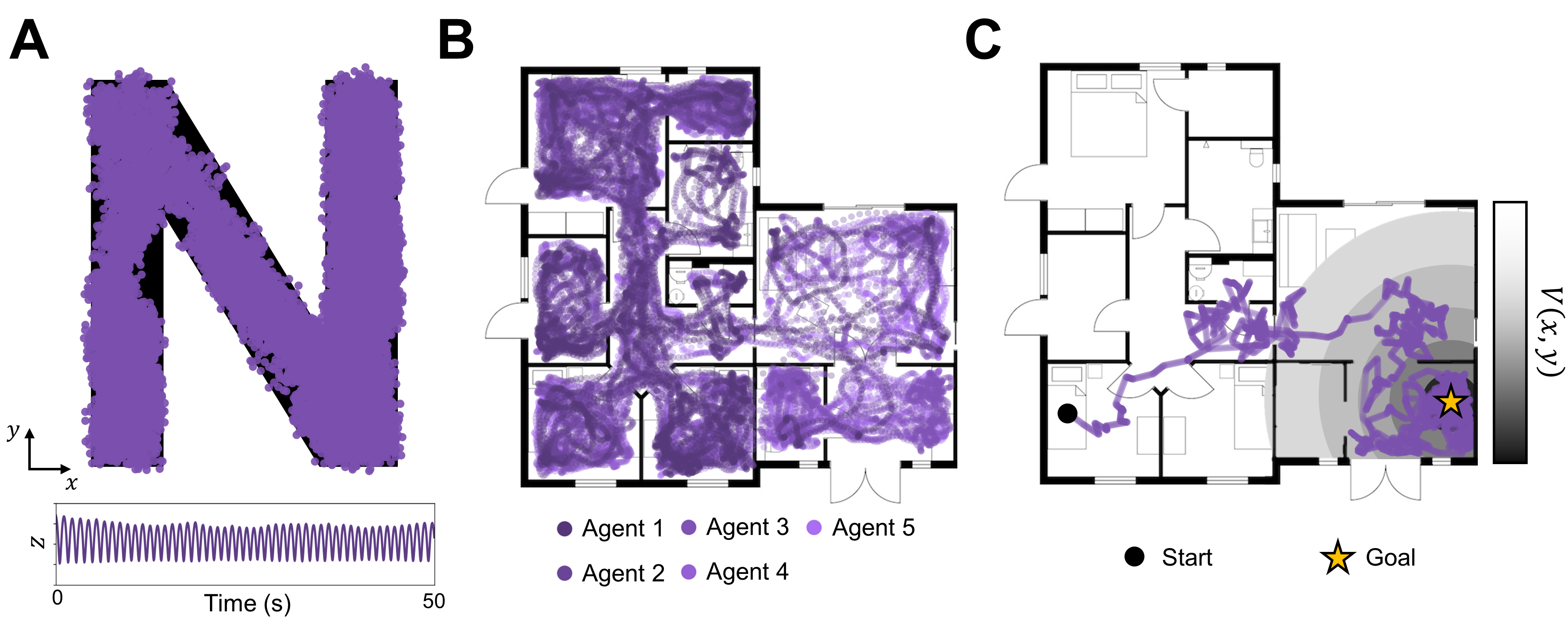}
    \caption{\textbf{SLIP maximally diffusive exploration in various settings}. \textbf{a}, Undirected maximally diffusive exploration in a constrained \textbf{N}-shaped environment. The boundaries of the environment, as well as safety constraints, are established through the use of control barrier functions, which enable safe and continuous maximally diffusive exploration without modifications to our approach. \textbf{b}, Undirected multiagent maximally diffusive exploration of more complex environment: a house's floor plan. Here, five agents with identical objectives perform maximally diffusive exploration. Because maximally diffusive exploration is ergodic, many tasks are inherently distributable between agents with linear scaling in complexity. \textbf{c}, Directed maximally diffusive exploration in a complex environment. Here, a single agent in a complex environment performs directed exploration in a potential that encodes a navigation goal.}
    \label{fig:supp_maxdiff_hopper}
\end{figure}

In order to realize maximally diffusive exploration, we make use of MPPI in conjunction with the MaxDiff trajectory synthesis objective in Eq.~\ref{eq:maxent_control_obj_freeen2}. In Supplementary Fig.~\ref{fig:supp_diffusion} we illustrate the results of this process. Supplementary Fig.~\ref{fig:supp_diffusion}(a) depicts the sample paths generated by the maximally diffusive exploration of the SLIP model's $x$-coordinate. The sample paths of the SLIP agent resemble the empirical statistics of Brownian particle paths despite the fact that the SLIP model is far from a non-inertial point mass. In Supplementary Fig.~\ref{fig:supp_diffusion}(b), we study the fluctuations of maximally diffusive exploration from the lens of statistical mechanics. Here, we analyze the mean squared displacement (MSD) statistics of undirected maximally diffusive exploration and compare to the statistics of standard and anomalous diffusion processes. MSD plots capture the deviations of a diffusing agent from some reference position over time. In standard diffusion processes, the relationship between MSD and time elapsed is linear on average. That is, we expect the squared deviation of a diffusing agent from its initial condition to grow linearly in proportion to the time elapsed (see blue line in Supplementary Fig.~\ref{fig:supp_diffusion}(b)). However, in general there exist other diffusion regimes characterized by the growth of MSD over time. These regimes are typically determined by fitting the exponent $\gamma$ in MSD$(x)\propto t^{\gamma}$, where normal diffusion has $\gamma=1$, superdiffusion has $1<\gamma<2$, and ballistic motion has $\gamma\geq2$. The purple line in Supplementary Fig.~\ref{fig:supp_diffusion}(b) depicts the MSD statistics of the SLIP model. The diffusion generated by the SLIP model's maximally diffusive exploration has superdiffusive displacements over short-time scales owing to the the inertial properties of the system. However, as we consider longer time-scales, the behavior of the SLIP model becomes indistinguishable from standard diffusion processes with $\gamma=1$. This difference in scaling exponents has been shown to be a general property of diffusion with inertial particles and should be expected in macroscopic systems~\cite{Scholz2018}.

Keeping with the SLIP dynamical system, in Supplementary Fig.~\ref{fig:supp_maxdiff_hopper} we study the behavior of MaxDiff trajectory synthesis across various standard robotics applications. In Supplementary Fig.~\ref{fig:supp_maxdiff_hopper}(a), a single SLIP agent is performing undirected MaxDiff exploration within the bounds of an \textbf{N}-shaped environment. In this task, the agent must be able to explore its \textit{x-y} plane by hopping along, without falling or exiting the bounds of the exploration domain. To ensure the SLIP model's safety, as well as establish the bounds of the environment, we made use of control barrier functions (CBFs)~\cite{Ames2014}---a standard technique in the field for guaranteeing safety.  Then, to illustrate another application application of the ergodicity guarantees of our method, in Supplementary Fig.~\ref{fig:supp_maxdiff_hopper}(b) we apply MaxDiff trajectory synthesis to multiagent exploration in a complex environment---a house floor plan---in conjunction with CBFs. Since maximally diffusive exploration is ergodic, the outcomes of a multiagent execution and a single agent execution are asymptotically identical. In this way, maximally diffusive exploration only incurs a linear scaling in computational complexity as a function of the number of agents. Finally, in Supplementary Fig.~\ref{fig:supp_maxdiff_hopper}(c) we return to the single agent case to illustrate directed maximally diffusive exploration in the same complex environment as before. Here, a potential function encoding a goal destination is flat beyond a certain distance, which leads to undirected exploration initially. However, as the agent nears the goal, it can detect variations in the potential and follows its gradients diffusively towards the goal. 

\begin{figure}[t!]
    \centering
    \includegraphics[width=0.999\linewidth]{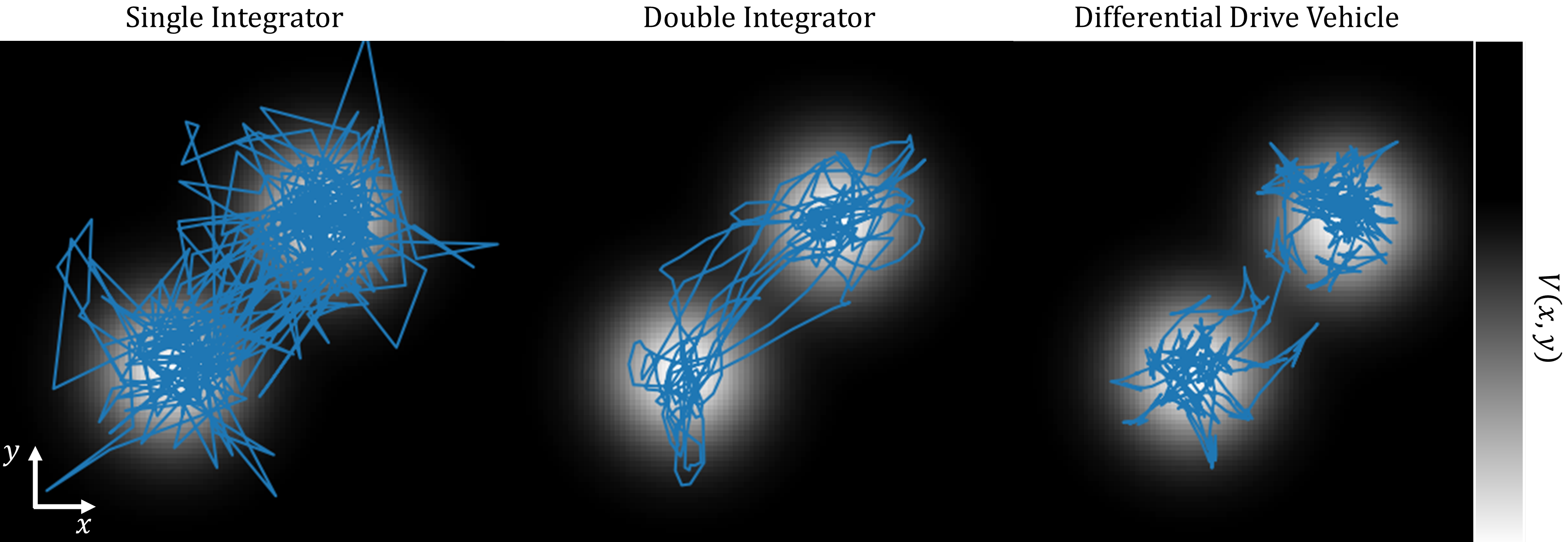}
    \caption{\textbf{Directed maximally diffusive exploration of bimodal potential across systems}. (left panel) The single integrator is a linear system whose velocities are directly determined by the controller. Hence, its sample paths behave exactly as free Brownian particles in a potential. (middle panel) The double integrator is the second-order equivalent of the single integrator system. In this system, the controller inputs acceleration commands that the system then integrates subject to its inertial properties. Despite being an inertial system, its interactions with the potential approximately follow the behavior of a Brownian particle in a potential. (right panel) The differential drive vehicle is a car-like system with simple nonlinear and nonholonomic dynamics with more complex controllability properties. Nonetheless, when we subject the differential drive vehicle to directed maximally diffusive exploration it traverses the potential as desired.}
    \label{fig:supp_systems}
\end{figure}

Now, we will highlight how the underlying properties of an agent's dynamics can affect the trajectories generated during maximally diffusive exploration. To this end, we consider a simple planar exploration task subject to a bimodal Gaussian potential ascribing a cost to system states far away from the distribution means. In Supplementary Fig.~\ref{fig:supp_systems}, we explore the planar domain with three different systems. First, exploration over the bimodal potential is shown with a single integrator system, which is a controllable first-order linear system. Since this system is effectively identical to a non-inertial point mass, its sample paths are formally the same as those of Brownian particles in a confining potential. In the middle panel of Supplementary Fig.~\ref{fig:supp_systems}, we consider a double integrator system, which is a controllable, linear, second-order system. However, for this system its diffusion tensor is degenerate because the noise only comes into the system as accelerations. Nonetheless, the system realizes ergodic coverage with respect to the underlying potential (in agreement with the theory of degenerate diffusion~\cite{Kliemann1987,Bourabee2008}). Finally, we consider the differential drive vehicle, which is a simple first-order nonlinear dynamical system with nontrivial controllability properties. Yet, the differential drive vehicle realizes ergodic coverage in the plane, as predicted by the properties of maximally diffusive systems.

As a final look into the properties of directed maximally diffusive exploration, we examine the role that the temperature parameter $\alpha$ plays on the behavior of the agent in a simpler setting. To this end, we revisit the differential drive vehicle dynamics and make use of MPPI once again to optimize our objective. However, instead of a bimodal Gaussian potential, we consider a quadratic potential centered at the origin with the system initialized at $(x,y)=(-4,-2)$. Quadratic potentials such as these are routinely implemented as cost functions throughout robotics and control theory. In Supplementary Fig.~\ref{fig:supp_temperature}, we depict the behavior of the system as a function of the temperature parameter. Initially, with the temperature set to zero the agent's paths are solely determined by the solution to the optimal control problem, smoothly driving towards the potential's minimum at the origin. Then, as we tune up $\alpha$, we increase diffusivity of our agent's sample paths. While at $\alpha=1$ the position of the system fluctuates very slightly at the bottom of the quadratic potential, at $\alpha=100$ the agent diffuses around violently by overcoming its energetic tendency to stay at the bottom of the well. If we were to continue increasing $\alpha$ to larger and larger values, we would observe that directed maximally diffusive exploration would cease to be ergodic, as predicted by~\cite{Wang2019}. This occurs as a result of the strength of diffusive fluctuations (here set by our $\alpha$ parameter) dominating the magnitude of the drift induced by the potential's gradient. This is to say that for a given problem, system, and operator preferences, there should be a range of $\alpha$ values that best achieve the task.

Throughout this section we have illustrated how maximally diffusive exploration, as formulated in Eq.~\ref{eq:maxent_control_obj_freeen2}, satisfies the behaviors predicted by our theoretical framework. Moreover, we have motivated how MaxDiff trajectory synthesis can be applied in a variety of common robotic applications while simultaneously guaranteeing safety, ergodicity, and task distributability. Broadly speaking, incorporating maximally diffusive exploration into most optimal control or reinforcement learning frameworks should be simple---particularly in light of the effort we have put towards deriving optimization objectives realizable in a broad class of application domains.

\begin{figure}[t]
    \centering
    \includegraphics[width=1.0\linewidth]{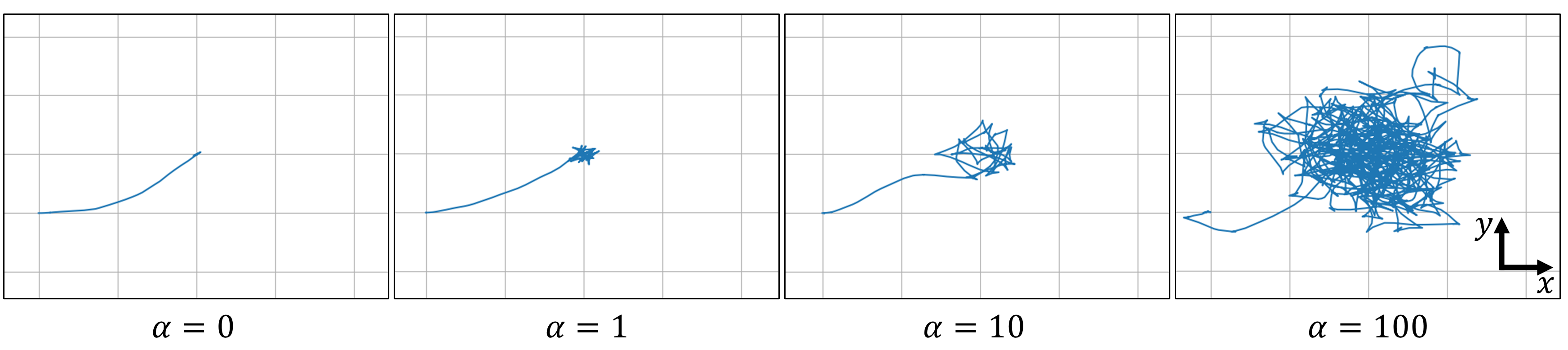}
    \caption{\textbf{Varying the $\alpha$ parameter of directed MaxDiff exploration}. Here, we are making a differential drive vehicle explore a quadratic potential centered at the origin under varying choices of $\alpha$ modulating the strength of the diffusive exploration within the potential. As we increase $\alpha$ the strength of the diffusion increases as well, leading to greater exploration of the basin of attraction of the quadratic potential well.}
    \label{fig:supp_temperature}
\end{figure}

\clearpage
\newpage

\subsection{Reinforcement learning implementation details}
\label{sec:implementation}
\subsubsection{General}
\label{sec:general}
All simulated examples use the reward functions specified MuJoCo environments unless otherwise specified \cite{brockman2016openai, todorov2012mujoco}. Supplementary Table~\ref{table:simulation_params} provides a list of all hyperparameters used in all implementations of MaxDiff RL, NN-MPPI, and SAC, for each environment. All experiments were run for a total of 1 million environment steps with each epoch being comprised of 1000 steps. For multi-shot tests, the episode was reset upon satisfying a ``done'' condition or completing the number of steps in an epoch. For single-shot tests, the environment was never reset and each epoch only constituted a checkpoint for saving cumulative rewards during the duration of that epoch.  All representations used ReLU activation functions, and 10 seeds were run for each configuration. 

For all model-based examples (i.e., MaxDiff RL and NN-MPPI), the system dynamics are represented in the following form, $x_{t+1} = x_t + f(x_t ,u_t)$, where the transition function $f(x_t, u_t)$ and reward function $r(x_t,u_t)$ are both modeled by fully-connected neural network representations. Both the reward function and transition function representations are optimized using Adam~\cite{kingma2014adam}. The network is regularized using the negative log-loss of a normal distribution where the variance, $\Sigma_\text{model} \in \mathbb{R}^{n \times n}$, is a hyperparameter that is simultaneously learned based on agent experience. The predicted reward utility is improved by the error between the predicted target and target reward equal to $\mathcal{L} = \Vert r_t + 0.95 \ r(x_{t+1}, u_{t+1}) - r(x_t, u_t) \Vert ^2$. The structure of this loss function is similar to those used in temporal-difference learning~\cite{boyan1999least, precup2001off}. The inclusion of the reward term from the next state and next action helps the algorithm learn in environments with rewards that do not strictly depend on the current state, as is the case with some MuJoCo examples.

For all model-free examples, we implement SAC to provide updates to our model-free policy. We use the hyperparameters for SAC specified by the parameters shared in~\cite{Haarnoja2018}, including the structure of the soft Q functions, but excluding the batch size parameter and the implemented policy's representation. Instead, we choose to match the batch size used during our model-based learning examples (i.e., with Maxdiff RL and NN-MPPI), and also introduce a simpler policy representation. As an alternative to the representation in~\cite{Haarnoja2018}, our policy is represented by a normal distribution parametrized by a mean function defined as a fully-connected neural network.

Reinforcement learning experiments were run on an Intel\textregistered~Xeon(R) Platinum 8380 CPU @ 2.30GHz x 160 server running Ubuntu 18.04 and Python 3.6 (\verb|pytorch| 1.7.0 and \verb|mujoco_py| 2.0). This hardware was loaned by the Intel Corporation, whose technical support we acknowledge. Finally, we note that our code repository contains a Dockerfile to facilitate validation of our results across platforms without the need to worry about package versioning.

\subsubsection{Point mass}
\label{sec:pointmass}
The goal of the point mass environment is to learn to move to the origin of a 2D environment. This is a custom environment in which the point mass dynamics are simulated as a 2D double integrator with states $[x,y,\dot{x},\dot{y}]$ and actions $[\ddot{x},\ddot{y}]$. Each episode is initialized at state $[-1,-1,0,0] + \epsilon$ where $\epsilon \sim \mathcal{N}(0,0.01)$. The reward function is specified in terms of location in the environment $r = - (x^2 + y^2)$. For multi-shot tests, the episode was terminated if the point mass exceeded a boundary defined as a square at $x,y = \pm 5$. The simulation uses RK-4 integration with a time step of 0.1.

\subsubsection{Swimmer}
\label{sec:swimmer}
The goal of the swimmer environment is to learn a gait to move forward in a 2D environment as quickly as possible. These tests use the ``v3'' variant of the OpenAI Gym MuJoCo Swimmer Environment, which includes all configuration states in the observation generated at each step. For the ``heavy-tailed'' tests, the default \texttt{xml} swimmer file is used, which includes a 3-link body with identical links. For the ``light-tailed'' tests, we modify the density of the ``tail'' link to be $10$ times lighter than other two links. The default link density in the swimmer is $1000$ and modified tail density is $100$. There is no ``done'' condition for this environment.

\subsubsection{Ant}
\label{sec:ant}
The goal of the ant environment is to learn a gait to move forward in a 3D environment as quickly as possible. These tests use the ``v3'' variant of the OpenAI Gym MuJoCo Ant Environment, which includes all configuration states in the observation generated at each step and includes no contact states. The control cost, contact cost, and healthy reward weights are all set to zero, so the modified reward function only depends on the change in the $x$-position during the duration of the step (with positive reward for progress in the positive $x$-direction). We also modified the ``done'' condition to make it possible for the ant to recover from falling. The ``done'' condition is triggered if the ant has been upside down for 1 second, and the ant is considered ``upside down'' if the torso angle that is nominally perpendicular to the ground exceeds 2.7 radians.

\subsubsection{Half-cheetah}
\label{sec:half-cheetah}
The goal of the half-cheetah environment is to learn a gait to move forward by applying torques on the joints in a 2D vertical plane. These tests use the ``v3'' variant of the OpenAI Gym MuJoCo Half-Cheetah Environment, which includes all configuration states in the observation generated at each step. There is no ``done'' condition for this environment.

\clearpage
\newpage

\addcontentsline{toc}{section}{Supplementary tables}
\section*{Supplementary tables}
\label{sec:tables_SI}

\renewcommand{\arraystretch}{1.2}

\begin{table}[ht]
\hspace{-0.65in}
\begin{tabular}{|c|cc||c||c|c|c|}
\hline
\multirow{2}{*}{\bfseries Algorithm} & \multicolumn{2}{c||}{ \multirow{2}{*}{ \bfseries Hyperparameter}}  & \bfseries Toy Problem  & \multicolumn{3}{c|}{\bfseries MuJoCo Gym (v3)}  \\
\cline{5-7}
& \multicolumn{2}{c||}{}  &\bfseries 2D Point mass & \bfseries Swimmer & \bfseries Ant & \bfseries Half-cheetah\\
\hline
\multirow{4}{*}{All} & \multicolumn{2}{c||}{  State Dim } & 4  & 10 & 29 & 18 \\
\cline{2-7}
& \multicolumn{2}{c||}{  Action Dim } & 2  & 2 & 8 & 6 \\
\cline{2-7}
& \multicolumn{2}{c||}{ Learning Rate }& 0.0005 & 0.0003 & 0.0003  & 0.0003\\ 
\cline{2-7}
& \multicolumn{2}{c||}{  Batch Size } & 128  & 128 & 256 & 256 \\
\hline 
\multirow{2}{*}{SAC} & \multicolumn{2}{c||}{  Policy Layers }  & $128\times128$  & $256\times256$ & $512\times512\times512$ & $256\times256$ \\
\cline{2-7}
& \multicolumn{2}{c||}{ Discount ($\gamma$) } &  0.99 & 0.99 &  0.99 &  0.99 \\
\cline{2-7}
& \multicolumn{2}{c||}{ Smoothing coefficient ($\tau$) } &  0.01 &  0.005 &  0.005 & 0.005 \\
\cline{2-7}
& \multicolumn{2}{c||}{  Reward Scale } & 0.25 & 100 & 5 & 5 \\
\hline
\multirow{4}{*}{NN-MPPI/} & \multicolumn{2}{c||}{  Model Layers }  &$ 128\times128 $& $200\times200$ &$512\times512\times512$& $200\times200$ \\
\cline{2-7}
\multirow{4}{*}{MaxDiff RL} & \multicolumn{2}{c||}{  Horizon } & 30 & 40 & 20 & 10\\
\cline{2-7}
& \multicolumn{2}{c||}{ Discount ($\gamma$) } &  0.95 & 0.95 & 0.95 &  0.95 \\
\cline{2-7}
\multirow{4}{*}{(Planning)}  & \multirow{2}{*}{Multi} & Samples & 500 & 500 & 1000 & 500\\
\cline{3-7}
& & Lambda  & 0.5 & 0.5 & 0.5  & 0.5\\
\cline{2-7}
& \multirow{2}{*}{SS} & Samples &  \multirow{2}{*}{NA} & 1000 &  1000  & 1000 \\
\cline{3-3}\cline{5-7}
& &  Lambda & & 0.1 & 0.5  & 0.5 \\
\hline
\multirow{6}{*}{MaxDiff RL} & \multirow{4}{*}{ Multi } &  \multirow{2}{*}{ Alpha }  & \multirow{2}{*}{5} & 1,5,10,50, & \multirow{2}{*}{ 15 } & \multirow{2}{*}{ 5 } \\
& &&   & 100,500,1000 &  &\\
\cline{3-7}
\multirow{4}{*}{(Exploration)}  &  & Dimensions   & $\begin{bmatrix}x, y, \dot{x}, \dot{y}  \end{bmatrix}$ & $\begin{bmatrix}x, y, \dot{x}, \dot{y} \end{bmatrix}$ & $\begin{bmatrix}x, y, z \end{bmatrix}$& $\begin{bmatrix}x, y, \dot{x}, \dot{y} \end{bmatrix} $\\
\cline{3-7}
&   & Weights & $[1,1,0.01,0.01] $& $[1,1,0.05,0.05] $& $[1,1,0.005]$ &  $[1,1,0.05,0.05] $\\
\cline{2-7}
&  & Alpha  & \multirow{3}{*}{NA}  & 50 & 15 & 5 \\
\cline{3-3}\cline{5-7}
& SS & Dimensions && $ \begin{bmatrix} x, y, \dot{x}, \dot{y} \end{bmatrix}$ & $\begin{bmatrix} x, y, \dot{x}, \dot{y} \end{bmatrix}$  & $\begin{bmatrix} x, y, \dot{x}, \dot{y} \end{bmatrix}$   \\ 
\cline{3-3}\cline{5-7}& &  Weights && $[1,1,0.05,0.05]$ & $[1,1,0.05,0.05] $  & $[1,1,0.05,0.05] $  \\
\hline
\end{tabular}
\vspace{0.1in}
\caption{\textbf{Simulation hyperparmeters for paper results.} ``Multi'' parameters only apply to multi-shot runs, and ``SS'' parameters only apply to single-shot runs. All weights are diagonal matrices with the values specified. }
\label{table:simulation_params}
\end{table}

\clearpage
\newpage

\begin{table}[ht]
\centering
\begin{tabular}{|l|l|ll|ll|}
\cline{3-6}
\multicolumn{2}{c|}{} 
    & \multicolumn{4}{|c|}{\textbf{Comparison}}   \\ \cline{3-6}

\multicolumn{2}{c|}{} 
    & \multicolumn{2}{c|}{\textbf{NN-MPPI}}      & \multicolumn{2}{c|}{\textbf{SAC}}          \\  \hline
\multicolumn{2}{|c|}{\textbf{Task}} & \multicolumn{1}{l|}{$P$-value}  & $P<0.05$ & \multicolumn{1}{l|}{$P$-value}  & $P<0.05$ \\ \hline
\multirow{4}{*}{Point Mass}   & $\beta=1$                                & \multicolumn{1}{l|}{$< 0.001$} & True     & \multicolumn{1}{l|}{$< 0.001$} & True     \\ \cline{2-6} 
                              & $\beta=0.1$                              & \multicolumn{1}{l|}{$< 0.001$} & True     & \multicolumn{1}{l|}{$< 0.001$} & True     \\ \cline{2-6} 
                              & $\beta=0.01$                             & \multicolumn{1}{l|}{$< 0.001$} & True     & \multicolumn{1}{l|}{$< 0.001$} & True     \\ \cline{2-6} 
                              & $\beta=0.001$                            & \multicolumn{1}{l|}{$< 0.001$} & True     & \multicolumn{1}{l|}{$< 0.001$} & True     \\ \hline
\multirow{6}{*}{Swimmer}      & Light Multi                              & \multicolumn{1}{l|}{$< 0.001$} & True     & \multicolumn{1}{l|}{$< 0.001$} & True     \\ \cline{2-6} 
                              & Light SS (Return)                        & \multicolumn{1}{l|}{0.0131}   & True     & \multicolumn{1}{l|}{0.0034}   & True     \\ \cline{2-6} 
                              & Light SS (Distance)                      & \multicolumn{1}{l|}{$< 0.001$}   & True     & \multicolumn{1}{l|}{$< 0.001$} & True     \\ \cline{2-6} 
                              & Heavy Multi                              & \multicolumn{1}{l|}{$< 0.001$} & True     & \multicolumn{1}{l|}{$< 0.001$} & True     \\ \cline{2-6} 
                              & Light-to-Heavy Multi                     & \multicolumn{1}{l|}{$< 0.001$} & True     & \multicolumn{1}{l|}{$< 0.001$} & True     \\ \cline{2-6} 
                              & Heavy-to-Light Multi                     & \multicolumn{1}{l|}{$< 0.001$} & True     & \multicolumn{1}{l|}{$< 0.001$} & True     \\ \hline
\multirow{3}{*}{Ant}          & Multi                                    & \multicolumn{1}{l|}{0.8343}   & False    & \multicolumn{1}{l|}{$< 0.001$} & True     \\ \cline{2-6} 
                              & SS (Return)                              & \multicolumn{1}{l|}{0.0154}   & True     & \multicolumn{1}{l|}{$< 0.001$} & True     \\ \cline{2-6} 
                              & SS (Distance)                            & \multicolumn{1}{l|}{$< 0.001$} & True     & \multicolumn{1}{l|}{$< 0.001$} & True     \\ \hline
\multirow{3}{*}{Half-cheetah} & Multi                                    & \multicolumn{1}{l|}{$< 0.001$} & True     & \multicolumn{1}{l|}{$< 0.001$}   & True     \\ \cline{2-6} 
                              & SS (Return)                              & \multicolumn{1}{l|}{$< 0.001$}   & True     & \multicolumn{1}{l|}{$< 0.001$} & True     \\ \cline{2-6} 
                              & SS (Distance)                            & \multicolumn{1}{l|}{$< 0.001$}   & True     & \multicolumn{1}{l|}{$< 0.001$} & True     \\ \hline
\end{tabular}
\vspace{0.1in}
\caption{\textbf{Results of statistical comparisons between MaxDiff RL and alternatives.} Across all learning experiments in the manuscript, differences between MaxDiff RL and comparisons are statistically significant ($P<0.05$) according to an unpaired two-sided Welch's t-test implementation in SciPy~\cite{SciPy2020}, except for one. However, since the Ant environment breaks ergodicity, we do not expect improvements over MaxEnt RL in multi-shot settings. ``Multi'' indicates a multi-shot experiment and ``SS'' indicates a single-shot experiment. For Multi experiments, statistical significance was determined by evaluating each final policy across 100 episodes. For SS experiments, significance was evaluated in two ways: first according to the terminal windowed return, and second according to the terminal distance traveled in each spatial navigation task. For more details, we refer readers to the Methods section of the main text, where our statistical methodology is explained and justified in detail.  We note that $P$-values below 0.001 are reported as $<0.001$.}
\label{table:stats}
\end{table}

\clearpage
\newpage 

\addcontentsline{toc}{section}{Supplementary movies}
\section*{Supplementary movies}
\label{sec:vids_SI}

\noindent Movie 1: \textbf{Effect of temperature parameter on MaxDiff RL.} Here, we depict an application of MaxDiff RL to MuJoCo's swimmer environment. To explore the role of the parameter $\alpha$ on the performance of agents, we vary it across three orders of magnitude and observe its effect on system behavior (10 seeds each). Tuning $\alpha$ is crucial because it can determine whether or not the underlying agent is ergodic. \\

\noindent Movie 2: \textbf{Robustness of MaxDiff RL across random seeds.} Here, we depict an application of MaxDiff RL to MuJoCo's swimmer environment, comparing with alternative state-of-the-art MaxEnt RL algorithms, NN-MPPI and SAC. We observe that the performance of MaxDiff RL achieves state-of-the-art performance and does not vary across seeds, which is a formal property of our framework. We test across two different system conditions: one with a light-tailed and more controllable swimmer, and one with a heavy-tailed and less controllable swimmer (10 seeds each). \\

\noindent Movie 3: \textbf{Zero-shot generalization of MaxDiff RL across embodiments.} Here, we depict an application of MaxDiff RL to MuJoCo's swimmer environment. We implement a transfer learning experiment in which neural representations are learned on a system with a given set of physical properties, and then are deployed on a system with different physical properties. We find that unlike alternative approaches, MaxDiff RL remains task capable across agent embodiments. \\

\noindent Movie 4: \textbf{Single-shot learning in MaxDiff RL agents.} Here, we depict an application of MaxDiff RL to MuJoCo's swimmer environment under a significant modification. Agents are unable to reset their environment, which requires all algorithms to learn to solve the task in a single deployment. First, we show representative snapshots of agents using representations learned in single-shot deployments, and observe that MaxDiff RL still achieves state-of-the-art performance that is robust to seeds. Then, for MaxDiff RL we also show a complete playback of an individual single-shot learning trial. We stagger the playback such that the first swimmer covers environment steps 1-2000, the next one 2001-4000, and so on for a total of 20,000 environment steps. In doing so, we visualize the single-shot learning process in real time.

\clearpage
\newpage

\addcontentsline{toc}{section}{Supplementary figures}
\section*{Supplementary figures}
\label{sec:figs_SI}

\begin{figure*}[ht]
\centering
\includegraphics[width=0.8\linewidth]{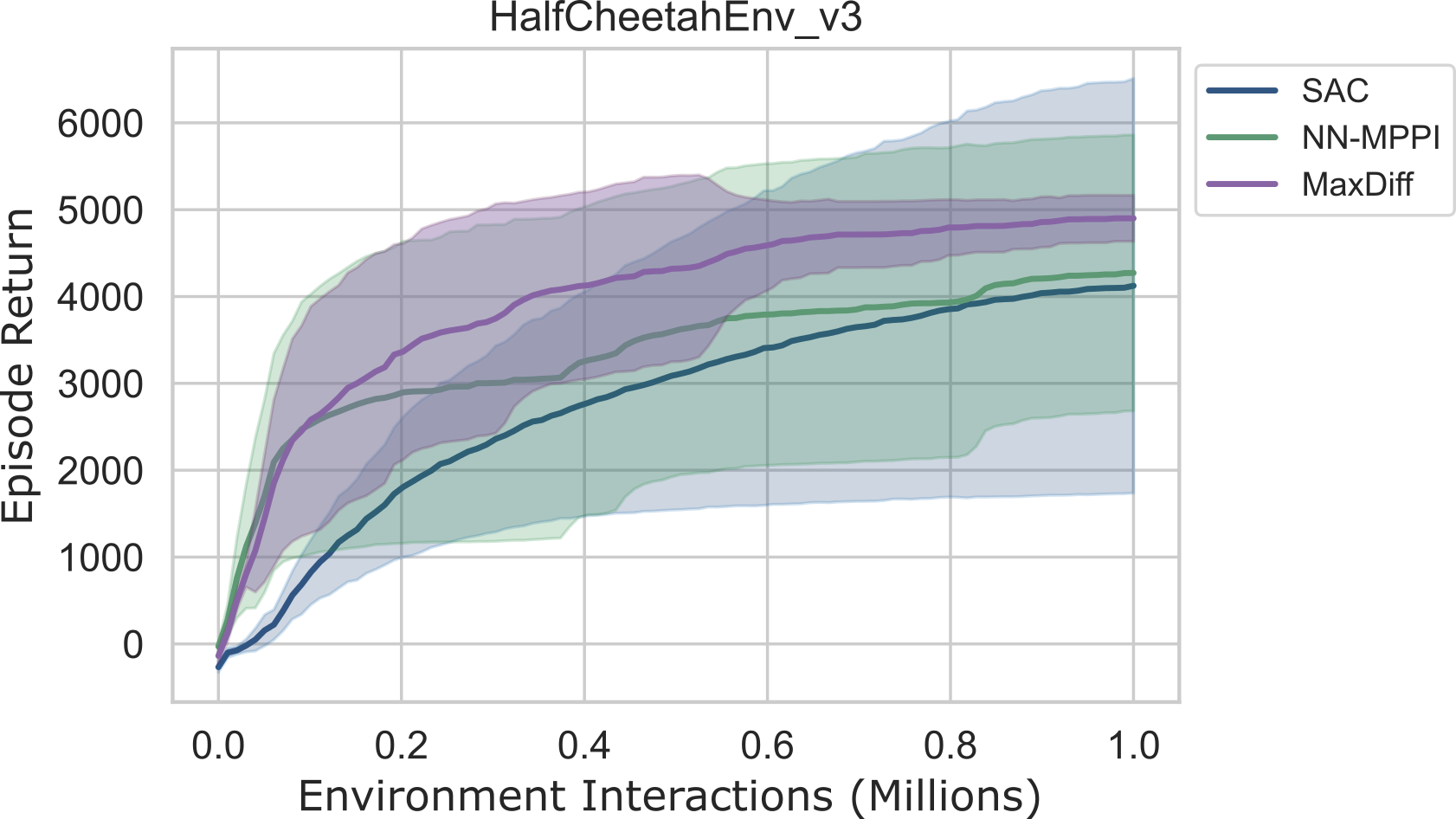}
    \caption{\textbf{Results of the half-cheetah benchmark.} This figure compares the performance of MaxDiff RL to NN-MPPI and SAC on MuJoCo's HalfCheetahEnv v3 in multi-shot. Since the half-cheetah can fall into an irreversible state (i.e., flipping upside down) this environment breaks the assumptions of MaxDiff RL. Nonetheless, we still achieve state-of-the-art performance with substantially less variance than alternative algorithms. For all reward curves, the shaded regions correspond to the standard deviation from the mean across 10 seeds.}
    \label{fig:cheetah_multi}
\end{figure*}

\begin{figure*}[ht]
\centering
\includegraphics[width=0.8\linewidth]{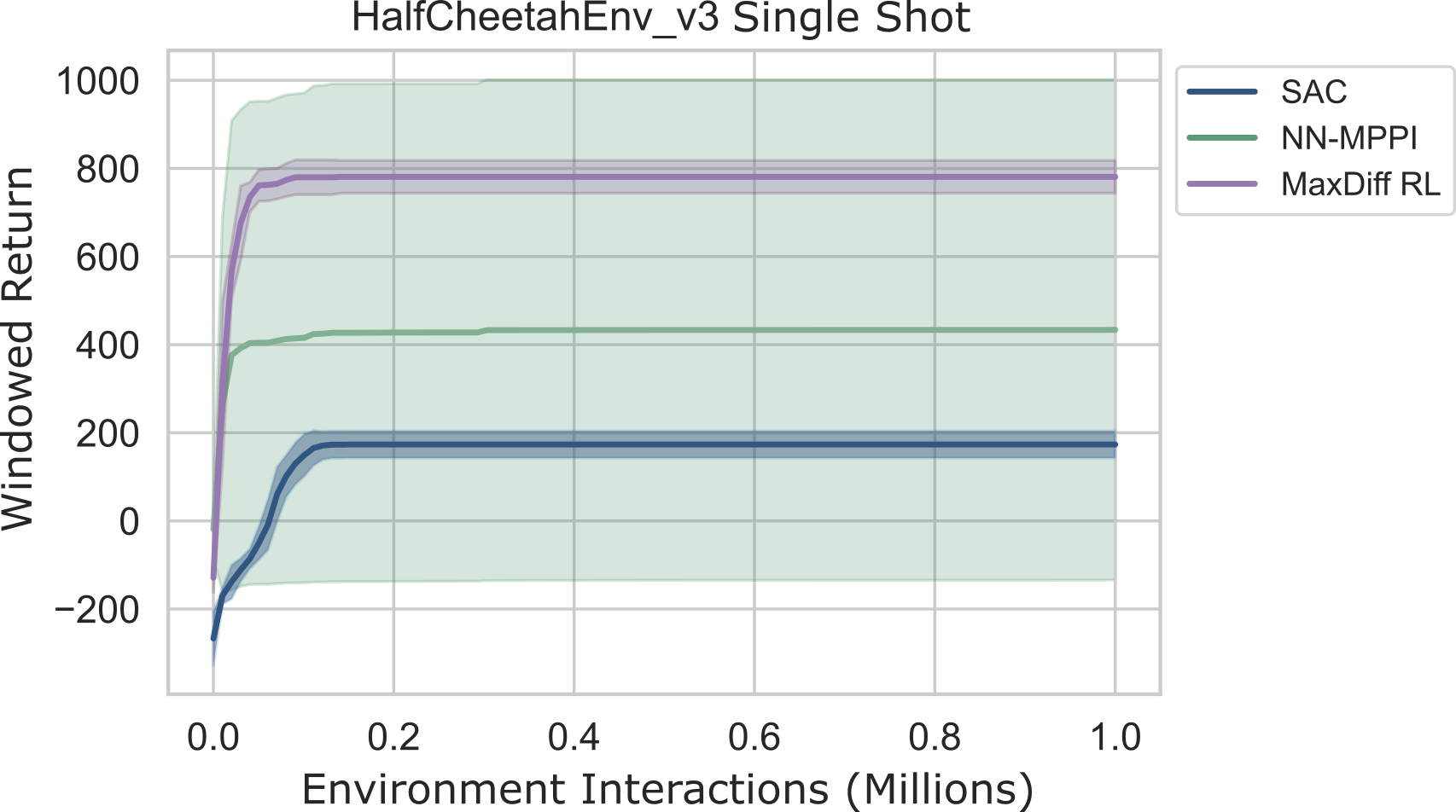}
    \caption{\textbf{Results of the half-cheetah benchmark in single-shot.} This figure compares the performance of MaxDiff RL to NN-MPPI and SAC on MuJoCo's HalfCheetahEnv v3 in single-shot. Since the half-cheetah can fall into an irreversible state (i.e., flipping upside down) this environment breaks the assumptions of MaxDiff RL. Nonetheless, we still succeed at the task with substantially less variance than alternative algorithms. For all reward curves, the shaded regions correspond to the standard deviation from the mean across 10 seeds.}
    \label{fig:cheetah_single}
\end{figure*}

\begin{figure*}[ht]
\centering
\includegraphics[width=0.8\linewidth]{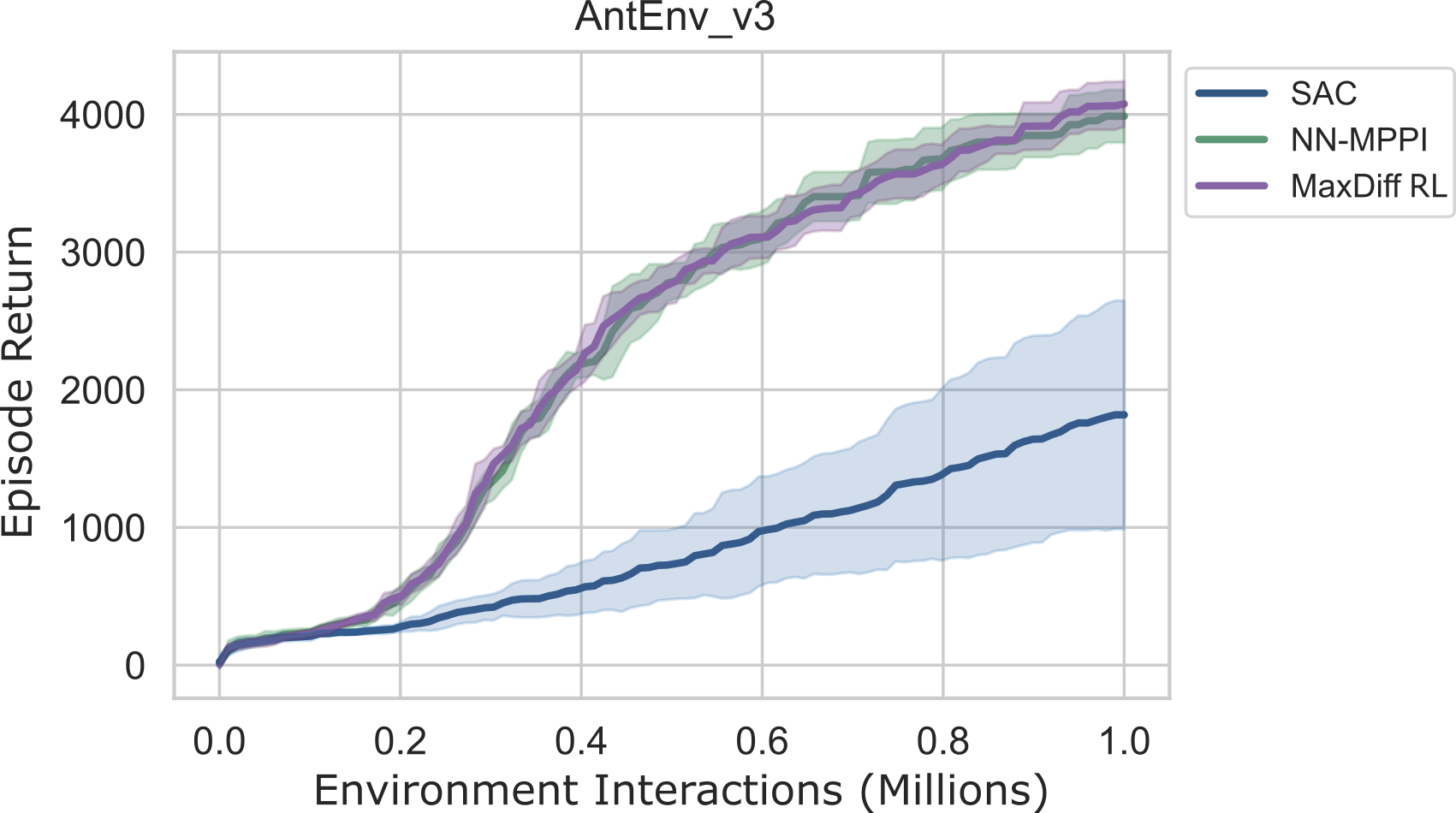}
    \caption{\textbf{Results of the ant benchmark.} This figure compares the performance of MaxDiff RL to NN-MPPI and SAC on MuJoCo's AntEnv v3 in multi-shot. Just as with our main text single-shot example, the ant environment breaks ergodicity, which pushes MaxDiff RL outside of the domain of its assumptions. Nonetheless, MaxDiff RL remains state-of-the-art with comparable performance to NN-MPPI. This is to be expected because in the worst case scenario where MaxDiff's additional entropy term in the objective has no effect on agent outcomes, our implementation of MaxDiff RL is identical to NN-MPPI. For all reward curves, the shaded regions correspond to the standard deviation from the mean across 10 seeds.}
    \label{fig:ant_multi}
\end{figure*}

\end{bibunit}

\end{document}